%% file: iclr2025_conference.tex
\documentclass[table]{article} 
\usepackage{iclr2025_conference,times}

\input{math_commands.tex}

\usepackage[utf8]{inputenc} 
\usepackage[T1]{fontenc}    
\usepackage{hyperref}       
\usepackage{url}            
\usepackage{booktabs}       
\usepackage{amsfonts}       
\usepackage{nicefrac}       
\usepackage{microtype}      
\usepackage{xcolor}         
\usepackage{todonotes}
\usepackage{makecell}
\usepackage{enumitem}
\usepackage{graphicx}
\graphicspath{ {./imgs/} }
\usepackage{subcaption}

\title{Complexity Lower Bounds of Adaptive Gradient Algorithms for Non-convex Stochastic Optimization under Relaxed Smoothness}

\author{Michael Crawshaw\\
Department of Computer Science \\
George Mason University \\
Fairfax, VA 22030, USA \\
\texttt{mcrawsha@gmu.edu} \\
\And
Mingrui Liu \\
Department of Computer Science \\
George Mason University \\
Fairfax, VA 22030, USA \\
\texttt{mingruil@gmu.edu}
}

\iclrfinalcopy 
\begin{document}

\maketitle

\begin{abstract}
Recent results in non-convex stochastic optimization demonstrate the convergence of popular adaptive algorithms (e.g., AdaGrad) under the $(L_0, L_1)$-smoothness condition, but the rate of convergence is a higher-order polynomial in terms of problem parameters like the smoothness constants. The complexity guaranteed by such algorithms to find an $\epsilon$-stationary point may be significantly larger than the optimal complexity of $\Theta \left( \Delta L \sigma^2 \epsilon^{-4} \right)$ achieved by SGD in the $L$-smooth setting, where $\Delta$ is the initial optimality gap, $\sigma^2$ is the variance of stochastic gradient. However, it is currently not known whether these higher-order dependencies can be tightened. To answer this question, we investigate complexity lower bounds for several adaptive optimization algorithms in the $(L_0, L_1)$-smooth setting, with a focus on the dependence in terms of problem parameters $\Delta, L_0, L_1$. We provide complexity bounds for three variations of AdaGrad, which show at least a quadratic dependence on problem parameters $\Delta, L_0, L_1$. Notably, we show that the decorrelated variant of AdaGrad-Norm requires at least $\Omega \left( \Delta^2 L_1^2 \sigma^2 \epsilon^{-4} \right)$ stochastic gradient queries to find an $\epsilon$-stationary point. We also provide a lower bound for SGD with a broad class of adaptive stepsizes. Our results show that, for certain adaptive algorithms, the $(L_0, L_1)$-smooth setting is fundamentally more difficult than the standard smooth setting, in terms of the initial optimality gap and the smoothness constants.
\end{abstract}

\section{Introduction} \label{sec:intro}
The best performing optimization algorithms for modern deep learning are
gradient-based optimizers with adaptive step sizes. For today's large-scale deep
learning tasks, such as training Large Language Models (LLMs), classical
non-adaptive optimizers like SGD perform significantly worse than their adaptive
counterparts, such as Adam \citep{kingma2014adam} and AdamW
\citep{loshchilov2018decoupled}. However, it remains open to theoretically
characterize the efficiency of adaptive gradient algorithms for non-convex
optimization.

An increasingly popular framework for describing optimization in deep learning
is $(L_0, L_1)$-smoothness, also known as relaxed smoothness
\citep{zhang2019gradient}. The conventional smoothness condition asserts that
the norm of the objective's Hessian is upper bounded by a constant, while the
weaker relaxed smoothness enforces only that the Hessian norm is upper bounded
by an affine function of the gradient norm (see Assumption \ref{ass:obj}).
Empirical evidence suggests that this condition may characterize neural network
training (for certain architectures) more accurately than conventional
smoothness \citep{zhang2019gradient, crawshaw2022robustness}.

Several recent works analyze the efficiency of adaptive algorithms for non-convex
optimization, particularly AdaGrad-Norm \citep{li2019convergence, ward2020adagrad,
wang2023convergence, attia2023sgd,faw2023beyond} and AdaGrad
\citep{wang2023convergence}. Indeed, adaptive algorithms are suited for relaxed
smoothness, since the local curvature of a relaxed smooth objective can be determined
from gradient information, and adaptive algorithms adjust their step size based on
gradients. Existing works demonstrate that AdaGrad can find an $\epsilon$-stationary
point with iteration complexity that scales as $\epsilon^{-4}$ in terms of $\epsilon$,
which matches the complexity of SGD in the stochastic, non-convex setting. However,
these guarantees also show that the complexity of AdaGrad (and some variants) is upper
bounded by a higher-order polynomial (i.e., at least quadratic) in terms of problem
parameters such as $\Delta$ (initial optimality gap), $\sigma^2$ (variance of stochastic
gradient), and the smoothness constants. See Table \ref{tab:complexity} for a summary of
these guarantees. This suggests the following question:

\begin{center}
\textbf{\textit{Can AdaGrad-type algorithms converge under relaxed smoothness without a
higher-order polynomial complexity in terms of problem parameters?}}
\end{center}

In this paper, we answer this question negatively for several variants of the
AdaGrad algorithm by providing complexity lower bounds that scale quadratically
in terms of the problem parameters $\Delta, L_0, L_1$. Our results are
summarized in Table \ref{tab:complexity}. This shows that, in the non-convex,
stochastic, relaxed smooth settings, the variants of AdaGrad considered here
cannot recover the $\Delta L_0 \sigma^2 \epsilon^{-4}$ complexity from the
$L_0$-smooth case; in this sense, these algorithms suffer a fundamental
difficulty in the relaxed smooth setting. In comparison, SGD with gradient
clipping does achieve the classical complexity of $\Delta L_0 \sigma^2
\epsilon^{-4}$ under the same setting as investigated in our lower
bounds~\citep{zhang2020improved}, which shows the surprising consequence that
SGD with gradient clipping outperforms AdaGrad in this setting. Additionally, we
give a lower bound for adaptive SGD with a broad class of adaptive step sizes,
in a setting where stochastic gradient noise scales linearly with the gradient
norm.

\begin{table}[t]
\caption{Iteration complexity to find an $\epsilon$-stationary point in non-convex
stochastic optimization. $\Delta$ is the initial optimality gap, $L$ and $(L_0, L_1)$
are the smoothness constants for the smooth and relaxed smooth cases, respectively.
$\sigma$ and $(\sigma_1, \sigma_2)$ are constants bounding the stochastic gradient
noise, depending on the stochastic assumption. See Assumptions \ref{ass:obj} and
\ref{ass:noise} for the full definitions. $\gamma$ is the stabilization constant of
AdaGrad. $^*$ denotes a high-probability guarantee with failure probability
$\delta$. $\gamma_1, \gamma_2, \gamma_3$ are defined and discussed in Section
\ref{sec:singlestep}.}
\label{tab:complexity}
\begin{center}
\resizebox{0.99\textwidth}{!}{\begin{tabular}{@{}ccc@{}}
\toprule
& Complexity & Stochasticity \\
\toprule
$L$-smooth & & \\
\cmidrule{1-1}
SGD \citep{ghadimi2013stochastic} & $\Theta \left( \frac{\Delta L \sigma^2}{\epsilon^4} + \frac{\Delta L}{\epsilon^2} \right)$ & (Bounded-Var) \\
\makecell{Decorrelated AdaGrad-Norm \\ \citep{li2019convergence}} & $\mathcal{O} \left( \left( \frac{\Delta L \sigma^2}{\epsilon^4} + \frac{\Delta L}{\epsilon^2} \right) \left( 1 + \frac{\Delta L + \sigma^2}{\gamma^2} \right) + \frac{\sqrt{\Delta L(\gamma^2 + \sigma^2)}}{\epsilon^2} \right)$ & (Subgaussian) \\
\makecell{AdaGrad-Norm \\ \citep{wang2023convergence}} & $\tilde{\mathcal{O}} \left( \left( \Delta L + \frac{\Delta^2 L^2 \sigma_2^4 + \sigma_1^4}{\gamma^2} \right) \left( \frac{\sigma_1^2}{\delta^4 \epsilon^4} + \frac{\sigma_2^2}{\delta^2 \epsilon^2} \right) \right)^*$ & (Affine-Var) \\
\makecell{AdaGrad-Norm \\ \citep{attia2023sgd}} & $\tilde{\mathcal{O}} \left( \frac{\Delta L \sigma_1^2 \left( \sigma_2^2 + 1 \right) + \sigma_1^4}{\epsilon^4} + \frac{\Delta L (1 + \sigma_2^4) + \gamma \sqrt{\Delta L (1 + \sigma_2^2)} + \sigma_1^2 (1 + \sigma_2^2)}{\epsilon^2} \right)^*$ & (Affine-Noise) \\
\makecell{AdaGrad-Norm \\ \citep{yang2024two}} & $\mathcal{O} \left( \frac{\Delta L \sigma^2 + \sigma^4}{\epsilon^4} + \frac{\Delta L + \gamma \sqrt{\Delta L} + \sigma^2 + \gamma \sigma}{\epsilon^2} \right)$ & (Bounded-Var) \\
\toprule
$(L_0, L_1)$-smooth & & \\
\cmidrule{1-1}
SGD \citep{li2024convex} & $\mathcal{O} \left( \frac{(\Delta + \sigma)^4 L_1^2}{\delta^4 \epsilon^4} + \frac{(\Delta + \sigma)^3 L_0}{\delta^3 \epsilon^4} + \frac{(\Delta + \sigma)^2 L_1^2}{\delta^2 \epsilon^2} + \frac{(\Delta + \sigma) L_0}{\delta \epsilon^2} \right)^*$ & (Bounded-Var) \\
\makecell{Gradient Clipping \\ \cite{zhang2019gradient, zhang2020improved}} & $\mathcal{O} \left( \frac{\Delta L_0 \sigma^2}{\epsilon^4} \right)$ & (Bounded-Noise) \\
\makecell{Gradient Clipping \\ \cite{koloskova2023revisiting}} & $\mathcal{O} \left( \frac{\Delta L_1 \sigma^4}{\epsilon^5} + \frac{\Delta}{\epsilon^2} \left( \frac{\sigma^2}{\epsilon^2} + 1 \right) \left( L_0 + \sqrt{L_0 L_1 \epsilon} + L_1 \epsilon \right) \right)$ & (Bounded-Var) \\
\makecell{AdaGrad-Norm \\ \cite{wang2023convergence}} & $\tilde{\mathcal{O}} \left( \left( \Delta^2 L_1^2 (1 + \sigma_2^4) + \Delta L_0 + \frac{(\Delta^4 L_1^4 + \Delta^2 L_0^2) \sigma_2^4 + \sigma_1^4}{\gamma^2} \right) \left( \frac{\sigma_1^2}{\delta^4 \epsilon^4} + \frac{\sigma_2^2}{\delta^2 \epsilon^2} \right) \right)^*$ & (Affine-Var) \\
\rowcolor{pink} \Gape[0pt][2pt]{\makecell{Decorrelated AdaGrad-Norm \\ (Theorem \ref{thm:adagrad_norm})}} & $\tilde{\Omega} \left( \frac{\Delta^2 L_1^2 \sigma^2}{\epsilon^4} + \frac{\Delta L_0 \sigma^2}{\epsilon^4} + \frac{\Delta^2 L_1^2}{\epsilon^2} \right)$ & (Bounded-Noise) \\
\rowcolor{pink} \Gape[0pt][2pt]{\makecell{Decorrelated AdaGrad \\ (Theorem \ref{thm:adagrad})}} & $\tilde{\Omega} \left( \frac{\Delta^2 L_0^2 \sigma^2}{\gamma^2 \epsilon^4} + \frac{\Delta^2 L_1^2 \sigma^2}{\gamma^2 \epsilon^2} \right)$ & (Bounded-Noise) \\
\rowcolor{pink} AdaGrad (Theorem \ref{thm:vanilla_adagrad}) & $\tilde{\Omega} \left( \frac{\Delta^2 L_0^2}{\epsilon^4} + \frac{\Delta^2 L_1^2}{\epsilon^2} \right)$ & (Bounded-Noise) \\
\rowcolor{pink} \Gape[0pt][2pt]{\makecell{Single-step Adaptive SGD \\ (Theorem \ref{thm:singlestep})}} & $\tilde{\Omega} \left( \frac{\Delta L_0 \sigma_1^2}{\epsilon^4} + \frac{(\Delta L_1)^{2- \gamma_2 - \gamma_3} \sigma_1^{\gamma_2 + \gamma_3 - \gamma_1}}{\epsilon^{2-\gamma_1} } \right)$ & (Affine-Noise) \\
\bottomrule
\end{tabular}}
\end{center}
\end{table}

We emphasize that the complexity's dependence on problem parameters can be important for
distinguishing the relative performance of optimization algorithms. A classic example is
the case of smooth, strongly convex functions, where both gradient descent and
Nesterov's Accelerated Gradient (NAG) exhibit linear convergence, but the iteration
complexity of NAG is faster than GD by a factor of $\sqrt{\kappa}$, where $\kappa$ is
the condition number of the objective function \citep{nesterov2013introductory}.

Our contributions can be summarized as follows:
\begin{enumerate}
    \item In Theorem \ref{thm:adagrad_norm}, we provide a complexity lower bound of
        $\Omega \left( \Delta^2 L_1^2 \sigma^2 \epsilon^{-4} \right)$ for Decorrelated
        AdaGrad-Norm (which uses decorrelated step sizes and a shared learning rate for
        all coordinates) under $(L_0, L_1)$-smoothness and almost surely bounded
        gradient noise. The proof uses a novel construction of a difficult objective for
        which Decorrelated AdaGrad-Norm may diverge (depending on the choice of
        hyperparameter), combined with a high-dimensional objective (adapted from
        \cite{drori2020complexity}). This lower bound matches the upper bound of
        AdaGrad-Norm in two out of three dominating terms, and only differs in
        terms of $\sigma$. See Section \ref{sec:adagrad_norm} for a comparison
        between these upper and lower bounds.
    \item In Theorem \ref{thm:adagrad}, we lower bound the complexity of Decorrelated
        AdaGrad by $\Omega \left( \Delta^2 L_0^2 \sigma^2 \gamma^{-2} \epsilon^{-4}
        \right)$, where $\gamma$ is a hyperparameter. The proof uses a novel
        high-dimensional objective for which the algorithm diverges when $\eta \geq
        \tilde{\Omega}(\gamma/(L_1 \sigma))$. Theorem \ref{thm:vanilla_adagrad}
        extends this result for the original AdaGrad algorithm, achieving a
        lower bound of $\Omega \left( \Delta^2 L_0^2 \epsilon^{-4} \right)$.
        While our lower bound for AdaGrad is weaker than for its decorrelated
        counterpart, this complexity is still larger than the optimal smooth
        rate in regimes when $\Delta$ or the smoothness constants are large
        compared to $\sigma$.
    \item In Theorem \ref{thm:singlestep}, we consider the setting of $(L_0,
        L_1)$-smoothness and gradient noise bounded by an affine function of the
        gradient norm. For SGD with a broad class of adaptive step sizes, we
        show a lower bound that is nearly quadratic in the problem parameters
        $\Delta, L_1$. This is proven by showing that one of the following must
        hold: (1) adaptive SGD can be forced into a biased random walk with a
        constant probability of divergence, or (2) the adaptive step size is
        $\mathcal{O} \left( 1/(\Delta L_1^2) \right)$ when optimizing a function
        with gradient magnitude equal to $\epsilon$, which leads to slow
        convergence.
\end{enumerate}

The remainder of the paper is structured as follows. We discuss related work in
Section \ref{sec:related_work}, then give the formal problem statement in
Section \ref{sec:problem}. We then present our complexity lower bounds for
Decorrelated AdaGrad-Norm (Section \ref{sec:adagrad_norm}), Decorrelated AdaGrad
and the original AdaGrad (Section \ref{sec:adagrad}), and adaptive SGD (Section
\ref{sec:singlestep}). We conclude with Section \ref{sec:discussion}.

\section{Related Work} \label{sec:related_work}
\textbf{Relaxed Smoothness.} Relaxed smoothness was introduced by
\citet{zhang2019gradient}, who showed that GD with normalization converges
faster than GD under this condition. This inspired a line of work focusing on
efficient algorithms under this condition. \citet{zhang2020improved} showed an
improved analysis of gradient clipping, and \citet{jin2021non} considered a
non-convex distributionally robust optimization satisfying this condition.
Several recent works
\citep{liu2022communication,crawshaw2023episode,crawshaw2023federated} designed
communication-efficient federated learning algorithms under relaxed smoothness.
\citet{li2024convex} analyzed gradient-based methods without gradient clipping
under generalized smoothness. \citet{crawshaw2022robustness} studied a
coordinate-wise version of relaxed smoothness, empirically showed that
transformers satisfy this condition, and designed a generalized signSGD
algorithm with convergence guarantees. \citet{chen2023generalized} proposed a
new notion of $\alpha$-symmetric generalized smoothness and analyzed a class of
normalized GD algorithms. More recently, a few works have investigated momentum
and variance reduction techniques within the framework of individual relaxed
smooth conditions \citep{liu2023nearlyoptimal} or on average relaxed smooth
conditions \citep{reisizadeh2023variance}. 

\textbf{Adaptive Gradient Methods.} Adaptive gradient optimization algorithms
automatically adjust the step size for each coordinate based on gradient information,
and have become very important in machine learning. Examples include Adagrad
\citep{duchi2011adaptive,mcmahan2010adaptive}, Adam \citep{kingma2014adam}, RMSProp
\citep{tieleman2012lecture}, and other variants
\citep{loshchilov2018decoupled,shazeer2018adafactor}. Most theoretical analyses of
adaptive optimization methods are based on the assumptions of smoothness or convexity
\citep{reddi2018convergence,chen2018convergence,guo2021novel}. Recently, some works
established convergence results for AdaGrad-Norm \citep{faw2023beyond,
wang2023convergence} and Adam \citep{li2023convergence} under the relaxed smoothness
condition, and all of the convergence rates in these works exhibit a higher order
polynomial dependence on $L_1$.

\textbf{Lower Bounds.} Lower bounds for first-order convex optimization are well
studied
\citep{nemirovskii1983problem,nesterov2013introductory,woodworth2017lower,woodworth2016tight}.
The lower bounds of nonconvex smooth optimization were studied in the
deterministic setting
\citep{cartis2010complexity,carmon2020lower,carmon2021lower}, finite-sum setting
\citep{fang2018spider} and stochastic setting
\citep{drori2020complexity,arjevani2023lower}. For relaxed smooth problems,
\citet{zhang2019gradient} and \citet{crawshaw2022robustness} derived a lower
bound for GD and showed that its complexity depends on the maximum magnitude of
the gradient in a sublevel set. \citet{faw2023beyond} considered the lower bound
for normalized SGD, clipped SGD, and signSGD with momentum in the affine noise
setting, and showed that these algorithms cannot converge under certain
parameter regimes. \citet{crawshaw2023federated} showed a lower bound for
minibatch SGD with gradient clipping in the affine noise setting.

\section{Problem Statement} \label{sec:problem}

\subsection{Optimization Objectives}
We consider the problem of finding an approximate stationary point of a
nonconvex, relaxed smooth function with access to a stochastic gradient. Let
$\Delta, L_0, L_1, \sigma_1, \sigma_2, \sigma > 0$. We will denote the objective
function as $f: \mathbb{R}^d \rightarrow \mathbb{R}$, the stochastic gradient as
$g: \mathbb{R}^d \times \Xi \rightarrow \mathbb{R}^d$, and the noise
distribution as $\gD,$ which is a distribution over $\Xi$. We then consider the
set of problem instances $(f, g, \gD)$ satisfying the following conditions:
\begin{assumption} \label{ass:obj}
\textbf{(1)} $f$ is bounded from below and $f(\mathbf{0}) - \inf_{\vx} f(\vx) \leq
\Delta$. \textbf{(2)} $f$ is continuously differentiable and $(L_0, L_1)$-smooth: For
every $\vx, \vy \in \mathbb{R}^d$ with $\|\vx - \vy\| \leq 1/L_1$: \[
    \|\nabla f(\vx) - \nabla f(\vy)\| \leq (L_0 + L_1 \|\nabla f(\vx)\|) \|\vx - \vy\|.
\] \textbf{(3)} $\mathbb{E}_{\xi \sim \gD}[g(\vx, \xi)] = \nabla f(\vx)$ for all $\vx
\in \mathbb{R}^d$.
\end{assumption}

\begin{assumption} \label{ass:noise}
For all $\vx \in \mathbb{R}^d$:
\begin{enumerate}[align=left]
    \item [\textup{(Bounded-Noise)}] $\|g(\vx, \xi) - \nabla f(\vx)\| \leq \sigma$ almost surely over $\xi \sim \gD$. \label{ass:bounded_noise}
    \item [\textup{(Affine-Noise)}] $\|g(\vx, \xi) - \nabla f(\vx)\| \leq \sigma_1 + \sigma_2 \|\nabla f(\vx)\|$ almost surely over $\xi \sim \gD$.
    \item [\textup{(Bounded-Var)}] $\mathbb{E}_{\xi \sim \gD} \left[ \|g(\vx, \xi) - \nabla f(\vx)\|^2 \right] \leq \sigma^2$.
    \item [\textup{(Affine-Var)}] $\mathbb{E}_{\xi \sim \gD} \left[ \|g(\vx, \xi) - \nabla f(\vx)\|^2 \right] \leq \sigma_1^2 + \sigma_2^2 \|\nabla f(\vx)\|^2 $.
    \item [\textup{(Subgaussian)}] $\mathbb{E}_{\xi \sim \gD} \left[ \exp(\|g(\vx, \xi) - \nabla f(\vx)\|^2/\sigma^2) \right] \leq 1$.
\end{enumerate}
\end{assumption}

We will denote by $\gF_{\text{as}}(\Delta, L_0, L_1, \sigma)$ the set of problems $(f,
g, \gD)$ satisfying Assumption \ref{ass:obj} and (Bounded-Noise), and by
$\gF_{\text{aff}}(\Delta, L_0, L_1, \sigma_1, \sigma_2)$ those satisfying Assumption
\ref{ass:obj} and (Affine-Noise).

In this paper, we present new results under (Bounded-Noise) and (Affine-Noise),
though we state the other assumptions for discussion with related work. It is
important to note that (Bounded-Noise) is strictly stronger than (Bounded-Var).
Therefore, the lower bounds that we prove for $\gF_{\text{as}}$ also hold for
the class of problems satisfying Assumption \ref{ass:obj} and (Bounded-Var).
This is because any difficult problem instance in the former class is also in
the latter. An analogous statement for (Affine-Noise) and (Affine-Var) holds by
the same reasoning. Our primary focus for stochasticity in this work is
(Bounded-Noise), since this is the standard assumption used by early work on
relaxed smoothness \citep{zhang2019gradient,zhang2020improved}.

\subsection{Optimization Algorithms}
We will consider four optimization algorithms --- Decorrelated AdaGrad-Norm,
Decorrelated AdaGrad, AdaGrad, and single-step adaptive SGD --- and their
behavior for problems in $\gF_{\text{as}}$ and $\gF_{\text{aff}}$.

\textbf{Decorrelated AdaGrad-Norm} We first consider a variant of AdaGrad that we refer
to as Decorrelated AdaGrad-Norm:
\begin{equation} \label{eq:adagrad_norm}
    \vx_{t+1} = \vx_t - \frac{\eta}{\sqrt{\gamma^2 + \sum_{i=0}^{t-1} \|\vg_i\|^2}} \vg_t,
\end{equation}
where $\eta > 0$ is a step size coefficient, $\vg_t = g(\vx_t, \xi_t)$, and $\xi_t \sim
\gD$ is independent over $t$. Notice that the denominator contains the sum of squared
gradient norms, as opposed to the coordinate-wise operations used in the original
AdaGrad. This type of denominator is used in AdaGrad-Norm, whose convergence was studied
under various conditions in \citet{ward2020adagrad, faw2022power, attia2023sgd,
wang2023convergence, yang2024two}. Further, the sum of squared gradients in the
denominator ranges from $i=0$ to $i=t-1$, meaning that it does not contain the
most recent stochastic gradient $g_t$. This type of decorrelated step size was
considered in \citet{li2019convergence}, which provided convergence guarantees
in the smooth setting (see Table \ref{tab:complexity}).

\textbf{AdaGrad and Decorrelated AdaGrad} Next, we consider two coordinate-wise variants
of AdaGrad, including the original AdaGrad and a variation with a decorrelated step
size. The original AdaGrad \citep{duchi2011adaptive} is defined as follows:
\begin{equation} \label{eq:adagrad}
    \vx_{t+1} = \vx_t - \frac{\eta}{\sqrt{\gamma^2 + \sum_{i=0}^t \vg_i^2}} \vg_t,
\end{equation}
where the squaring $\vg_i^2$ is performed element-wise. Decorrelated AdaGrad
\citep{li2019convergence} is similarly defined as
\begin{equation} \label{eq:vanilla_adagrad}
    \vx_{t+1} = \vx_t - \frac{\eta}{\sqrt{\gamma^2 + \sum_{i=0}^{t-1} \vg_i^2}} \vg_t,
\end{equation}
the only difference from AdaGrad being that the sum in the denominator does not contain
the gradient from the current step, so the step size at step $t$ is independent of the
stochastic gradient noise at step $t$.

\textbf{Single-Step Adaptive SGD} Last, we consider a class of algorithms that implement
stochastic gradient descent with an adaptive step size, but whose step size function
only depends on the current gradient. For $\alpha: \mathbb{R}^d \rightarrow \mathbb{R}$,
single-step adaptive SGD is defined as:
\begin{equation} \label{eq:singlestep}
    \vx_{t+1} = \vx_t - \alpha(\vg_t) \vg_t,
\end{equation}
where again $\vg_t = g(\vx_t, \xi_t)$ and $\xi_t \sim \gD$ is independent over $t$. At
each step $t$, the update $\vx_{t+1} - \vx_t$ is determined completely by the stochastic
gradient sampled at step $t$, hence the name ``single-step". However, the step size in
the direction $\vg_t$ is computed as an \textit{arbitrary} function $\alpha$ of the
stochastic gradient. This class of algorithms includes SGD with constant step size, SGD
with gradient clipping, and normalized SGD; it does not include Adam or AdaGrad.

\subsection{Complexity}
Given a problem $(f, g, \gD)$ and $\epsilon > 0$, the goal of an optimization algorithm
$A$ is to find an $\epsilon$-approximate stationary point of $f$, that is, a point $\vx
\in \mathbb{R}^d$ such that $\|\nabla f(\vx)\| < \epsilon$. We want to characterize the
number of gradient calls required by an algorithm to find such a point. Since an
algorithm can only gain information about the objective $f$ through \textit{stochastic}
gradients, it cannot necessarily guarantee to find an $\epsilon$-stationary point, but
it may find one in expectation or with high probability. Denote by $\{\vx_t\}$ the
sequence of points at which the stochastic gradient is queried by $A$ when given $(f, g,
\gD)$ as input. We then define the worst-case complexity of $A$ on problem class $\gF$
as \[
    \gT(A, \gF, \epsilon) = \sup_{(f, g, \gD) \in \gF} \min \left\{ t \geq 1 \;\bigg|\; \min_{s < t} \mathbb{E} \left[ \|\nabla f(\vx_s)\| \right] < \epsilon \right\}.
\]
To summarize, the worst-case complexity $\gT(A, \gF, \epsilon)$ measures the number of
gradient calls required by $A$ to find an $\epsilon$-approximate stationary point in
expectation, for any problem in $\gF$. We also consider the worst-case complexity for
finding an $\epsilon$-stationary point with high probability: \[
    \gT(A, \gF, \epsilon, \delta) = \sup_{(f, g, \gD) \in \gF} \min \left\{ t \geq 1 \;\bigg|\; \text{Pr} \left( \min_{s < t} \|\nabla f(\vx_s)\| < \epsilon \right) > 1 - \delta \right\}.
\]

Following \cite{arjevani2023lower}, most of our results (Theorems
\ref{thm:adagrad_norm}, \ref{thm:adagrad}, \ref{thm:vanilla_adagrad}) will provide
in-expectation lower bounds, that is, lower bounds for $\gT(A, \gF, \epsilon)$. Our last
result (Theorem \ref{thm:singlestep}) will provide lower bounds for $\gT(A, \gF,
\epsilon, \delta)$ for any given $\delta$, i.e., high-probability lower bounds.
Throughout the paper, $\mathcal{O}(\cdot), \Omega(\cdot)$ and $\Theta(\cdot)$ omit
universal constants, and $\tilde{\mathcal{O}}(\cdot), \tilde{\Omega}(\cdot)$, and
$\tilde{\Theta}(\cdot)$ omit universal constants and factors logarithmic in terms of
problem parameters $\Delta, L_0, L_1, \sigma_1, \sigma_2, \sigma$, and target gradient
norm $\epsilon$.

\section{Decorrelated AdaGrad-Norm} \label{sec:adagrad_norm}
Our first result gives a lower bound for Decorrelated AdaGrad-Norm, which shows
that the complexity has a quadratic dependence in terms of problem parameters
$\Delta, L_1$.

\begin{theorem} \label{thm:adagrad_norm}
Denote $\gF = \gF_{\textup{as}}(\Delta, L_0, L_1, \sigma)$, and let algorithm
$A_{\text{DAN}}$ denote Decorrelated AdaGrad-Norm (\Eqref{eq:adagrad_norm}) with
parameters $\eta > 0$ and $0 < \gamma \leq \tilde{\mathcal{O}} \left( \Delta L_1
\right)$. Let \(
    0 < \epsilon \leq \mathcal{O} \left( \min \left\{ \sqrt{\Delta L_0}, \sqrt{\Delta L_1 \gamma}, \Delta L_1 \right\} \right).
\)
If $\Delta L_1^2 \geq L_0$, then \[
    \gT(A_{\text{DAN}}, \gF, \epsilon) \geq \Omega \left( \frac{\Delta^2 L_1^2 \sigma^2}{\epsilon^4} + \frac{\Delta L_0 \sigma^2 \log(1 + \sigma^2/\gamma^2)}{\epsilon^4} + \frac{\Delta^2 L_1^2}{\epsilon^2} \right).
\]
\end{theorem}
The proof is given in Appendix \ref{app:adagrad_norm_proof}. Before giving a
sketch of the proof, we make a few observations about the result. \textbf{(1)}
The lower bound contains the term $\Delta L_0 \sigma^2 \epsilon^{-4}$, which is
the optimal complexity for a first-order algorithm in the $L_0$-smooth,
non-convex, stochastic setting \citep{arjevani2023lower}. This means that
Decorrelated AdaGrad-Norm requires \textit{at least} as many iterations to solve
the current problem as any first order algorithm requires to solve the smooth
counterpart. \textbf{(2)} The dominating term is quadratic in $\Delta, L_1$.
Therefore, \textbf{under relaxed smoothness, Decorrelated AdaGrad-Norm cannot
recover the optimal complexity of the smooth case}, as one might hope.
\textbf{(3)} In the deterministic case (i.e., $\sigma = 0$), the complexity is
$\Delta^2 L_1^2 \epsilon^{-2}$, which is still quadratic in the problem
parameters $\Delta, L_1$ and does not match the complexity achieved by
deterministic GD in the $L_0$-smooth case, i.e., $\Delta L_0 \epsilon^{-2}$.
\textbf{(4)} \textbf{This lower bound for Decorrelated AdaGrad-Norm matches the
upper bound of AdaGrad-Norm in two out of three dominating terms}. The
dominating terms of the upper bound of AdaGrad-Norm from
\citet{wang2023convergence} (see Table 1) are
\begin{equation}
    \tilde{\mathcal{O}} \left( \frac{\Delta^2 L_1^2 \sigma^2}{\epsilon^4} + \frac{\Delta L_0 \sigma^2}{\epsilon^4} + \frac{\sigma^6}{\gamma^2 \epsilon^4} \right).
\end{equation}
The first two terms of this upper bound match our lower bound up to log terms.
Note that this result (Theorem 8 from \citet{wang2023convergence}) uses
(Bounded-Var), whereas we use (Bounded-Noise). However, their upper bound still
applies for the stronger (Bounded-Noise) and our lower bound still applies for
the weaker (Bounded-Var). The gap between our lower bound and this upper bound
is the third term (due to the noise $\sigma$), which means that either (a) the
upper bound can be decreased; (b) the lower bound can be increased; (c)
Decorrelated AdaGrad-Norm differs from AdaGrad-Norm in its dependence on the
noise $\sigma$; or (d) the gap is caused by the difference in noise assumptions.

Lastly, the condition $\Delta L_1^2 \geq L_0$ for Theorem \ref{thm:adagrad_norm} ensures
that $(L_0, L_1)$-smoothness does not degenerate to $L$-smoothness. Indeed, Lemma 3.5 of
\citet{li2024convex} implies that $\|\nabla f(\vx)\| \leq \mathcal{O}(\Delta L_1)$ for
every $\vx$ with $f(\vx) \leq f(\vzero)$, so $\|\nabla^2 f(\vx)\| \leq L_0 + L_1
\|\nabla f(\vx)\| \leq \mathcal{O} \left( L_0 + \Delta L_1^2 \right)$. Therefore, if the
condition $\Delta L_1^2 \geq L_0$ fails, then any objective $f$ which is $(L_0,
L_1)$-smooth is also $\Theta(L_0)$-smooth in a sublevel set containing the initial
point. We also require an upper bound on the stabilization constant: $\gamma \leq
\tilde{\mathcal{O}}(\Delta L_1)$, which covers all practical regimes in which $\gamma$
is chosen as a small constant. In Appendix \ref{app:gamma}, we show that this condition
can be removed in the deterministic setting while recovering the complexity lower-bound
$\tilde{\Omega}(\Delta^2 L_1^2 \epsilon^{-2})$.

\subsection{Proof Outline} \label{sec:adagrad_norm_sketch}
The proof of Theorem \ref{thm:adagrad_norm} follows two cases, depending on the
choice of the parameter $\eta$. If $\eta \geq 1/L_1$, then the algorithm can
diverge on a fast growing function. On the other hand, if $\eta \leq 1/L_1$,
then the algorithm converges slowly on a function with small gradient. This
proof structure is similar to previous lower bounds under relaxed smoothness
\citep{zhang2019gradient, crawshaw2022robustness}, but our result requires
significantly different constructions due to the structure of AdaGrad updates,
and since previous bounds only achieve $\epsilon^{-2}$ dependence, whereas we
show $\epsilon^{-4}$ dependence.

\textbf{Divergence when $\eta \geq 1/L_1$} We want the update size $\|\vx_{t+1} -
\vx_t\|$ to be lower bounded by a constant, but the step size decreases over $t$ due to
the sum of squared gradients in the denominator. Intuitively, this means that the
gradient magnitude $\|\vg_t\|$ should increase with $t$ to offset the decreasing step
size. However, faster growth of $\|\vg_t\|$ causes faster decrease in the effective step
size. We can balance these two effects and force the trajectory to diverge with a
properly constructed objective function and a sequence of gradients satisfying
$\|\vg_t\| = \Theta \left( (t \log t)^t \right)$, which is executed in Lemma
\ref{lem:adagrad_norm_div}.

\begin{lemma} \label{lem:adagrad_norm_div}
Suppose that $\Delta L_1^2 \geq L_0$, $\eta \geq \frac{1}{L_1}$, and $\gamma \leq
\tilde{\mathcal{O}}(\eta \Delta L_1^2)$. Then there exists a problem instance $(f, g,
\gD) \in \gF_{\text{as}}(\Delta, L_0, L_1, 0)$ such that $\|\nabla f(\vx_t)\| \geq
\Delta L_1$ for all $t \geq 0$.
\end{lemma}

\begin{figure} \label{fig:psi_div}
\centering
\begin{subfigure}[t]{0.48\textwidth}
    \includegraphics[width=\textwidth]{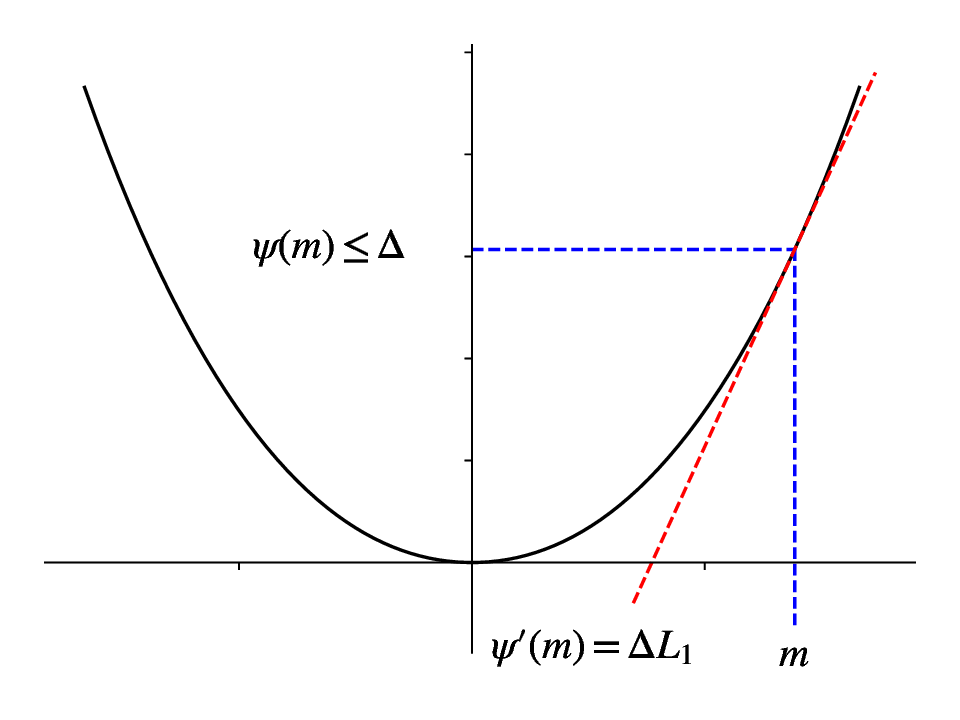}
    \caption{Graph of $\psi(x) = \frac{L_0}{L_1^2} \left( \exp(L_1 |x|) - L_1 |x| - 1 \right)$.}
    \label{fig:psi}
\end{subfigure}
~
\begin{subfigure}[t]{0.48\textwidth}
    \includegraphics[width=\textwidth]{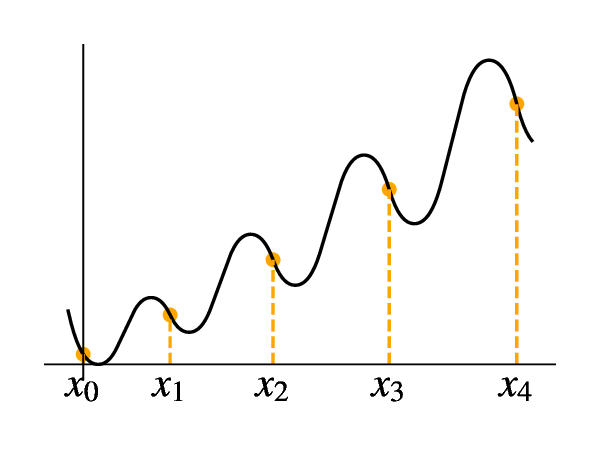}
    \caption{Trajectory from Lemma \ref{lem:adagrad_norm_div}.}
    \label{fig:adagrad_norm_div}
\end{subfigure}
\caption{Objectives from Lemma \ref{lem:adagrad_norm_div}. $m := (\psi')^{-1}(\Delta L_1) = \frac{1}{L_1} \log \left( 1 + \frac{\Delta L_1^2}{L_0} \right)$.}
\end{figure}

The trajectory of the algorithm analyzed in Lemma \ref{lem:adagrad_norm_div} is
informally pictured in Figure \ref{fig:adagrad_norm_div}. The objective function is a
piecewise combination of copies of the function \(
    \psi(x) = \frac{L_0}{L_1^2} \left( \exp(L_1 |x|) - L_1 |x| - 1 \right),
\)
which is shown in Figure \ref{fig:psi}. $\psi$ is constructed to satisfy
$|\psi''(x)| = L_0 + L_1 |\psi'(x)|$ for all $x$, so it grows as quickly as
possible under $(L_0, L_1)$-smoothness. As shown in Figure
\ref{fig:adagrad_norm_div}, at each step the algorithm receives a gradient $g_t$
and ``jumps" over a valley to reach a new point with gradient $g_{t+1}$. In
order to achieve this jump, we need the sequence of gradients to satisfy \[
    \frac{\eta |g_t|}{\sqrt{\gamma^2 + \sum_{i=0}^{t-1} g_i^2}} \geq \frac{4}{L_1} \log \left( 1 + \frac{L_1 |g_{t+1}|}{L_0} \right).
\]
This recurrent inequality is tricky since the magnitude of each $|g_t|$ is constrained
not just by the history $\{|g_i|\}_{i < t}$, but also by the future $g_{t+1}$. We show
that this requirement is satisfied if $g_t = \Theta \left( \left( (t+1) \log (1 + \Delta
L_1^2 L_0^{-1} (t+1)) \right)^t \Delta L_1 \right)$, and that this sequence of gradients
can be realized by an objective function in $\gF_{\text{as}}(\Delta, L_0, L_1, \sigma)$.

\textbf{Slow convergence when $\eta \leq 1/L_1$} In this case, previous lower
bounds \citep{zhang2019gradient, crawshaw2022robustness} consider a
one-dimensional function with deterministic gradients to show a complexity of
$\Omega \left( \Delta^2 L_1^2 \epsilon^{-2} \right)$. To achieve $\epsilon^{-4}$
dependence, we consider a high-dimensional function with stochastic gradients
(adapted from \citet{drori2020complexity}), for which the first partial
derivative $\nabla_1 f$ has magnitude $\epsilon$, and the stochastic gradient
noise affects coordinates with index greater than 1. Since the same learning
rate is shared by all coordinates, the noise in later coordinates will decrease
the learning rate for the first coordinate. Combining with $\eta \leq 1/L_1$
leads to the desired complexity.

\begin{lemma} \label{lem:adagrad_norm_slow}
Let $T = \Theta \left( \Delta^2 L_1^2 \sigma^2 \epsilon^{-4} + \Delta L_0
\sigma^2 \epsilon^{-4} + \Delta^2 L_1^2 \epsilon^{-2} \right)$, and suppose $d
\geq T$ and $\epsilon \leq \mathcal{O} \left( \min \left\{ \sqrt{\Delta L_0},
\sqrt{\Delta L_1 \gamma} \right\} \right)$. If $\eta \leq \frac{1}{L_1}$, then
there exists some $(f, g, \gD) \in \gF_{\text{as}}(\Delta, L_0, L_1, \sigma)$
such that $\|\nabla f(\vx_t)\| = \epsilon$ for all $0 \leq t \leq T-1$.
\end{lemma}

See Appendix \ref{app:adagrad_norm_proof} for details on the objective function,
stochastic gradient oracle, and analysis of the trajectory. The theorem is then
proved by combining Lemmas \ref{lem:adagrad_norm_div} and
\ref{lem:adagrad_norm_slow}: No matter the choice of the parameter $\eta$, the
algorithm will not find an $\epsilon$-stationary point within the first $T$
steps.

\section{AdaGrad and Decorrelated AdaGrad} \label{sec:adagrad}
Here we present lower bounds for Decorrelated AdaGrad (Theorem \ref{thm:adagrad}) and
the original AdaGrad (Theorem \ref{thm:vanilla_adagrad}). Both lower bounds are
quadratic in $\Delta, L_0$, but our result for Decorrelated AdaGrad has a stronger
dependence on $\sigma$ than that of AdaGrad. This discrepancy is further discussed
below.

\begin{theorem} \label{thm:adagrad}
Denote $\gF = \gF_{\text{as}}(\Delta, L_0, L_1, \sigma)$, and let $A_{\text{DA}}$ denote
Decorrelated AdaGrad (\Eqref{eq:adagrad}) with parameters $\eta, \gamma > 0$. Suppose $0
< \epsilon < \tilde{\mathcal{O}} \left( \min \left\{ \Delta L_1, \sqrt{\Delta L_1
\sigma} \right\} \right)$. Then \[
    \gT(A_{\text{DA}}, \gF, \epsilon) \geq \Omega \left( \frac{\Delta^2 L_0^2 \sigma^2}{\gamma^2 \epsilon^4} + \frac{\Delta^2 L_1^2 \sigma^2}{\gamma^2 \epsilon^2 \log^2 \left( 1 + \frac{\Delta L_1^2}{L_0} \right)} \right).
\]
\end{theorem}

\begin{theorem} \label{thm:vanilla_adagrad}
Denote $\gF = \gF_{\text{as}}(\Delta, L_0, L_1, \sigma)$, and let $A_{\text{ada}}$
denote AdaGrad (\Eqref{eq:vanilla_adagrad}) with parameters $\eta, \gamma > 0$. Suppose
$0 < \epsilon < \tilde{\mathcal{O}} \left( \min \left\{ \Delta L_1, \sqrt{\Delta L_1
\sigma} \right\} \right)$. If $\gamma \leq \sigma$, then \[
    \gT(A_{\text{ada}}, \gF, \epsilon) \geq \Omega \left( \frac{\Delta^2 L_0^2}{\epsilon^4} + \frac{\Delta^2 L_1^2}{\epsilon^2 \log^2 \left( 1 + \frac{\Delta L_1^2}{L_0} \right)} \right).
\]
\end{theorem}

The proofs of both theorems above are given in Appendix \ref{app:adagrad_proof}.
The results above exhibit several important properties. \textbf{(1)} As in
Theorem \ref{thm:adagrad_norm}, the dominating term $\Delta^2 L_0^2 \sigma^2
\gamma^{-2} \epsilon^{-4}$ of the lower bound in Theorem \ref{thm:adagrad} is
greater than the optimal complexity $\Delta L_0 \sigma^2 \epsilon^{-4}$ of the
smooth case (up to the choice of $\gamma$, which is usually a small constant).
In fact, the dominating term is quadratic in $\Delta$ and $L_0$, so \textbf{the
complexity of Decorrelated AdaGrad in this setting is fundamentally larger than
the optimal complexity of the smooth counterpart}. Unlike Theorem
\ref{thm:adagrad_norm}, $L_1$ does not appear in the dominating term of this
bound. \textbf{(2)} Compared to Decorrelated AdaGrad, our bound for AdaGrad
loses a factor of $\sigma^2/\gamma^2$. This arises in our construction from the
fact that, at step $t$, any noise present in $\vg_t$ will also appear in the
denominator of the update, so that the update size of AdaGrad is not as
sensitive to noise as the decorrelated counterpart. Still, in regimes where
$\Delta, L_0$ are large compared to $\sigma$, our complexity lower bound of
$\tilde{\Omega}(\Delta^2 L_0^2 \epsilon^{-4})$ is larger than the optimal
complexity of the smooth case. This shows that \textbf{AdaGrad cannot recover the
optimal complexity of the smooth case in all relaxed smooth regimes}.
\textbf{(3)} The lower bound of Theorem \ref{thm:adagrad} diverges to $\infty$
when the $\gamma$ goes to 0, which confirms the conventional wisdom that a
non-zero stabilization constant is necessary in practice.

\subsection{Proof Outline}
The structure of the proof is similar to Theorem \ref{thm:adagrad_norm} (outlined in
Section \ref{sec:adagrad_norm_sketch}), but we can achieve divergence for Decorrelated
AdaGrad under the weaker condition $\eta \geq \Theta \left( \min \left\{ \gamma/(L_1
\sigma), \gamma \epsilon/(L_0 \sigma) \right\} \right)$ using a novel high-dimensional
construction that takes advantage of the coordinate-wise learning rates of Decorrelated
AdaGrad by injecting noise into one coordinate per timestep. When $\eta$ is smaller than
this threshold, convergence is slow for a one-dimensional, linear function with slope
$\epsilon$.

\textbf{Divergence when $\eta \geq \Theta \left( \min \left\{ \gamma/(L_1 \sigma),
\gamma \epsilon/(L_0 \sigma) \right\} \right)$} For any $d \geq 1$, we consider the
objective function $f(\vx) = \sum_{i=1}^d \psi(\langle \vx, \mathbf{e}_i \rangle)$,
where $\psi$ is as defined in Section \ref{sec:adagrad_norm_sketch}. Letting $m =
(\psi')^{-1}(\epsilon) = \frac{1}{L_1} \log \left( 1 + \frac{L_1 \epsilon}{L_0}
\right)$, consider the initialization $\vx_0 = m \mathbf{e}_1$, which by construction
satisfies $\|\nabla f(\vx_0)\| = \epsilon$. For the initialization, all of the
coordinates besides the first one are already at their optimal values, and the partial
gradient for these coordinates is zero; the stochastic gradient injects noise into the
second coordinate, so that $\vg_0 = \nabla f(\vx_0) \pm \sigma \mathbf{e}_2 = \epsilon
\mathbf{e}_1 \pm \sigma \mathbf{e}_2$. Based on the magnitude of $\eta$, this guarantees
$|\langle \vx_1, \mathbf{e}_2 \rangle| \geq m$, and consequently $|\nabla f(\vx_1)| \geq
|\psi'(\langle \vx_1, \mathbf{e}_2 \rangle)| \geq \psi'(m) = \epsilon$. Intuitively, the
size of $\eta$ causes the second coordinate to "jump" from the minimum at $0$ to another
point whose partial derivative is larger than $\epsilon$. This process continues with
$t$: at each step $t$, the stochastic gradient noise affects the coordinate indexed
$(t+2)$, so that this coordinate of the iterate jumps from $0$ to a point with magnitude
at least $m$. This guarantees that the algorithm does not reach a stationary point for
$d$ steps, and $d$ can be arbitrarily large. An important detail of this process is that
the coordinate-wise learning rates ensure that the length of each ``jump" (i.e., the
per-coordinate update size) does not decrease with $t$. This argument is made formal in
the following lemma.

\begin{lemma} \label{lem:adagrad_div}
Let $0 < \epsilon < \mathcal{O}(\Delta L_1)$. If $\eta \geq \Omega \left(
\frac{\gamma}{L_1 \sigma} \log \left( 1 + \frac{L_1 \epsilon}{L_0} \right) \right)$,
then for any $T \geq 1$, there exists some $f \in \gF$ such that Decorrelated AdaGrad
satisfies $\|\nabla f(\vx_t)\| \geq \epsilon$ for all $0 \leq t \leq T-1$.

Similarly, if $\eta \geq \Omega \left( \frac{1}{L_1} \log \left( 1 + \frac{L_1
\epsilon}{L_0} \right) \right)$ and $\gamma \leq \sigma$, then for any $T \geq 1$ there
exists some $f \in \gF$ such that AdaGrad satisfies $\|\nabla f(\vx_t)\| \geq \epsilon$
for all $0 \leq t \leq T-1$.
\end{lemma}

Notice that the $\eta$ requirement for AdaGrad in Lemma \ref{lem:adagrad_div} is
stronger than that of Decorrelated AdaGrad. As previously mentioned, this happens
because the size of AdaGrad's update is less sensitive to stochastic gradient noise than
Decorrelated AdaGrad.

\textbf{Slow convergence when $\eta \leq \Theta \left( \min \left\{ \gamma/(L_1 \sigma),
\gamma \epsilon/(L_0 \sigma) \right\} \right)$} In this case, the desired complexity
follows by analyzing the trajectory of each algorithm on a one-dimensional, linear
function with slope equal to $\epsilon$, similarly to existing lower bounds
\citep{zhang2019gradient, crawshaw2022robustness}.

\begin{lemma} \label{lem:adagrad_slow}
Let $0 < \epsilon < \mathcal{O} \left( \min \left\{ \Delta L_1, \sqrt{\Delta L_1 \sigma}
\right\} \right)$. If $\eta \leq \frac{\sqrt{2} \gamma}{L_1 \sigma} \log \left( 1 +
\frac{L_1 \epsilon}{L_0} \right)$, then there exists some $(f, g, \gD) \in
\gF_{\text{as}}(\Delta, L_0, L_1, \sigma)$ such that Decorrelated AdaGrad satisfies
$\|\nabla f(\vx_t)\| \geq \epsilon$ for all $t \leq \tilde{\mathcal{O}} \left( \Delta^2
L_0^2 \sigma^2 \gamma^{-2} \epsilon^{-4} + \Delta^2 L_1^2 \sigma^2 \gamma^{-2}
\epsilon^{-2} \right)$.

Similarly, if $\eta \leq \mathcal{O} \left( \frac{1}{L_1} \log \left( 1 +
\frac{L_1 \epsilon}{L_0} \right) \right)$, then there exists some $(f, g, \gD)
\in \gF_{\text{as}}(\Delta, L_0, L_1, \sigma)$ such that AdaGrad satisfies
$\|\nabla f(\vx_t)\| \geq \epsilon$ for all $t \leq \mathcal{O} \left( \Delta^2
L_0^2 \epsilon^{-4} + \Delta^2 L_1^2 \epsilon^{-2} \right)$.
\end{lemma}

Theorems \ref{thm:adagrad} and \ref{thm:vanilla_adagrad} can then be proven by combining
Lemmas \ref{lem:adagrad_div} and \ref{lem:adagrad_slow}.

\section{Single-Step Adaptive SGD} \label{sec:singlestep}
In this section, we consider single-step adaptive SGD (\Eqref{eq:singlestep}). Our lower
bound shows that, due to relaxed smoothness and affine noise, any algorithm of this type
will incur a higher-order dependence on $\Delta, L_1$. The results below are stated in
terms of constants $\gamma_i$ and a function $\zeta$, which are defined in terms of
$\delta$ and $\sigma_2$ (see Appendix \ref{app:singlestep_defs} for the definitions). In
the discussion following the theorem statement, we specify the limiting behavior of
these constants in terms of $\sigma_2, \delta$.

\begin{theorem} \label{thm:singlestep}
Denote $G = \tilde{\Theta}(\Delta L_1)$ and suppose $G \geq \sigma_1$. Let \(
    0 < \epsilon \leq \min \left\{ \sigma_1, \frac{G}{2}, \frac{G - \sigma_1}{\sigma_2 - 1}, \frac{\sqrt{\Delta L_0}}{\sqrt{2}} \right\}.
\)
Let algorithm $A_{\text{single}}$ denote single-step adaptive SGD
(\Eqref{eq:singlestep}) with any step size function $\alpha: \mathbb{R}^d
\rightarrow \mathbb{R}$ for a sufficiently large $d$, and let $\gF =
\gF_{\textup{aff}}(\Delta, L_0, L_1, \sigma_1, \sigma_2)$. If $\sigma_2 \geq 3$,
then \[
    \gT(A_{\text{single}}, \gF, \epsilon, \delta) \geq \tilde{\Omega} \left( \frac{\Delta L_0 \sigma_1^2}{\epsilon^4} + \frac{(\Delta L_1)^{2- \gamma_2 - \gamma_3} \sigma_1^{\gamma_2 + \gamma_3 - \gamma_1}}{\epsilon^{2-\gamma_1} } \right).
\]
Otherwise, if $1 < \sigma_2 < 3$, then \[
    \gT(A_{\text{single}}, \gF, \epsilon, \delta) \geq \tilde{\Omega} \left( \frac{\Delta L_0 \sigma_1^2}{\epsilon^4} + \frac{(\Delta L_1)^{2-\gamma_5-\gamma_6}}{\epsilon^{2-\gamma_4}} (\sigma_2-1)^2 \left( \frac{\sigma_1}{\sigma_2-1} \right)^{\gamma_5+\gamma_6-\gamma_4} \right).
\]
\end{theorem}

The proof is given in Appendix \ref{app:singlestep_proof}. Below, we specify the error
terms $\gamma_i$ in two regimes of $\sigma_2$.

\textbf{Large $\sigma_2$} For $\sigma_2 > 3$, the error terms $\gamma_1,
\gamma_2, \gamma_3$ satisfy: $\gamma_1, \gamma_3 = \Theta \left( \log \left( 1 +
\zeta(2/3, \delta) \right) \right)$ and $\gamma_2 = \Theta \left( \sigma_2^{-1}
\right)$, where $\zeta(p, \delta)$ is defined in \Eqref{eq:zeta_def}. Lemma
\ref{lem:random_walk_div} shows that $\lim_{\delta \rightarrow 0} \zeta(p,
\delta) = 0$ for all $p \in (0, 1)$, so when $\delta \rightarrow 0$, the lower
bound approaches \[
    \Omega \left( \frac{\Delta L_0 \sigma^2}{\epsilon^4} + \frac{\Delta^2 L_1^2 }{\epsilon^2 \log \left( 1 + \frac{\Delta L_1^2}{L_0} \right) } \left( \frac{\sigma_1}{\Delta L_1} \right)^{\Theta(1/\sigma_2)} \right).
\]
In this limiting case, the complexity has a nearly quadratic dependence on $\Delta,
L_1$, but only in the non-dominating term. Still, we emphasize the generality of our
result, which applies for any adaptive SGD algorithm whose learning rate only depends on
the current gradient, and shows that adaptivity based on the current gradient alone will
incur higher-order dependencies on $\Delta, L_1$.

\textbf{Small $\sigma_2$} Existing lower bounds that utilize similar
constructions of a biased random walk under affine
noise \citep{faw2023beyond,crawshaw2023federated} require that $\sigma_2$ be
bounded away from $1$. Our Theorem \ref{thm:singlestep} covers the case that
$\sigma_2 \rightarrow 1$, albeit with a lower bound that approaches $0$ when
$\sigma_2 \rightarrow 1$. The error terms $\gamma_4, \gamma_5, \gamma_6$ depend
only on $\delta, \sigma_2$ and satisfy: $\lim_{\delta \rightarrow 0} \gamma_4 = 0, \quad \lim_{\sigma_2 \rightarrow 1} \gamma_5 = 1, \quad \lim_{\sigma_2 \rightarrow 1} \gamma_6 = 0, \quad \lim_{\delta \rightarrow 0} \gamma_6 = 0.$
Note that $\gamma_5$ does not depend on $\delta$, and $\gamma_5 < 1$. Therefore, letting
$\delta \rightarrow 0$ yields a lower bound of \[
    \Omega \left( \frac{\Delta L_0 \sigma^2}{\epsilon^4} + \frac{\Delta L_1 \sigma_1 (\sigma_2-1)}{\epsilon^2 \log \left( 1 + \frac{\Delta L_1^2}{L_0} \right)} \left( \frac{\Delta L_1 (\sigma_2-1)}{\sigma_1} \right)^{1-\gamma_5} \right).
\]
Since $1-\gamma_5 \rightarrow 0$ as $\sigma_2 \rightarrow 1$, the second term in the
lower bound goes to $0$ when $\sigma_2 \rightarrow 1$. This shows that our construction
relies on $\sigma_2$ bounded away from $1$, and raises the question whether a
single-step adaptive SGD algorithm can converge without a quadratic dependence on
$\Delta, L_1$ in the regime $\sigma_2 < 1$.

\subsection{Proof Outline} \label{sec:singlestep_sketch}
The full proof of Theorem \ref{thm:singlestep} can be found in Appendix
\ref{app:singlestep_proof}, and we provide a sketch of the main ideas here. The
proof of Theorem \ref{thm:singlestep} has three main steps:

\textbf{Step 1} If the step size function $\alpha$ does not satisfy $0 \leq
\alpha(\vg) \leq \tilde{\mathcal{O}} \left( 1/(L_1 \|\vg\|) \right)$ for every
$\vg$ with $\|\vg\| \in [\epsilon, \sigma_1 + (\sigma_2 + 1) \Delta L_1]$, then
$A_{\text{single}}$ will diverge for some exponential $f_{\text{exp}}$, proven
in Lemma \ref{lem:app_singlestep_large_lr}. The difficult objective is
constructed similarly as in Lemma \ref{lem:adagrad_norm_div}, pictured in Figure
\ref{fig:adagrad_norm_div}.

\textbf{Step 2} If there exist $\vg_1, \vg_2 \in \mathbb{R}^d$ such that $\vg_2
= c \vg_1$ for some $c < 0$, and $\|\vg_1\| \leq \gO(\|\vg_2\|)$, but
$\alpha(\vg_1) \|\vg_1\| \geq \Omega(\alpha(\vg_2) \|\vg_2\|)$ (i.e., a ``tricky
pair", see Definition \ref{def:tricky_pair} in Appendix
\ref{app:singlestep_defs}), then there is some $(f, g, \gD) \in \gF$ for which
$A_{\text{single}}$ will diverge. The construction is based on the idea that
$\vg_1$ and $\vg_2$ are stochastic gradients at a given point, where $\vg_2$
points towards the minimum and $\vg_1$ points away from the minimum, but
$\alpha(\vg_1)$ is close enough to $\alpha(\vg_2)$ that $A_{\text{single}}$ has
nearly equal expected movement in each direction. In this case, the sequence
$\{\vx_t\}$ follows a biased random walk that will diverge with probability at
least $\delta$. This argument is made formal in Lemma
\ref{lem:app_singlestep_tricky_diverge}.

\textbf{Step 3} If neither of the above cases hold, then $\alpha(\vg) \leq
\tilde{\mathcal{O}} \left( 1/(L_1 \|\vg\|) \right)$ and there do not exist any tricky
pairs. The non-existence of tricky pairs means that $\alpha(\vg) \|\vg\|$ grows
sufficiently fast in terms of $\|\vg\|$ when $\|\vg\| \in [\epsilon, \sigma_1 +
(\sigma_2 + 1) \Delta L_1]$. In order for $\alpha$ to respect $\alpha(\vg) \leq
\tilde{\mathcal{O}} \left( 1/(L_1 \|\vg\|) \right)$ while also growing quickly, it must
be that $\alpha(\vg)$ is small whenever $\|\vg\|$ is small. Lemma
\ref{lem:app_singlestep_notricky_lr_ub} formalizes this idea to show an upper bound for
$\alpha(\vg)$ whenever $\|\vg\| = \epsilon$. The final bound follows by analyzing the
trajectory of $A_{\text{single}}$ for a piecewise linear objective with gradient $\vg$
satisfying $\|\vg\| = \epsilon$, since the convergence rate is inversely proportional to
$\alpha(\vg)$. (Lemma \ref{lem:app_singlestep_linear_small_lr})

\section{Discussion and Conclusion} \label{sec:discussion}
It has been stated in the optimization literature
\citep{woodworth2018graph,woodworth2021min} that a complexity lower bound should
not be interpreted as an unquestionable limit of performance, but rather as a
tool to examine the assumptions that led to the bound and explore alternatives.
In this spirit, we consider the implications of our choice of problem
formulation.

First, the optimization problem investigated in this paper (i.e., problem
instances satisfying Assumptions \ref{ass:obj} and \ref{ass:noise}) may not be a
sufficient theoretical framework to explain the behavior of adaptive
optimization algorithms in deep learning. A complete explanation of this type
may require additional assumptions about the structure of the objective
functions, such as enforcing a neural network architecture or a particular data
distribution.

Also, the negative results presented in this paper may be bypassed by algorithms
other than those we have considered. In particular, it is possible that
higher-order polynomial dependence on problem parameters can be avoided by more
practical algorithms such as Adam and AdamW. Presently, it is unknown whether
these algorithms can recover the optimal complexity of the smooth case (as does
SGD with clipping), or if they behave more like the AdaGrad variants considered
in this paper.

\textbf{Limitations} The most important limitation of our work is that our strongest
lower bounds (Theorem \ref{thm:adagrad_norm}, \ref{thm:adagrad}) are obtained for
decorrelated variants of AdaGrad, which are not commonly used in practice. We view these
decorrelated methods as a starting point for lower bounds of adaptive algorithms under
relaxed smoothness, similarly to early work \cite{li2019convergence} showing upper
bounds for decorrelated versions of adaptive algorithms. Also, our result for the
original AdaGrad (Theorem \ref{thm:vanilla_adagrad}) is weaker in terms of the
dependence on $\sigma$. It remains open whether this result can be improved, and whether
there is a fundamental difference in the complexity of AdaGrad compared to the
decorrelated variants. Further, even if the same complexity can be achieved by the
original AdaGrad, the existing upper bounds do not exactly match our lower bounds.
Therefore, it remains to exactly characterize the complexity of AdaGrad (and its
variants) by providing matching upper and lower bounds.

\subsubsection*{Acknowledgments}
We would like to thank the anonymous reviewers for their helpful comments. This work is
supported by the Institute for Digital Innovation fellowship, a ORIEI seed funding, an IDIA P3 fellowship from George Mason University, a Cisco Faculty Research Award, and NSF award \#2436217, \#2425687.

\bibliography{iclr2025_conference}
\bibliographystyle{iclr2025_conference}

\newpage
\appendix
\input{appendix}

\end{document}

%% file: math_commands.tex
\usepackage{amsmath,amsfonts,bm}
\usepackage{amsthm}

\newtheorem{theorem}{Theorem}
\newtheorem{lemma}{Lemma}

\newtheorem{assumption}{Assumption}

\newtheorem{definition}{Definition}

\def\eqref#1{equation~\ref{#1}}
\def\Eqref#1{Equation~\ref{#1}}

\def\floor#1{\lfloor #1 \rfloor}
\def\1{\bm{1}}

\def\vzero{{\bm{0}}}

\def\vg{{\bm{g}}}

\def\vw{{\bm{w}}}
\def\vx{{\bm{x}}}

\def\vy{{\bm{y}}}

\DeclareMathAlphabet{\mathsfit}{\encodingdefault}{\sfdefault}{m}{sl}
\SetMathAlphabet{\mathsfit}{bold}{\encodingdefault}{\sfdefault}{bx}{n}

\def\gD{{\mathcal{D}}}

\def\gF{{\mathcal{F}}}

\def\gO{{\mathcal{O}}}

\def\gT{{\mathcal{T}}}

\newcommand{\Eqmark}[2]{\stackrel{(#1)}{#2}}



%% file: appendix.tex
\section{Proof of Theorem \ref{thm:adagrad_norm}} \label{app:adagrad_norm_proof}

\begin{lemma}[Restatement of Lemma \ref{lem:adagrad_norm_div}] \label{lem:app_adagrad_norm_div}
Suppose that $\Delta L_1^2 \geq L_0$, and $\eta \geq \frac{1}{L_1}$, and
\begin{equation} \label{eq:adagrad_norm_div_gamma_cond}
    \gamma \leq \frac{\eta \Delta L_1^2}{8 \log \left( 1 + 48 \frac{\Delta L_1^2}{L_0} \right)}.
\end{equation}
Then there exists a problem instance $(f, g, \gD) \in \gF_{\text{as}}(\Delta, L_0, L_1,
0)$ such that $\|\nabla f(\vx_t)\| \geq \Delta L_1$ for all $t \geq 0$.
\end{lemma}

\begin{proof}
The difficult function $f$ will be piecewise linear and exponential, and constructed in
such a way that the gradients $\nabla f(\vx_t)$ increase at a rate of $\approx t^t$.
The rapid growth rate of the gradients ensures, even with the update
normalization, that the update size $\|\vx_{t+1} - \vx_t\|$ increases at every
step.

Recall the function $\psi: \mathbb{R} \rightarrow \mathbb{R}$ defined as \[
    \psi(x) = \frac{L_0}{L_1^2} \left( \exp(L_1 |x|) - L_1 |x| - 1 \right),
\]
with \[
    \psi'(x) = \text{sign}(x) \frac{L_0}{L_1} \left( \exp(L_1 |x|) - 1 \right).
\]
It is straightforward to verify that $\psi$ bounded from below by $0$,
continuously differentiable, and $(L_0, L_1)$-smooth.

The difficult function $f$ will be constructed in terms of the following:
\begin{align*}
    g_t &= \left( 576 (t+1) \log \left( 1 + \frac{\Delta L_1^2}{L_0} (t+1) \right) \right)^t \Delta L_1 \\
    m_t &= \frac{1}{L_1} \log \left( 1 + \frac{L_1 g_t}{L_0} \right) \\
    \ell_t &= \frac{\eta g_t}{\sqrt{\gamma^2 + \sum_{i=0}^{t-1} g_i^2}} \\
    d_t &= \sum_{i=0}^{t-1} \ell_i.
\end{align*}
For each $t \geq 0$, define $\phi_t: \mathbb{R} \rightarrow \mathbb{R}$ as:
\begin{equation*}
    \phi_t(x) = \begin{cases}
        \psi(x - m_t) & x \leq m_t + m_{t+1} \\
        g_t (x - m_t - m_{t+1}) + \psi(m_{t+1}) & x \in (m_t + m_{t+1}, \ell_t - 2 m_{t+1}) \\
        -\psi(x - \ell_t + m_{t+1}) + 2 \psi(m_{t+1}) + g_t (\ell_t - 3 m_{t+1} - m_t) & x \geq \ell_t - 2 m_{t+1}
    \end{cases}.
\end{equation*}
These functions are constructed to satisfy $\phi_t'(0) = \psi'(-m_t) = -g_t$ and
$\phi_t'(\ell_t) = \psi'(-m_{t+1}) = -g_{t+1}$. To see that this definition makes sense, we
should show that $\ell_t - 2 m_{t+1} \geq m_t + m_{t+1}$, so that the boundary of the first
piece is smaller than the boundary of the third piece. This is equivalent to: $\ell_t \geq
m_t + 3 m_{t+1}$. Using $\Delta L_1^2 \geq L_0$, the sequence $g_t$ is increasing, and
consequently so is $m_t$. Therefore it suffices to prove
\begin{align}
    \ell_t &\geq 4 m_{t+1} \nonumber \\
    \frac{\eta g_t}{\sqrt{\gamma^2 + \sum_{i=0}^{t-1} g_i^2}} &\geq \frac{4}{L_1} \log \left( 1 + \frac{L_1 g_{t+1}}{L_0} \right) \label{eq:adagrad_norm_div_cond},
\end{align}
We will prove \Eqref{eq:adagrad_norm_div_cond} separately for the cases $t=0$
and $t \geq 1$.

\paragraph{Case 1} $t=0$. In this case, the desired condition is \[
    \frac{\eta g_0}{\gamma} \geq \frac{4}{L_1} \log \left( 1 + \frac{L_1 g_1}{L_0} \right).
\]
The RHS of the above inequality can be bounded as
\begin{align*}
    \frac{4}{L_1} \log \left( 1 + \frac{L_1 g_1}{L_0} \right) &= 4 \log \left( 1 + \frac{L_1}{L_0} \left( 1152 \log \left( 1 + 2 \frac{\Delta L_1^2}{L_0} \right) \right) \Delta L_1 \right) \\
    &= \frac{4}{L_1} \log \left( 1 + 1152 \frac{\Delta L_1^2}{L_0} \log \left( 1 + 2 \frac{\Delta L_1^2}{L_0} \right) \right) \\
    &\Eqmark{i}{\leq} \frac{4}{L_1} \log \left( 1 + 2304 \left( \frac{\Delta L_1^2}{L_0} \right)^2 \right) \\
    &\Eqmark{ii}{\leq} \frac{8}{L_1} \log \left( 1 + 48 \frac{\Delta L_1^2}{L_0} \right) \\
    &\Eqmark{iii}{\leq} \frac{\Delta L_1 \eta}{\gamma},
\end{align*}
where $(i)$ uses $\log x \leq 1 + x$ for all $x > 0$, $(ii)$ uses $\log(1+x^n) \leq
\log((1+x)^n) \leq n \log(1+x)$ for all $x > 0$, and $(iii)$ uses the assumed condition
on $\gamma$ (\Eqref{eq:adagrad_norm_div_gamma_cond}). This concludes the first case.

\paragraph{Case 2} $t \geq 1$. We first simplify the denominator $\sqrt{\gamma^2
+ \sum_{i=0}^{t-1} g_i^2}$. First, the assumed condition on $\gamma$
(\Eqref{eq:adagrad_norm_div_gamma_cond}) implies that $\gamma \leq \eta \Delta
L_1^2 = \eta L_1 g_0$, so \[
    \gamma^2 + \sum_{i=0}^{t-1} g_i^2 \leq \eta^2 L_1^2 g_0^2 + \sum_{i=0}^{t-1} g_i^2 \leq (1 + \eta^2 L_1^2) \sum_{i=0}^{t-1} g_i^2 \leq 2 \eta^2 L_1^2 \sum_{i=1}^{t-1} g_i^2.
\]
Also,
\begin{align*}
    \sum_{i=0}^{t-2} g_i^2 &\leq (t-1) g_{t-2}^2 \\
    &= (t-1) \left( 576 (t-1) \log \left( 1 + \frac{\Delta L_1^2}{L_0} (t-1) \right) \right)^{2(t-2)} \Delta^2 L_1^2 \\
    &\leq (t-1) \left( 576 t \log \left( 1 + \frac{\Delta L_1^2}{L_0} t \right) \right)^{2(t-2)} \Delta^2 L_1^2 \\
    &\Eqmark{i}{\leq} \left( 576 t \log \left( 1 + \frac{\Delta L_1^2}{L_0} t \right) \right)^{2(t-1)} \Delta^2 L_1^2 \\
    &= g_{t-1}^2.
\end{align*}
Therefore $\gamma^2 + \sum_{i=0}^{t-1} g_i^2 \leq 4 \eta^2 L_1^2 g_{t-1}^2$. So
the LHS of \Eqref{eq:adagrad_norm_div_cond} can be bounded as
\begin{align}
    \frac{\eta g_t}{\sqrt{\gamma^2 + \sum_{i=0}^{t-1} g_i^2}} &\geq \frac{g_t}{2 L_1 g_{t-1}} \nonumber \\
    &= \frac{\left( 576 (t+1) \log \left( 1 + \frac{\Delta L_1^2}{L_0} (t+1) \right) \right)^t \Delta L_1}{2 L_1 \left( 576 t \log \left( 1 + \frac{\Delta L_1^2}{L_0} t \right) \right)^{t-1} \Delta L_1} \nonumber \\
    &= \frac{288 (t+1)}{L_1} \log \left( 1 + \frac{\Delta L_1^2}{L_0} (t+1) \right) \left( \frac{576 (t+1) \log \left( 1 + \frac{\Delta L_1^2}{L_0} (t+1) \right)}{576 t \log \left( 1 + \frac{\Delta L_1^2}{L_0} t \right)} \right)^{t-1} \nonumber \\
    &\geq \frac{288 (t+1)}{L_1} \log \left( 1 + \frac{\Delta L_1^2}{L_0} (t+1) \right). \label{eq:adagrad_norm_div_inter_1}
\end{align}
The RHS of \Eqref{eq:adagrad_norm_div_cond} can be bounded as
\begin{align*}
    \frac{4}{L_1} \log \left( 1 + \frac{L_1 g_{t+1}}{L_0} \right) &= \frac{4}{L_1} \log \left( 1 + \frac{\Delta L_1^2}{L_0} \left( 576 (t+2) \log \left( 1 + \frac{\Delta L_1^2}{L_0} (t+2) \right) \right)^{t+1} \right) \\
    &\Eqmark{i}{\leq} \frac{4(t+1)}{L_1} \log \left( 1 + 576 \left( \frac{\Delta L_1^2}{L_0} \right)^{1/(t+1)} (t+2) \log \left( 1 + \frac{\Delta L_1^2}{L_0} (t+2) \right) \right) \\
    &\Eqmark{ii}{\leq} \frac{4(t+1)}{L_1} \log \left( 1 + 576 \left( \frac{\Delta L_1^2}{L_0} \right)^{1 + 1/(t+1)} (t+2)^2 \right) \\
    &\Eqmark{iii}{\leq} \frac{4(t+1)}{L_1} \log \left( 1 + 576 \left( \frac{\Delta L_1^2}{L_0} \right)^2 (t+2)^2 \right) \\
    &\Eqmark{iv}{\leq} \frac{8(t+1)}{L_1} \log \left( 1 + 24 \frac{\Delta L_1^2}{L_0} (t+2) \right) \\
    &\Eqmark{v}{\leq} \frac{8(t+1)}{L_1} \log \left( 1 + 36 \frac{\Delta L_1^2}{L_0} (t+1) \right),
\end{align*}
where $(i)$ uses $\log(1+x^n) \leq \log((1+x)^n) = n \log(1+x)$ for $x > 0$,
$(ii)$ uses $\log(1+x) \leq x$ for all $x > 0$, $(iii)$ uses $\Delta L_1^2 \geq
L_0$, $(iv)$ again uses $\log(1+x^n) \leq n \log(1+x)$, and $(v)$ uses $t+2 \leq
\frac{3}{2}(t+1)$ since $t \geq 1$. Further,
\begin{align}
    \frac{4}{L_1} \log \left( 1 + \frac{L_1 g_{t+1}}{L_0} \right) &\Eqmark{i}{\leq} \frac{8(t+1)}{L_1} \left( \log \left( 1 + \frac{\Delta L_1^2}{L_0} (t+1) \right) + \log(37) \right) \nonumber \\
    &= \frac{8(t+1)}{L_1} \log \left( 1 + \frac{\Delta L_1^2}{L_0} (t+1) \right) + 8 \log(37) (t+1) \nonumber \\
    &\Eqmark{ii}{\leq} \frac{8 (1 + \log(37)) (t+1)}{L_1} \log \left( 1 + \frac{\Delta L_1^2}{L_0} (t+1) \right) \nonumber \\
    &\leq \frac{288 (t+1)}{L_1} \log \left( 1 + \frac{\Delta L_1^2}{L_0} (t+1) \right), \label{eq:adagrad_norm_div_inter_2}
\end{align}
where $(i)$ uses $\log(1+ab) \leq \log((1+a)(1+b)) \leq \log(1+a) + \log(1+b)$
for all $a, b > 0$, and $(ii)$ uses $\Delta L_1^2 \geq L_0$ and $t \geq 1$.
Combining \Eqref{eq:adagrad_norm_div_inter_1} and
\Eqref{eq:adagrad_norm_div_inter_2} proves \Eqref{eq:adagrad_norm_div_cond}.

This proves that the definition of $\phi_t$ makes sense for all $t$. We can finally
define the difficult objective $f$ as follows: \[
    f(x) = \phi_{j(x)}(x - d_{j(x)}) + \sum_{i=0}^{j(x)-1} \phi_i(\ell_i),
\]
where \[
    j(x) = \begin{cases}
        \max \left\{ t \geq 0 \;|\; d_t \leq x \right\} & x \geq 0 \\
        0 & x < 0
    \end{cases}.
\]
With this definition, $f$ is essentially a piece-wise function, where each piece
is an interval $[d_t, d_{t+1}]$ whose function value is a translation of
$\phi_t$. $f$ is informally pictured in Figure \ref{fig:adagrad_norm_div} of the
main text. Notice that $f$ is continuous and differentiable within each piece.
At the boundary of each piece,
\begin{align*}
    \lim_{x \rightarrow d_{t+1}^-} f(x) &= \lim_{x \rightarrow d_{t+1}^-} \phi_t(x - d_t) + \sum_{i=0}^{t-1} \phi_i(\ell_i) \\
    &= \phi_t(d_{t+1} - d_t) + \sum_{i=0}^{t-1} \phi_i(\ell_i) = \phi_t(\ell_t) + \sum_{i=0}^{t-1} \phi_i(\ell_i) = \sum_{i=0}^t \phi_i(\ell_i), \\
    \lim_{x \rightarrow d_{t+1}^+} f(x) &= \lim_{x \rightarrow d_{t+1}^+} \phi_{t+1}(x - d_{t+1}) + \sum_{i=0}^t \phi_i(\ell_i) = \phi_{t+1}(0) + \sum_{i=0}^t \phi_i(\ell_i) = \sum_{i=0}^t \phi_i(\ell_i).
\end{align*}
Also
\begin{align*}
    \lim_{x \rightarrow d_{t+1}^-} f'(x) &= \lim_{x \rightarrow d_{t+1}^-} \phi_t'(x - d_t) = \phi_t'(d_{t+1} - d_t) = \phi_t'(\ell_t) = -g_{t+1} \\
    \lim_{x \rightarrow d_{t+1}^+} f'(x) &= \lim_{x \rightarrow d_{t+1}^+} \phi_{t+1}'(x - d_{t+1}) = \phi_{t+1}'(0) = -g_{t+1}.
\end{align*}
Therefore $f$ is differentiable everywhere. Also, \[
    \inf_x f(x) = \inf_{t \geq 0} \left\{ \inf_{x \in [0, \ell_t]} \phi_t(x) + \sum_{i=0}^{t-1} \phi_i(\ell_i) \right\} \geq \inf_{t \geq 0} \inf_{x \in [0, \ell_t]} \phi_t(x) = 0.
\]
The initial point $x_0 = 0$ satisfies \[
    f(x_0) = \phi_0(0) = \psi(-m_0) = \frac{g_0}{L_1} - \frac{L_0}{L_1^2} \log \left( 1 + \frac{L_1 g_0}{L_0} \right) \leq \frac{g_0}{L_1} = \Delta.
\]
Therefore $f(x_0) - \inf_x f(x) \leq \Delta$. Since each $\phi_t$ is $(L_0,
L_1)$-smooth, so is $f$.

We will use a stochastic gradient $g, \gD$ for this function which is always equal to
the true gradient, so that the noise conditions are trivially satisfied. Therefore $(f,
g, \gD) \in \gF_{\text{as}}(\Delta, L_0, L_1, 0)$.

Now, consider the trajectory when starting from the initial point $x_0 = 0$. We claim
that $x_t = d_t$ for all $t \geq 0$, which we will prove by induction. The base case
$t=0$ holds by construction. So suppose that $x_i = d_i$ for all $0 \leq i \leq t$. Then
$f'(x_i) = f'(d_i) = -g_i$ for all $i$. So \[
    x_{t+1} = x_t - \frac{\eta f'(x_t)}{\sqrt{\gamma^2 + \sum_{i=0}^{t-1} \left( f'(x_i) \right)^2}} = d_t + \frac{\eta g_t}{\sqrt{\gamma^2 + \sum_{i=0}^{t-1} g_i^2}} = d_t + \ell_t = d_{t+1}.
\]
This completes the induction.

Therefore, for all $t \geq 0$, we have $|f'(x_t)| = g_t \geq g_0 = \Delta L_1$.
\end{proof}

The following lemma uses a difficult objective which is adapted from Theorem 2
of \cite{drori2020complexity}.

\begin{lemma}[Restatement of Lemma \ref{lem:adagrad_norm_slow}] \label{lem:app_adagrad_norm_slow}
Let \[
    T = 1 + \frac{\Delta^2 L_1^2 \sigma^2}{144 \epsilon^4} + \frac{\Delta L_0 \sigma^2 \log(1 + \sigma^2/\gamma^2)}{24 \epsilon^4} + \frac{\Delta^2 L_1^2}{144 \epsilon^2}.
\]
and suppose $d \geq T$ and $\epsilon \leq \min \left\{ \frac{\sqrt{2}}{3} \sqrt{\Delta
L_0}, \frac{1}{\sqrt{3}} \sqrt{\Delta L_1 \gamma} \right\}$. If $\eta \leq
\frac{1}{L_1}$, then there exists some $(f, g, \gD) \in \gF_{\text{as}}(\Delta, L_0,
L_1, \sigma)$ such that $\|\nabla f(\vx_t)\| = \epsilon$ for all $0 \leq t \leq T-1$.
\end{lemma}

\begin{proof}
Let $d \geq T$, and define $f: \mathbb{R}^d \rightarrow \mathbb{R}$ as: \[
    f(\vx) = \epsilon \langle \vx, \mathbf{e}_1 \rangle + \sum_{i=2}^T h_i(\langle \vx_t, \mathbf{e}_i \rangle),
\]
where
\begin{align*}
    h_i(x) &= \begin{cases}
        \frac{L_0}{2} x^2 & |x| < \frac{a_i}{2} \\
        -\frac{L_0}{2} (x-a_i)^2 + \frac{L_0}{4} a_i^2 & |x| \in \left[ \frac{a_i}{2}, a_i \right] \\
        \frac{L_0}{4} a_i^2 & |x| > a_i
    \end{cases} \\
    a_i &= \alpha_i \sigma \\
    \alpha_i &= \frac{\eta}{\sqrt{\gamma^2 + (i-2) (\epsilon^2 + \sigma^2)}}.
\end{align*}
To see that $f$ is $(L_0, L_1)$-smooth, notice that each $h_i$ is $L_0$-smooth.
Therefore, for any $\vx, \vy \in \mathbb{R}^d$,
\begin{align*}
    \|\nabla f(\vx) - \nabla f(\vy)\|^2 &= \left( \nabla_1 f(\vx) - \nabla_1 f(\vy) \right)^2 + \sum_{i=2}^d \left( \nabla_i f(\vx) - \nabla_i f(\vy) \right)^2 \\
    &= \sum_{i=2}^d \left( h_i'(x_i) - h_i'(y_i) \right)^2 \\
    &\leq L^2 \sum_{i=2}^d (x_i - y_i)^2 \\
    &\leq L^2 \|\vx-\vy\|^2.
\end{align*}
Therefore $f$ is $L_0$-smooth, and consequently is also $(L_0, L_1)$-smooth. We will
also define the following stochastic gradient for $f$: \[
    F(\vx, \xi) = \nabla f(\vx) + (2 \xi - 1) \sigma \mathbf{e}_{j(\vx)},
\]
where \[
    j(\vx) = \begin{cases}
        T & \langle \vx, \mathbf{e}_i \rangle \neq 0 \text{ for all $i$ with } 2 \leq i \leq d \\
        \min \left\{ 2 \leq i \leq d \;|\; \langle \vx, \mathbf{e}_i \rangle = 0 \right\} & \text{otherwise}
    \end{cases}
\]
This oracle is defined so that the stochastic gradient noise at step $t$ only affects
coordinate $t+2$ (this will be shown later). Let $\gD$ be the distribution of $\xi$,
defined as $P(\xi = 0) = P(\xi = 1) = \frac{1}{2}$. With this definition, the stochastic
gradient $F$ satisfies
\begin{align*}
    \mathbb{E} [F(\vx, \xi)] &= \nabla f(\vx) \\
    \|F(\vx, \xi) - \nabla f(\vx)\| &\leq \sigma \quad \text{(almost surely)}.
\end{align*}
Therefore, all of the conditions for $(f, F, \gD) \in \gF_{\text{as}}(\Delta, L_0, L_1,
\sigma)$ are satisfied other than the condition that $f$ is bounded from below and
$f(\vx_0) - \inf_{\vx} f(\vx) \leq \Delta$. This condition will be addressed at the end
of this lemma's proof.

Now consider the trajectory when optimizing $f$ from the starting point $\vx_0 =
\mathbf{0}$. We claim that, for each $0 \leq t \leq T-1$, the iterate $\vx_t$ satisfies
the following conditions:
\begin{align}
    \langle \vx_t, \mathbf{e}_1 \rangle &= -\epsilon \sum_{i=2}^{t+1} \alpha_i \label{eq:adagrad_norm_inductive_1} \\
    |\langle \vx_t, \mathbf{e}_j \rangle| &= a_j \quad \text{ for all } 2 \leq j \leq t+1 \label{eq:adagrad_norm_inductive_2} \\
    \langle \vx_t, \mathbf{e}_j \rangle &= 0 \quad \text{ for all } j > t+1. \label{eq:adagrad_norm_inductive_3}
\end{align}
We will prove this claim by induction. The base case $t=0$ holds since $\vx_0 =
\mathbf{0}$. So suppose that for some $0 \leq t \leq T-2$ the claim holds for all $0
\leq i \leq t$. Then for each such $i$,
\begin{align*}
    \nabla_1 f(\vx_i) &= \epsilon \\
    \nabla_j f(\vx_i) &= h_j'(\langle \vx_i, \mathbf{e}_j) \Eqmark{i}{=} h_j'(a_j) \Eqmark{iii}{=} 0 \quad \text{ for all } 2 \leq j \leq i+1 \\
    \nabla_j f(\vx_i) &= h_j'(\langle \vx_i, \mathbf{e}_j) \Eqmark{ii}{=} h_j'(0) \Eqmark{iv}{=} 0 \quad \text{ for all } j > i+1,
\end{align*}
where $(i)$ uses \Eqref{eq:adagrad_norm_inductive_2} from the inductive hypothesis together
with the fact that $h_j'(a_j) = h_j'(-a_j) = 0$, $(ii)$ uses
\Eqref{eq:adagrad_norm_inductive_3} from the inductive hypothesis, and both $(iii)$ and
$(iv)$ use the definition of $h_j$. Therefore $\nabla f(\vx_i) = \epsilon \mathbf{e}_1$.
Also, $j(\vx) = i+2$. From the definition of the stochastic gradient oracle,
\begin{align*}
    F(\vx_i, \xi_i) &= \nabla f(\vx_i) + (2 \xi_i - 1) \sigma \mathbf{e}_{j(\vx_i)} \\
    &= \epsilon \mathbf{e}_1 + (2 \xi_t - 1) \sigma \mathbf{e}_{i+2}.
\end{align*}
Therefore, the update from $\vx_t$ to $\vx_{t+1}$ only affects coordinates with index
$1$ and index $t+2$. Further, the above implies $\|F(\vx_i, \xi_i)\|^2 = \epsilon^2 +
\sigma^2$. Therefore, the effective learning rate of the algorithm at step $t$ is \[
    \eta_t = \frac{\eta}{\sqrt{\gamma^2 + \sum_{i=0}^{t-1} \|F(\vx_i, \xi_i)\|^2}} = \frac{\eta}{\sqrt{\gamma^2 + t(\epsilon^2 + \sigma^2)}} = \alpha_{t+2}.
\]
We can then verify the inductive hypothesis for step $t+1$ by considering the
coordinates of $\vx_{t+1}$:
\begin{align*}
    \langle \vx_{t+1}, \mathbf{e}_1 \rangle &= \langle \vx_t - \eta_t F(\vx_t, \xi_t), \mathbf{e}_1 \rangle \\
    &= \langle \vx_t, \mathbf{e}_1 \rangle - \eta_t \langle F(\vx_t, \xi_t), \mathbf{e}_1 \rangle \\
    &= \langle \vx_t, \mathbf{e}_1 \rangle - \epsilon \alpha_{t+2} \\
    &\Eqmark{i}{=} -\epsilon \left( \sum_{i=2}^{t+1} \alpha_t \right) - \epsilon \alpha_{t+2} \\
    &= -\epsilon \sum_{i=1}^{t+2} \alpha_t,
\end{align*}
where $(i)$ uses \Eqref{eq:adagrad_norm_inductive_1} from the inductive hypothesis, and this
completes the inductive step for \Eqref{eq:adagrad_norm_inductive_1}. For
\Eqref{eq:adagrad_norm_inductive_2}, we separately consider $j \leq t+1$ and $j=t+2$. For $j
\leq t+1$, \[
    |\langle \vx_{t+1}, \mathbf{e}_j \rangle| = |\langle \vx_t, \mathbf{e}_j \rangle - \eta_t \langle F(\vx_t, \xi_t), \mathbf{e}_j \rangle| = |\langle \vx_t, \mathbf{e}_j \rangle| \Eqmark{i}{=} a_j,
\]
where $(i)$ uses \Eqref{eq:adagrad_norm_inductive_2} from the inductive hypothesis. For $j = t+2$: \[
    |\langle \vx_{t+1}, \mathbf{e}_{t+2} \rangle| = |\langle \vx_t, \mathbf{e}_{t+2} \rangle - \eta_t \langle F(\vx_t, \xi_t), \mathbf{e}_{t+2} \rangle| \Eqmark{i}{=} \eta_t \sigma = \alpha_{t+2} \sigma = a_{t+2},
\]
where $(i)$ uses \Eqref{eq:adagrad_norm_inductive_3} from the inductive hypothesis. This
completes the inductive step for \Eqref{eq:adagrad_norm_inductive_2}. For
\Eqref{eq:adagrad_norm_inductive_3}, we consider $j > t+1$: \[
    \langle \vx_{t+1}, \mathbf{e}_j \rangle = \langle \vx_t, \mathbf{e}_j \rangle - \eta_t \langle F(\vx_t, \xi_t), \mathbf{e}_j \rangle = 0,
\]
where the last equality uses \Eqref{eq:adagrad_norm_inductive_3}. This completes the
inductive step for \Eqref{eq:adagrad_norm_inductive_3}, and consequently completes the
induction. As a result, we have that $\nabla f(\vx_t) = \epsilon$ for all $0 \leq t \leq
T-1$.

The only remaining detail is whether the objective $f$ satisfies the condition $f(\vx_0)
- \inf_{\vx} f(\vx) \leq \Delta$. Actually, $f$ does not satisfy this condition because
$f$ is not even lower bounded, due to the linear term $\epsilon \langle \vx,
\mathbf{e}_1 \rangle$. Similarly to \cite{drori2020complexity}, we instead argue that
there exists a lower bounded function $\hat{f}$ that has the same first-order
information as $f$ at all of the points $\vx_t$ for $0 \leq t \leq T-1$. If this
happens, then the behavior of $A$ when optimizing $\hat{f}$ is the same as that of $A$
when optimizing $f$, so the conclusion $\|\nabla \hat{f}(\vx_t)\| = \epsilon$ still
holds. Specifically, we need $\hat{f}$ which is lower bounded and that satisfies: \[
    \nabla \hat{f}(\vx_t) = \nabla f(\vx_t), \quad \hat{f}(\vx_t) = f(\vx_t)
\]
for all $0 \leq t \leq T$. The existence of such an $\hat{f}$ follows immediately from
Lemma 1 of \cite{drori2020complexity}, and this $\hat{f}$ satisfies \[
    \inf_{\vx} \hat{f}(\vx) \geq \min_{0 \leq t \leq T-1} f(\vx_t) - \frac{3 \epsilon^2}{2L_0},
\]
so that
\begin{equation} \label{eq:adagrad_norm_hatf_delta}
    \hat{f}(\vx_0) - \inf_{\vx} \hat{f}(\vx) \leq \frac{3 \epsilon^2}{2 L_0} + \max_{0 \leq t \leq T-1} -f(\vx_t).
\end{equation}
Recall that $f(\vx_0) = 0$. For all $t \geq $, we can write each $-f(\vx_t)$ as:
\begin{align}
    -f(\vx_t) &= -\epsilon \langle \vx_t, \mathbf{e}_1 \rangle - \sum_{i=2}^T h_i(\langle \vx_t, \mathbf{e}_i \rangle) \nonumber \\
    &\Eqmark{i}{=} \epsilon^2 \sum_{i=2}^{t+1} \alpha_i - \sum_{i=2}^{t+1} h_i(a_i) - \sum_{i=t+2}^T h_i(0) \nonumber \\
    &\Eqmark{ii}{=} \epsilon^2 \sum_{i=2}^{t+1} \alpha_i - \frac{L_0}{4} \sum_{i=2}^{t+1} a_i^2 \nonumber \\
    &= \epsilon^2 \eta \underbrace{\sum_{i=0}^{t-1} \frac{1}{\sqrt{\gamma^2 + i (\epsilon^2 + \sigma^2)}}}_{S_1} - \frac{L_0 \sigma^2}{4} \eta^2 \underbrace{\sum_{i=0}^{t-1} \frac{1}{\gamma^2 + i (\epsilon^2 + \sigma^2)}}_{S_2}, \label{eq:adagrad_norm_delta_inter}
\end{align}
where $(i)$ uses \Eqref{eq:adagrad_norm_inductive_1}, \Eqref{eq:adagrad_norm_inductive_2}, and
\Eqref{eq:adagrad_norm_inductive_3}, and $(ii)$ uses the definition of $h_i$. We can
bound $S_1$ as follows:
\begin{align}
    S_1 &= \sum_{i=0}^{t-1} \frac{1}{\sqrt{\gamma^2 + i (\epsilon^2 + \sigma^2)}} = \frac{1}{\gamma} + \sum_{i=1}^{t-1} \frac{1}{\sqrt{\gamma^2 + i (\epsilon^2 + \sigma^2)}} \nonumber \\
    &\leq \frac{1}{\gamma} + \int_0^{t-1} \frac{1}{\sqrt{\gamma^2 + x (\epsilon^2 + \sigma^2)}} ~dx \Eqmark{i}{=} \frac{1}{\gamma} + \frac{1}{\epsilon^2 + \sigma^2} \int_{\gamma^2}^{\gamma^2 + (t-1)(\epsilon^2 + \sigma^2)} \frac{1}{\sqrt{u}} ~du \nonumber \\
    &= \frac{1}{\gamma} + \frac{2}{\epsilon^2 + \sigma^2} \left( \sqrt{\gamma^2 + (t-1)(\epsilon^2 + \sigma^2)} - \gamma \right) \nonumber \\
    &= \frac{1}{\gamma} + \frac{2}{\epsilon^2 + \sigma^2} \left( \sqrt{\gamma^2 + (t-1)(\epsilon^2 + \sigma^2)} - \gamma \right) \frac{\sqrt{\gamma^2 + (t-1)(\epsilon^2 + \sigma^2)} + \gamma}{\sqrt{\gamma^2 + (t-1)(\epsilon^2 + \sigma^2)} + \gamma} \nonumber \\
    &= \frac{1}{\gamma} + \frac{2}{\epsilon^2 + \sigma^2} \frac{(t-1)(\epsilon^2 + \sigma^2)}{\sqrt{\gamma^2 + (t-1)(\epsilon^2 + \sigma^2)} + \gamma} = \frac{1}{\gamma} + \frac{2(t-1)}{\sqrt{\gamma^2 + (t-1)(\epsilon^2 + \sigma^2)} + \gamma} \nonumber \\
    &\leq \frac{1}{\gamma} + \frac{2(t-1)}{\sqrt{(t-1)(\epsilon^2 + \sigma^2)}} = \frac{1}{\gamma} + \frac{2 \sqrt{t-1}}{\sqrt{\epsilon^2 + \sigma^2}}, \label{eq:adagrad_norm_S1_ub}
\end{align}
where $(i)$ uses the substitution $u = \gamma^2 + x(\epsilon^2 + \sigma^2)$.
Similarly for $S_2$:
\begin{align}
    S_2 &= \sum_{i=0}^{t-1} \frac{1}{\gamma^2 + i (\epsilon^2 + \sigma^2)} \geq \int_0^t \frac{1}{\gamma^2 + x (\epsilon^2 + \sigma^2)} ~dx \nonumber \\
    &\Eqmark{i}{=} \frac{1}{\epsilon^2 + \sigma^2} \int_{\gamma^2}^{\gamma^2 + t (\epsilon^2 + \sigma^2)} \frac{1}{u} ~du = \frac{1}{\epsilon^2 + \sigma^2} \log \left( \frac{\gamma^2 + t (\epsilon^2 + \sigma^2)}{\gamma^2} \right) \nonumber \\
    &= \frac{1}{\epsilon^2 + \sigma^2} \log \left( 1 + \frac{t (\epsilon^2 + \sigma^2)}{\gamma^2} \right) \geq \frac{1}{\epsilon^2 + \sigma^2} \log \left( 1 + \frac{t \sigma^2}{\gamma^2} \right), \label{eq:adagrad_norm_S2_lb}
\end{align}
where $(i)$ uses the substitution $u = \gamma^2 + x(\epsilon^2 + \sigma^2)$.
Plugging \Eqref{eq:adagrad_norm_S1_ub} and \Eqref{eq:adagrad_norm_S2_lb} into \Eqref{eq:adagrad_norm_delta_inter}:
\begin{align*}
    -f(\vx_t) &\leq \epsilon^2 \left( \frac{1}{\gamma} + \frac{2 \sqrt{t-1}}{\sqrt{\epsilon^2 + \sigma^2}} \right) \eta - \frac{L_0 \sigma^2 \log \left( 1 + \frac{t \sigma^2}{\gamma^2} \right)}{4 (\epsilon^2 + \sigma^2)} \eta^2 \\
    &\leq \epsilon^2 \left( \frac{1}{\gamma} + \frac{2 \sqrt{T-1}}{\sqrt{\epsilon^2 + \sigma^2}} \right) \eta - \frac{L_0 \sigma^2 \log \left( 1 + \frac{\sigma^2}{\gamma^2} \right)}{4 (\epsilon^2 + \sigma^2)} \eta^2
\end{align*}
We can decompose $T = 1 + T_1 + T_2$, where \[
    T_1 = \frac{\Delta^2 L_1^2 \sigma^2}{144 \epsilon^4} + \frac{\Delta^2 L_1^2}{144 \epsilon^2}, \quad T_2 = \frac{\Delta L_0 \sigma^2 \log(1 + \sigma^2/\gamma^2)}{24 \epsilon^4}.
\]
Then $\sqrt{T-1} = \sqrt{T_1 + T_2} \leq \sqrt{T_1} + \sqrt{T_2}$, so
\begin{equation}
    -f(\vx_t) \leq \underbrace{\epsilon^2 \left( \frac{1}{\gamma} + \frac{2 \sqrt{T_1}}{\sqrt{\epsilon^2 + \sigma^2}} \right) \eta}_{D_1} + \underbrace{\frac{2 \epsilon^2 \sqrt{T_2}}{\sqrt{\epsilon^2 + \sigma^2}} \eta - \frac{L_0 \sigma^2 \log \left( 1 + \frac{\sigma^2}{\gamma^2} \right)}{4 (\epsilon^2 + \sigma^2)} \eta^2}_{D_2}. \label{eq:adagrad_norm_delta_inter_2}
\end{equation}
We can bound $D_1$ and $D_2$ separately:
\begin{align*}
    D_1 &= \left( \frac{\epsilon^2}{\gamma} + \frac{2 \epsilon^2 \sqrt{T_1}}{\sqrt{\epsilon^2 + \sigma^2}} \right) \eta \\
    &\Eqmark{i}{\leq} \frac{\epsilon^2}{\gamma L_1} + \frac{2 \epsilon^2 \sqrt{T_1}}{L_1 \sqrt{\epsilon^2 + \sigma^2}} \\
    &\Eqmark{ii}{\leq} \frac{\Delta}{6} + \frac{2 \epsilon^2 \sqrt{T_1}}{L_1 \sqrt{\epsilon^2 + \sigma^2}} \\
    &\Eqmark{iii}{=} \frac{\Delta}{6} + \frac{2 \epsilon^2}{L_1 \sqrt{\epsilon^2 + \sigma^2}} \sqrt{\frac{\Delta^2 L_1^2 \sigma^2}{144 \epsilon^4} + \frac{\Delta^2 L_1^2}{144 \epsilon^2}} \\
    &\Eqmark{iv}{\leq} \frac{\Delta}{6} + \frac{2 \epsilon^2}{L_1 \sqrt{\epsilon^2 + \sigma^2}} \left( \frac{\Delta L_1 \sigma}{12 \epsilon^2} + \frac{\Delta L_1}{12 \epsilon} \right) \\
    &= \frac{\Delta}{6} + \frac{\Delta \sigma}{6 \sqrt{\epsilon^2 + \sigma^2}} + \frac{\Delta \epsilon}{6 \sqrt{\epsilon^2 + \sigma^2}} \\
    &\leq \frac{\Delta}{6} + \frac{\Delta}{6} + \frac{\Delta}{6} = \frac{\Delta}{3},
\end{align*}
where $(i)$ uses the condition $\eta \leq 1/L_1$, $(ii)$ uses the condition
$\epsilon \leq \frac{1}{\sqrt{6}} \sqrt{\gamma \Delta L_1}$, $(iii)$ uses the
definition of $T_1$, and $(iv)$ uses $\sqrt{a+b} \leq \sqrt{a} + \sqrt{b}$.
Notice that $D_2$ is a quadratic function of $\eta$ with negative leading
coefficient, so $D_2$ is upper bounded by the vertex of the corresponding
parabola, i.e. $ax^2 + bx \leq -b^2/2a$ when $a < 0$. Therefore
\begin{align*}
    D_2 &\leq \left( \frac{2 \epsilon^2 \sqrt{T_2}}{\sqrt{\epsilon^2 + \sigma^2}} \right)^2 \frac{2 (\epsilon^2 + \sigma^2)}{L_0 \sigma^2 \log \left( 1 + \frac{\sigma^2}{\gamma^2} \right)} \\
    &= \frac{8 \epsilon^4}{L_0 \sigma^2 \log \left( 1 + \frac{\sigma^2}{\gamma^2} \right)} T_2 \\
    &\Eqmark{i}{=} \frac{8 \epsilon^4}{L_0 \sigma^2 \log \left( 1 + \frac{\sigma^2}{\gamma^2} \right)} \frac{\Delta L_0 \sigma^2 \log(1 + \sigma^2/\gamma^2)}{24 \epsilon^4} = \frac{\Delta}{3},
\end{align*}
where $(i)$ uses the definition of $T_2$.

Finally, plugging back to \Eqref{eq:adagrad_norm_delta_inter_2} yields
$-f(\vx_t) \leq \frac{2\Delta}{3}$, and plugging this back into
\Eqref{eq:adagrad_norm_hatf_delta}:
\begin{align*}
    \hat{f}(\vx_0) - \inf_{\vx} \hat{f}(\vx) &\leq \frac{3 \epsilon^2}{2 L_0} + \frac{2 \Delta}{3} \\
    &\Eqmark{i}{\leq} \frac{\Delta}{3} + \frac{2\Delta}{3} = \Delta,
\end{align*}
where $(i)$ uses the condition $\epsilon \leq \frac{\sqrt{2}}{3} \sqrt{\Delta
L_0}$. Therefore $\hat{f}$ satisfies all conditions of $\gF_{\text{as}}(\Delta,
L_0, L_1, \sigma)$.
\end{proof}

\begin{theorem} \label{thm:app_adagrad_norm}[Restatement of Theorem \ref{thm:adagrad_norm}]
Let $\Delta, L_0, L_1, \sigma > 0$, and let $\gF = \gF_{\textup{as}}(\Delta, L_0, L_1,
\sigma)$. Let algorithm $A_{\text{DAN}}$ denote Decorrelated AdaGrad-Norm with
parameters $\eta > 0$ and \[
    0 < \gamma \leq \frac{\Delta L_1}{8 \log \left( 1 + 48 \frac{\Delta L_1^2}{L_0} \right)}.
\]
Let \(
    0 < \epsilon \leq \min \left\{ \frac{\sqrt{2}}{3} \sqrt{\Delta L_0}, \frac{1}{\sqrt{3}} \sqrt{\Delta L_1 \gamma}, \Delta L_1 \right\}.
\)
If $\Delta L_1^2 \geq L_0$, then \[
    \gT(A_{\text{DAN}}, \gF, \epsilon) \geq 1 + \frac{\Delta^2 L_1^2 \sigma^2}{144 \epsilon^4} + \frac{\Delta L_0 \sigma^2 \log(1 + \sigma^2/\gamma^2)}{24 \epsilon^4} + \frac{\Delta^2 L_1^2}{144 \epsilon^2}.
\]
\end{theorem}

\begin{proof}
We only need to combine Lemmas \ref{lem:adagrad_norm_div} and \ref{lem:adagrad_norm_slow}. If $\eta \leq \frac{1}{L_1}$, then
\[
    \frac{\gamma}{\eta} \leq \frac{\Delta L_1}{8 \eta \log \left( 1 + 48 \frac{\Delta L_1^2}{L_0} \right)} \leq \frac{\Delta L_1^2}{8 \log \left( 1 + 48 \frac{\Delta L_1^2}{L_0} \right)},
\]
so the conditions of $\gamma$ and $\eta$ in Lemma \ref{lem:adagrad_norm_div} are
satisfied. Therefore, by Lemma \ref{lem:adagrad_norm_div} there exists a problem
instance $(f, g, \gD) \in \gF$ for which $\|\nabla f(\vx_t)\| \geq \Delta L_0 >
\epsilon$ for all $t \geq 0$. If $\eta \leq \frac{1}{L_1}$, then by Lemma
\ref{lem:adagrad_norm_div} there exists a problem instance $(f, g, \gD) \in \gF$
for which $\|\nabla f(\vx_t)\| \geq \epsilon$ for all $t \leq T := 1 +
\frac{\Delta^2 L_1^2 \sigma^2}{144 \epsilon^4} + \frac{\Delta L_0 \sigma^2
\log(1 + \sigma^2/\gamma^2)}{24 \epsilon^4} + \frac{\Delta^2 L_1^2}{144
\epsilon^2}$. In both cases, $A_{\text{DAN}}$ requires at least $T$ gradient
queries to find an $\epsilon$-approximate stationary point.
\end{proof}

\section{Proofs of Theorem \ref{thm:adagrad} and \ref{thm:vanilla_adagrad}} \label{app:adagrad_proof}

\begin{lemma}[Restatement of Lemma \ref{lem:adagrad_div}] \label{lem:app_adagrad_div}
Let $0 < \epsilon < \Delta L_1$. If the parameters of Decorrelated AdaGrad
satisfy $\eta \geq \frac{\sqrt{2} \gamma}{L_1 \sigma} \log \left( 1 + \frac{L_1
\epsilon}{L_0} \right)$, then for any $T \geq 1$, there exists some $f \in
\gF_{\textup{as}}(\Delta, L_0, L_1, \sigma)$ such that $\|\nabla f(\vx_t)\| \geq
\epsilon$ for all $0 \leq t \leq T-1$.

Similarly, if the parameters of AdaGrad satisfy $\eta \geq \frac{\sqrt{2}}{L_1} \log
\left( 1 + \frac{L_1 \epsilon}{L_0} \right)$ and $\gamma \leq \sigma$, then for any $T
\geq 1$ there exists some $f \in \gF_{\textup{as}}(\Delta, L_0, L_1, \sigma)$
such that $\|\nabla f(\vx_t)\| \geq \epsilon$ for all $0 \leq t \leq T-1$.
\end{lemma}

\begin{proof}
First, recall the definition of $\psi$: \[
    \tilde{\psi}(x) = \frac{L_0}{L_1^2} \left( \exp \left( L_1 |x| \right) - L_1 |x| - 1 \right).
\]
Then define \[
    f(\vx) = \sum_{i=1}^T \psi(\langle \vx, \mathbf{e}_i \rangle).
\]
To see that $f$ is $(L_0, L_1)$-smooth, let $\vx, \vy \in \mathbb{R}^d$. Denoting $\vx =
(x_1, \ldots, x_T)$ and $\vy = (y_1, \ldots, y_T)$,
\begin{align*}
    \|\nabla f(\vx) - \nabla f(\vy)\|^2 &= \sum_{i=1}^T \left( \nabla_i f(\vx) - \nabla_i f(\vx) \right)^2 \\
    &= \sum_{i=1}^T \left( \psi'(x_i) - \psi'(y_i) \right)^2 \\
    &\Eqmark{i}{\leq} \sum_{i=1}^T \left( L_0 + L_1 |\psi'(x_i)| \right)^2 (x_i - y_i)^2 \\
    &\Eqmark{ii}{\leq} \sum_{i=1}^T \left( L_0 + L_1 \|\nabla f(\vx)\| \right)^2 (x_i - y_i)^2 \\
    &= \left( L_0 + L_1 \|\nabla f(\vx)\| \right)^2 \sum_{i=1}^T (x_i - y_i)^2 \\
    &= \left( L_0 + L_1 \|\nabla f(\vx)\| \right)^2 \|\vx - \vy\|^2,
\end{align*}
where $(i)$ uses the fact that $\psi$ is $(L_0, L_1)$-smooth and $(ii)$ uses
$|\psi'(x_i)| \leq \|\nabla f(\vx)\|$. Therefore $f$ is $(L_0, L_1)$-smooth.
Also, define $m = \frac{1}{L_1} \log \left( 1 + \frac{L_1 \epsilon}{L_0}
\right)$, so that $\psi'(m) = \epsilon$. Consider the initial point $\vx_0 = m
\mathbf{e}_1$. Then
\begin{align*}
    f(\vx_0) - \inf_{\vx} f(\vx) &= \psi(m) \\
    &= \frac{L_0}{L_1^2} \left( \exp \left( L_1 m \right) - L_1 m - 1 \right) \\
    &= \frac{L_0}{L_1^2} \left( 1 + \frac{L_1 \epsilon}{L_0} - \log \left( 1 + \frac{L_1 \epsilon}{L_0} \right) - 1 \right) \\
    &= \frac{\epsilon}{L_1} - \frac{L_0}{L_1^2} \log \left( 1 + \frac{L_1 \epsilon}{L_0} \right) \\
    &= \frac{\epsilon}{L_1} \Eqmark{i}{\leq} \Delta,
\end{align*}
where $(i)$ uses the condition $\epsilon \leq \Delta L_1$.

We also define a stochastic gradient for $f$ as follows: \[
    F(\vx, \xi) = \nabla f(\vx) + (2 \xi - 1) \sigma \mathbf{e}_{j(x)},
\]
where \[
    j(x) = \begin{cases}
        0 & \langle \vx, \mathbf{e}_i \rangle \neq 0 \text{ for all } 1 \leq i \leq T \\
        \min \left\{ 1 \leq i \leq T \;|\; \langle \vx, \mathbf{e}_i \rangle = 0 \right\} & \text{otherwise}
    \end{cases},
\]
and the distribution $\gD$ of $\Xi$ is defined as $P(\xi = 0) = P(\xi = 1) = 0.5$. Then
$\mathbb{E}_{\xi}[F(\vx, \xi)] = \nabla f(\vx)$ and $\|F(\vx, \xi) - \nabla f(\vx)\|
\leq \sigma$ almost surely. Therefore, $(f, g, \gD) \in \gF_{\text{as}}(\Delta, L_0,
L_1, \sigma)$.

Now consider the trajectory of Decorrelated AdaGrad when optimizing $(f, g, \gD)$ from
the initial point $\vx_0 = m \mathbf{e}_1$. We claim that for all $0 \leq t \leq T-1$:
\begin{align}
    |\langle \vx_t, \mathbf{e}_{t+1} \rangle| &\geq m \label{eq:adagrad_inductive_1} \\
    \langle \vx_t, \mathbf{e}_j \rangle &= 0 \text{ for all } j > t+1, \label{eq:adagrad_inductive_2}
\end{align}
which we will prove by induction on $t$. The base case $t=0$ holds from the choice of
the initial point $\vx_0$. So suppose that \Eqref{eq:adagrad_inductive_1} and
\Eqref{eq:adagrad_inductive_2} hold for all $0 \leq i \leq t$ for some $0 \leq t \leq
T-2$. Then $j(\vx_t) = t+2$, so
\begin{align*}
    \langle F(\vx_t, \xi_t), \mathbf{e}_{t+2} \rangle &= \langle \nabla f(\vx_t), \mathbf{e}_{t+2} \rangle + \langle (2 \xi - 1) \sigma \mathbf{e}_{t+2}, \mathbf{e}_{t+2} \rangle \\
    &= \tilde{\psi}'(\langle \vx_t, \mathbf{e}_{t+2}) + (2 \xi - 1) \sigma \\
    &\Eqmark{i}{=} \tilde{\psi}'(0) + (2 \xi - 1) \sigma \\
    &= (2 \xi - 1) \sigma,
\end{align*}
where $(i)$ uses \Eqref{eq:adagrad_inductive_2} from the inductive hypothesis.
Therefore, for Decorrelated AdaGrad:
\begin{align*}
    \langle \vx_{t+1}, \mathbf{e}_{t+2} \rangle &= \langle \vx_t, \mathbf{e}_{t+2} \rangle - \frac{\eta}{\sqrt{\gamma^2 + \sum_{i=0}^{t-1} \left( \langle F(\vx_i, \xi_i), \mathbf{e}_{t+2} \rangle \right)^2}} (2 \xi_t - 1) \sigma \\
    &\Eqmark{i}{=} - \frac{\eta}{\sqrt{\gamma^2 + \sum_{i=0}^{t-1} \left( \langle F(\vx_i, \xi_i), \mathbf{e}_{t+2} \rangle \right)^2}} (2 \xi_t - 1) \sigma \\
    &\Eqmark{ii}{=} - \frac{\eta}{\sqrt{\gamma^2 + \sum_{i=0}^{t-1} \left( \psi'(0) \right)^2}} (2 \xi_t - 1) \sigma \\
    &= - \frac{\eta}{\gamma} (2 \xi_t - 1) \sigma,
\end{align*}
where both $(i)$ and $(ii)$ use \Eqref{eq:adagrad_inductive_2} from the inductive
hypothesis. Therefore \[
    |\langle \vx_{t+1}, \mathbf{e}_{t+2} \rangle| = \frac{\eta}{\gamma} \sigma \geq m,
\]
where the inequality uses the condition $\eta \geq \frac{\gamma m}{\sigma}$ for
Decorrelated AdaGrad. Similarly for AdaGrad:
\begin{align*}
    \langle \vx_{t+1}, \mathbf{e}_{t+2} \rangle &= \langle \vx_t, \mathbf{e}_{t+2} \rangle - \frac{\eta}{\sqrt{\gamma^2 + \sum_{i=0}^t \left( \langle F(\vx_i, \xi_i), \mathbf{e}_{t+2} \rangle \right)^2}} (2 \xi_t - 1) \sigma \\
    &\Eqmark{i}{=} - \frac{\eta}{\sqrt{\gamma^2 + \sum_{i=0}^t \left( \langle F(\vx_i, \xi_i), \mathbf{e}_{t+2} \rangle \right)^2}} (2 \xi_t - 1) \sigma \\
    &\Eqmark{ii}{=} - \frac{\eta}{\sqrt{\gamma^2 + \sum_{i=0}^{t-1} \left( \psi'(0) \right)^2 + (\psi'(0) + \sigma)^2}} (2 \xi_t - 1) \sigma \\
    &= - \frac{\eta}{\sqrt{\gamma^2 + \sigma^2}} (2 \xi_t - 1) \sigma.
\end{align*}
where both $(i)$ and $(ii)$ use \Eqref{eq:adagrad_inductive_2} from the inductive
hypothesis. Therefore
\begin{align*}
    |\langle \vx_{t+1}, \mathbf{e}_{t+2} \rangle| &= \frac{\eta}{\sqrt{\gamma^2 + \sigma^2}} \sigma \Eqmark{i}{\geq} \frac{\eta}{\sqrt{2}} \Eqmark{ii}{\geq} \frac{\eta}{\sqrt{2}} \geq m,
\end{align*}
where $(i)$ uses the assumed condition $\gamma \leq \sigma$ and $(ii)$ uses the
condition $\eta \geq \sqrt{2} m$. This proves the inductive step for
\Eqref{eq:adagrad_inductive_1}, for both Decorrelated AdaGrad and AdaGrad. The inductive
step for \Eqref{eq:adagrad_inductive_2} follows immediately from the inductive
hypothesis (\Eqref{eq:adagrad_inductive_2}) together with the stochastic gradient
definition and $j(\vx_t) = t+2$. This completes the induction.

For all $0 \leq t \leq T-1$, the conclusion of the lemma follows from
\Eqref{eq:adagrad_inductive_1} by: \[
    \|\nabla f(\vx_t)\| \geq \langle \nabla f(\vx_t), \mathbf{e}_{t+1} \rangle = \psi'(\langle \vx_t, \mathbf{e}_{t+1} \rangle) \Eqmark{i}{\geq} \psi'(m) = \epsilon,
\]
where $(i)$ uses \Eqref{eq:adagrad_inductive_1} together with the fact that
$\psi'(x)$ increases with $|x|$.
\end{proof}

\begin{lemma}[Restatement of Lemma \ref{lem:adagrad_slow}] \label{lem:app_adagrad_slow}
Let \[
    0 < \epsilon < \min \left\{ \frac{\Delta L_1}{2}, \sqrt{\frac{\Delta L_1 \sigma}{4 \sqrt{2} \log \left( 1 + \frac{\Delta L_1^2}{L_0} \right)}} \right\}.
\]
If the parameters of Decorrelated AdaGrad satisfy $\eta \leq \frac{\sqrt{2} \gamma}{L_1
\sigma} \log \left( 1 + \frac{L_1 \epsilon}{L_0} \right)$, then there exists some $(f,
g, \gD) \in \gF_{\text{as}}(\Delta, L_0, L_1, \sigma)$ such that $\|\nabla f(\vx_t)\|
\geq \epsilon$ for all \[
    t \leq \frac{\Delta^2 L_0^2 \sigma^2}{256 \gamma^2 \epsilon^4} + \frac{\Delta^2 L_1^2 \sigma^2}{256 \gamma^2 \epsilon^2 \log^2 \left( 1 + \frac{\Delta L_1^2}{L_0} \right)}.
\]

Similarly, if the parameters of AdaGrad satisfy $\eta \leq \frac{\sqrt{2}}{L_1} \log
\left( 1 + \frac{L_1 \epsilon}{L_0} \right)$, then there exists some $(f, g, \gD) \in
\gF_{\text{as}}(\Delta, L_0, L_1, \sigma)$ such that $\|\nabla f(\vx_t)\| \geq \epsilon$
for all \[
    t \leq \frac{\Delta^2 L_0^2}{128 \epsilon^4} + \frac{\Delta^2 L_1^2}{128 \epsilon^2 \log^2 \left( 1 + \frac{\Delta L_1^2}{L_0} \right)}.
\]
\end{lemma}

\begin{proof}
Define $m = \frac{1}{L_1} \log \left( 1 + \frac{L_1 \epsilon}{L_0} \right)$, and
consider the objective \[
    f(x) = \begin{cases}
        -\epsilon (x+m) + \psi(m) & x < -m \\
        \psi(x) & x \in [-m, m] \\
        \epsilon (x-m) + \psi(m) & x > m
    \end{cases}.
\]
This function is differentiable everywhere since $\psi'(m) = \epsilon$ and $\psi'(-m) =
-\epsilon$. Since $\psi$ is $(L_0, L_1)$-smooth, so is $f$. Also, $m$ satisfies
\begin{align*}
    \psi(m) &= \frac{L_0}{L_1^2} \left( \exp(L_1 m) - L_1 m - 1 \right) \\
    &= \frac{L_0}{L_1^2} \left( 1 + \frac{L_1 \epsilon}{L_0} - \log \left( 1 + \frac{L_1 \epsilon}{L_0} \right) - 1 \right) \\
    &= \frac{\epsilon}{L_1} - \frac{L_0}{L_1^2} \log \left( 1 + \frac{L_1 \epsilon}{L_0} \right) \\
    &\Eqmark{i}{\leq} \frac{\epsilon}{L_1} \leq \frac{\Delta}{2},
\end{align*}
where $(i)$ uses the condition $\epsilon \leq \frac{1}{2} \Delta L_1$. Therefore, with the initial point $x_0 = m + \frac{\Delta}{2 \epsilon}$, the objective satisfies
\begin{align*}
    f(x_0) - \inf_x f(x) &= \epsilon (x_0 - m) + \psi(m) \\
    &= \epsilon \frac{\Delta}{2 \epsilon} + \frac{\Delta}{2} \\
    &= \Delta.
\end{align*}

We will define the stochastic gradient $g$ with noise distribution $\gD$ as equal to the
true gradient, i.e. $g(x, \xi) = f'(x)$ for every $x, \xi$. Therefore $(f, g, \gD) \in
\gF_{\text{as}}(\Delta, L_0, L_1, \sigma)$.

Now consider the trajectory of Decorrelated AdaGrad when optimizing $(f, g, \gD)$ from
the initial point $x_0 = m + \frac{\Delta}{2 \epsilon}$. Let $t_0 = \max \left\{ t \geq
0 \;|\; x_t \geq m \right\}$. Then $f'(x_t) = \epsilon$ for all $t \leq t_0$, so that \[
    x_{t+1} = x_t - \frac{\eta \epsilon}{\sqrt{\gamma^2 + \sum_{i=0}^{t-1} \epsilon^2}} = x_t - \frac{\eta \epsilon}{\sqrt{\gamma^2 + t \epsilon^2}},
\]
and unrolling yields
\begin{align*}
    x_{t+1} &= x_0 - \eta \epsilon \sum_{i=0}^t \frac{1}{\sqrt{\gamma^2 + i \epsilon^2}} \\
    &= x_0 - \frac{\eta \epsilon}{\gamma} - \eta \epsilon \sum_{i=1}^t \frac{1}{\sqrt{\gamma^2 + i \epsilon^2}} \\
    &\geq x_0 - \frac{\eta \epsilon}{\gamma} - \eta \epsilon \int_0^t \frac{1}{\sqrt{\gamma^2 + x \epsilon^2}} dx \\
    &= x_0 - \frac{\eta \epsilon}{\gamma} - \frac{2 \eta}{\epsilon} \left[ \sqrt{\gamma^2 + x \epsilon^2} \right]_0^t \\
    &= x_0 - \frac{\eta \epsilon}{\gamma} - \frac{2 \eta}{\epsilon} \left( \sqrt{\gamma^2 + t \epsilon^2} - \gamma \right) \\
    &= x_0 - \frac{\eta \epsilon}{\gamma} - \frac{2 \eta}{\epsilon} \frac{t \epsilon^2}{ \sqrt{\gamma^2 + t \epsilon^2} + \gamma } \\
    &\geq x_0 - \frac{\eta \epsilon}{\gamma} - \frac{2 \eta t \epsilon}{\sqrt{\gamma^2 + t \epsilon^2}}.
\end{align*}
Plugging $t = t_0$ then yields
\begin{equation} \label{eq:adagrad_slow_inter_1}
    x_{t_0+1} \geq x_0 - \frac{\eta \epsilon}{\gamma} - \frac{2 \eta t_0 \epsilon}{\sqrt{\gamma^2 + t_0 \epsilon^2}}.
\end{equation}
On the other hand,
\begin{equation} \label{eq:adagrad_slow_inter_2}
    x_{t_0+1} < m.
\end{equation}
Combining \Eqref{eq:adagrad_slow_inter_1} and \Eqref{eq:adagrad_slow_inter_2}:
\begin{align}
    m &\geq x_0 - \frac{\eta \epsilon}{\gamma} - \frac{2 \eta t_0 \epsilon}{\sqrt{\gamma^2 + t_0 \epsilon^2}} \nonumber \\
    \frac{2 \eta t_0 \epsilon}{\sqrt{\gamma^2 + t_0 \epsilon^2}} &\geq x_0 - m - \frac{\eta \epsilon}{\gamma} \nonumber \\
    \frac{2 \eta t_0 \epsilon}{\sqrt{\gamma^2 + t_0 \epsilon^2}} &\Eqmark{i}{\geq} \frac{\Delta}{2 \epsilon} - \frac{\eta \epsilon}{\gamma} \nonumber \\
    \frac{t_0}{\sqrt{\gamma^2 + t_0 \epsilon^2}} &\geq \frac{1}{\epsilon} \left( \frac{\Delta}{4 \eta \epsilon} - \frac{\epsilon}{2 \gamma} \right), \label{eq:adagrad_slow_inter_3}
\end{align}
where $(i)$ uses the definition of $x_0$. The last term in
\Eqref{eq:adagrad_slow_inter_3} can be bounded as
\begin{align*}
    \frac{\epsilon}{2 \gamma} &\leq \frac{\Delta}{8 \eta \epsilon} \frac{4 \eta \epsilon^2}{\Delta \gamma} \\
    &\Eqmark{i}{\leq} \frac{\Delta}{8 \eta \epsilon} \frac{4 \epsilon^2}{\Delta \gamma} \frac{\sqrt{2} \gamma}{L_1 \sigma} \log \left( 1 + \frac{L_1 \epsilon}{L_0} \right) \\
    &= \frac{\Delta}{8 \eta \epsilon} \frac{4 \sqrt{2} \epsilon^2}{\Delta L_1 \sigma} \log \left( 1 + \frac{L_1 \epsilon}{L_0} \right) \\
    &\Eqmark{ii}{\leq} \frac{\Delta}{8 \eta \epsilon} \frac{4 \sqrt{2} \epsilon^2}{\Delta L_1 \sigma} \log \left( 1 + \frac{\Delta L_1^2}{L_0} \right) \\
    &\Eqmark{iii}{\leq} \frac{\Delta}{8 \eta \epsilon},
\end{align*}
where $(i)$ uses the condition $\eta \leq \frac{\sqrt{2} \gamma}{L_1 \sigma} \log \left(
1 + \frac{L_1 \epsilon}{L_0} \right)$, $(ii)$ uses the condition $\epsilon \leq \Delta
L_1$, and $(iii)$ uses the condition $\epsilon \leq \sqrt{\frac{\Delta L_1 \sigma}{4
\sqrt{2} \log \left( 1 + \frac{\Delta L_1^2}{L_0} \right)}}$. Plugging back to
\Eqref{eq:adagrad_slow_inter_3} yields
\begin{align}
    \frac{t_0}{\sqrt{\gamma^2 + t_0 \epsilon^2}} &\geq \frac{\Delta}{8 \eta \epsilon^2} \nonumber \\
    \frac{t_0}{\sqrt{t_0 \epsilon^2}} &\geq \frac{\Delta}{8 \eta \epsilon^2} \nonumber \\
    \sqrt{t_0} &\geq \frac{\Delta}{8 \eta \epsilon} \nonumber \\
    t_0 &\geq \frac{\Delta^2 }{64 \eta^2 \epsilon^2} \label{eq:adagrad_slow_inter_4}
\end{align}
From the assumed upper bound on $\eta$, \[
    \eta \leq \frac{\sqrt{2} \gamma}{L_1 \sigma} \log \left( 1 + \frac{L_1 \epsilon}{L_0} \right) \Eqmark{i}{\leq} \frac{\sqrt{2} \gamma}{L_1 \sigma} \log \left( 1 + \frac{\Delta L_1^2}{L_0} \right),
\]
and \[
    \eta \leq \frac{\sqrt{2} \gamma}{L_1 \sigma} \log \left( 1 + \frac{L_1 \epsilon}{L_0} \right) \Eqmark{ii}{\leq} \frac{\sqrt{2} \gamma \epsilon}{L_0 \sigma},
\]
where $(i)$ uses the condition $\epsilon \leq \Delta L_1$ and $(ii)$ uses $\log(1+x)
\leq x$. Therefore
\begin{align*}
    \frac{1}{\eta} &\geq \max \left\{ \frac{L_1 \sigma}{\sqrt{2} \gamma \log \left( 1 + \frac{\Delta L_1^2}{L_0} \right)}, \frac{L_0 \sigma}{\sqrt{2} \gamma \epsilon} \right\} \\
    \frac{1}{\eta^2} &\geq \max \left\{ \frac{L_1^2 \sigma^2}{2 \gamma^2 \log^2 \left( 1 + \frac{\Delta L_1^2}{L_0} \right)}, \frac{L_0^2 \sigma^2}{2 \gamma^2 \epsilon^2} \right\}
    \geq \frac{L_1^2 \sigma^2}{4 \gamma^2 \log^2 \left( 1 + \frac{\Delta L_1^2}{L_0} \right)} + \frac{L_0^2 \sigma^2}{4 \gamma^2 \epsilon^2}.
\end{align*}
Plugging back to \Eqref{eq:adagrad_slow_inter_4} yields
\begin{align*}
    t_0 \geq \frac{\Delta^2 L_0^2 \sigma^2}{256 \gamma^2 \epsilon^4} + \frac{\Delta^2 L_1^2 \sigma^2}{256 \gamma^2 \epsilon^2 \log^2 \left( 1 + \frac{\Delta L_1^2}{L_0} \right)}.
\end{align*}
Since $|f'(x_t)| = \epsilon$ for all $t \leq t_0$, this completes the proof for
Decorrelated AdaGrad.

The corresponding proof for AdaGrad is nearly identical, so we list only the key steps
here. For all $t \leq t_0$, \[
    x_{t+1} = x_t - \frac{\eta \epsilon}{\sqrt{\gamma^2 + (t+1) \epsilon^2}} \geq x_t - \frac{\eta}{\sqrt{t+1}}.
\]
After unrolling and applying the same bound for $\sum_i \frac{1}{\sqrt{i}}$ as in the
decorrelated case, then choosing $t = t_0$, we have \[
    x_{t_0+1} \geq x_0 - 2 \eta \sqrt{t_0+1}.
\]
From $x_{t_0+1} \leq m$, we have
\begin{align*}
    m &\geq x_0 - 2 \eta \sqrt{t_0+1} \\
    t_0 + 1 &\geq \frac{\Delta^2}{16 \eta^2 \epsilon^2}.
\end{align*}
The assumed upper bound on $\eta$ yields \[
    \frac{1}{\eta^2} \geq \frac{L_1^2}{4 \log^2 \left( 1 + \frac{\Delta L_1^2}{L_0} \right)} + \frac{L_0^2}{4 \epsilon^2},
\]
so that \[
    t_0 + 1 \geq \frac{\Delta^2 L_0^2}{64 \epsilon^4} + \frac{\Delta^2 L_1^2}{64 \epsilon^2 \log^2 \left( 1 + \frac{\Delta L_1^2}{L_0} \right)}.
\]
Since $2 t_0 \geq t_0 + 1$ for all $t_0 \geq 1$, this means \[
    t_0 \geq \frac{\Delta^2 L_0^2}{128 \epsilon^4} + \frac{\Delta^2 L_1^2}{128 \epsilon^2 \log^2 \left( 1 + \frac{\Delta L_1^2}{L_0} \right)}.
\]
\end{proof}

\begin{theorem} \label{thm:app_adagrad}[Restatement of Theorem \ref{thm:adagrad}]
Let $\Delta, L_0, L_1, \sigma > 0$ and let $\gF = \gF_{\text{as}}(\Delta, L_0, L_1,
\sigma)$. Let $A_{\text{DA}}$ and $A_{\text{ada}}$ denote Decorrelated AdaGrad and
AdaGrad (respectively) with parameters $\eta, \gamma > 0$. Suppose \[
    0 < \epsilon < \min \left\{ \frac{\Delta L_1}{2}, \sqrt{\frac{\Delta L_1 \sigma}{4 \sqrt{2} \log \left( 1 + \frac{\Delta L_1^2}{L_0} \right)}} \right\}.
\]
Then \[
    \gT(A_{\text{DA}}, \gF, \epsilon, \delta) \geq \frac{\Delta^2 L_0^2 \sigma^2}{256 \gamma^2 \epsilon^4} + \frac{\Delta^2 L_1^2 \sigma^2}{256 \gamma^2 \epsilon^2 \log^2 \left( 1 + \frac{\Delta L_1^2}{L_0} \right)}.
\]
Also, if $\gamma \leq \sigma$, then \[
    \gT(A_{\text{ada}}, \gF, \epsilon, \delta) \geq \frac{\Delta^2 L_0^2}{128 \epsilon^4} + \frac{\Delta^2 L_1^2}{128 \epsilon^2 \log^2 \left( 1 + \frac{\Delta L_1^2}{L_0} \right)}.
\]
\end{theorem}

\begin{proof}
We only have to combine Lemmas \ref{lem:adagrad_div} and \ref{lem:adagrad_slow}. We
first consider Decorrelated AdaGrad. If $\eta \geq \frac{\sqrt{2} \gamma}{L_1 \sigma}
\log \left( 1 + \frac{L_1 \epsilon}{L_0} \right)$, then by Lemma \ref{lem:adagrad_div}
there exists a problem instance for which Decorrelated AdaGrad will never find an
$\epsilon$-approximate stationary point. Otherwise, by Lemma \ref{lem:adagrad_slow}
there exists a problem instance for which Decorrelated AdaGrad requires a number of
steps at least as large as \[
    \frac{\Delta^2 L_0^2 \sigma^2}{256 \gamma^2 \epsilon^4} + \frac{\Delta^2 L_1^2 \sigma^2}{256 \gamma^2 \epsilon^2 \log^2 \left( 1 + \frac{\Delta L_1^2}{L_0} \right)}.
\]
Therefore in either case, we have \[
    \gT(A_{\text{DA}}, \gF, \epsilon, \delta) \geq \frac{\Delta^2 L_0^2 \sigma^2}{256 \gamma^2 \epsilon^4} + \frac{\Delta^2 L_1^2 \sigma^2}{256 \gamma^2 \epsilon^2 \log^2 \left( 1 + \frac{\Delta L_1^2}{L_0} \right)}.
\]
\end{proof}

The corresponding proof for AdaGrad (Theorem \ref{thm:vanilla_adagrad}) is
nearly identical, so we omit it.

\section{Proof of Theorem \ref{thm:singlestep}} \label{app:singlestep_proof}

\subsection{Preliminary Definitions} \label{app:singlestep_defs}
We first provide definitions of constants and objects that will be used
throughout the proof.

For $p \in (0, 1)$ and $\lambda > 0$, consider the random walk parameterized by
$(p, \lambda)$:
\begin{align}
    X_0 &= 1 \nonumber \\
    P(X_{t+1} = X_t + \lambda) &= p \label{eq:random_walk} \\
    P(X_{t+1} = X_t - 1) &= 1-p. \nonumber
\end{align}
Then we can define
\begin{align}
    z_{p,\lambda} &= P(\exists t > 0: X_t \leq 0) \nonumber \\
    \lambda_0(p, \delta) &= \inf \left\{ \lambda \geq 0 : z_{p,\lambda} \leq 1-\delta \right\} \nonumber \\
    \zeta(p, \delta) &= \lambda_0(p, \delta) - \lambda_0(p, 0). \label{eq:zeta_def}
\end{align}
Informally, $z_{p,\lambda}$ is the probability that the random walk reaches a
non-positive value, and $\lambda_0(p, \delta)$ is the smallest $\lambda$
required to ensure that the chance of never reaching a non-positive value is at
least $\delta$.

For $\sigma_2 \geq 3$, define the following constants:
\begin{align} \label{eq:gammas_123}
    \gamma_1 &= \frac{ \log \left( 1 + 2 \zeta(2/3, \delta) \right) }{\log 2}, \quad \gamma_2 = 1 - \frac{ \log 2 }{\log \left( 2 + \frac{6}{\sigma_2-2} \right)}, \quad \gamma_3 = \frac{ \log \left( 1 + 2 \zeta(2/3, \delta) \right)}{\log \left( 2 + \frac{6}{\sigma_2-2} \right)}
\end{align}
For $\sigma_2 \in (1, 3)$, define the following constants:
\begin{align*}
    \gamma_4 &= \frac{ \log \left( 1 + 2 \zeta(\frac{1}{12} (\sigma_2+5), \delta) \right) }{\log \left( \frac{12}{-\sigma_2+7} - 1 \right)}, \quad \gamma_5 = 1 - \frac{ \log \left( \frac{12}{-\sigma_2+7} - 1 \right) }{\log \left( \frac{18}{\sigma_2-1} - 1 \right)} \\
    \gamma_6 &= \frac{ \log \left( 1 + 2 \zeta(\frac{1}{12}(\sigma_2+5), \delta) \right) }{\log \left( \frac{18}{\sigma_2-1} - 1 \right)}.
\end{align*}

Also, we will denote:
\begin{equation} \label{eq:G_def}
    G = \frac{\Delta L_1}{1 + 4 \log \left( 1 + \frac{\Delta L_1^2}{L_0} \right)},
\end{equation}
which will be used in the following definition.

\begin{definition} \label{def:tricky_pair}
For $p \in \left( \frac{1}{2}, \frac{\sigma_2}{\sigma_2+1} \right)$ and $\delta \in (0,
1)$, we say that $\vg_1, \vg_2 \in \mathbb{R}^d$ forms a $(p, \delta)$-tricky pair with
respect to the stepsize function $\alpha$ if all of the following conditions hold:
\begin{align}
    \vg_1 &= c_1 \vg, \quad \textup{ and } \quad \vg_2 = c_2 \vg \textup{ for some } \vg \in \mathbb{R}^d \textup { with } \|\vg\| = 1 \\
    \text{sign}(c_1) &\neq \text{sign}(c_2) \label{eq:tricky_opposite} \\
    |c_1| &\geq \epsilon \quad \text{and} \quad |c_2| \geq \epsilon \label{eq:tricky_g_min} \\
    |c_1| &\leq \frac{1-p}{p} \sigma_1 + \left( \frac{1-p}{p} \sigma_2 - 1 \right) G \label{eq:tricky_g1_max} \\
    |c_2| &\geq \begin{cases}
        \frac{p |c_1| + \epsilon}{1-p} & |c_1| \leq \frac{1-p}{p} \sigma_1 + \left( \frac{1-p}{p} \sigma_2 - 1 \right) \epsilon \\
        \frac{(\sigma_2+1) p|c_1| - \sigma_1}{(\sigma_2+1)(1-p)-1} & |c_1| > \frac{1-p}{p} \sigma_1 + \left( \frac{1-p}{p} \sigma_2 - 1 \right) \epsilon
    \end{cases} \label{eq:tricky_g2_min} \\
    |c_2| &\leq \frac{p |c_1| + G}{1-p} \label{eq:tricky_g2_max} \\
    \frac{\alpha(\vg_1) \|\vg_1\|}{\alpha(\vg_2) \|\vg_2\|} &\geq \lambda_0(p, \delta). \label{eq:tricky_lr}
\end{align}
\end{definition}

Notice that the lower bound of $|c_2|$ for the second case of \Eqref{eq:tricky_g2_min}
is positive, since $p < \frac{\sigma_2}{\sigma_2+1}$. The significance of a tricky pair,
as shown in Lemma \ref{lem:app_singlestep_tricky_diverge}, is that it can be used to
construct an instance $(f, g, \gD) \in \gF_{\text{aff}}(\Delta, L_0, L_1, \sigma_1,
\sigma_2)$ for which $A$ diverges with probability at least $\delta$.

Finally, let $\hat{P}_{\vy}(\vx) = \frac{\langle \vx, \vy \rangle}{\|\vy\|}$, so that
$\hat{P}_{\vy}(\vx)$ denotes the component of $\vx$ in the direction of
$\frac{\vy}{\|\vy\|}$. Note the difference from the common notation $P_{\vy}(\vx) =
\frac{\langle \vx, \vy \rangle}{\|\vy\|^2} \vx$.

\subsection{Proofs}
We now provide proofs of the lemmas mentioned in Section \ref{sec:singlestep_sketch}.

\begin{lemma} \label{lem:app_singlestep_large_lr}
Suppose that there exists some $\vg \in \mathbb{R}^d$ with $\|\vg\| \in [\epsilon,
\sigma_1 + (\sigma_2 + 1) \Delta L_1]$ and \[
    \alpha(\vg) \leq 0, \quad \text{or} \quad \alpha(\vg) \geq \frac{4}{L_1 \|\vg\|} \log \left( 1 + \frac{L_1 \min(\|\vg\|, \Delta L_1)}{L_0} \right).
\]
Then there exists $(f_{\textup{exp}}, g_{\textup{exp}}, \gD_{\textup{exp}}) \in
\gF_{\textup{aff}}(\Delta, L_0, L_1, \sigma_1, \sigma_2)$ such that $\|\nabla
f_{\textup{exp}}(\vx_t)\| \geq \epsilon$ for all $t \geq 0$.
\end{lemma}

\begin{proof}
We will construct $f: \mathbb{R}^d \rightarrow \mathbb{R}$ piecewise with linear and
exponential pieces so that $\|\nabla f(x_t)\| = \min(\|\vg\|, \Delta L_1) \geq \epsilon$
for all $t \geq 0$.

First, define $\tilde{\vg} := \min(\|\vg\|, \Delta L_1) \frac{\vg}{\|\vg\|}$, $m :=
\frac{1}{L_1} \log \left( 1 + \frac{L_1 \|\tilde{\vg}\|}{L_0} \right)$. Recall the
function $\psi: \mathbb{R} \rightarrow \mathbb{R}$ defined as \[
    \psi(x) = \frac{L_0}{L_1^2} \left( \exp(L_1 |x|) - L_1 |x| - 1 \right).
\]
It is straightforward to verify that $\psi$ bounded from below by $0$,
continuously differentiable, $(L_0, L_1)$-smooth, and satisfies
\begin{align*}
    \psi(-m) &= \psi(m) = \frac{\|\tilde{\vg}\|}{L_1} - \frac{L_0}{L_1^2} \log \left( 1 + \frac{L_1 \|\tilde{\vg}\|}{L_0} \right) \leq \frac{\|\tilde{\vg}\|}{L_1} \leq \frac{\Delta L_1}{L_1} \leq \Delta \\
    |x| \leq m &\implies |\psi'(x)| \leq \|\tilde{\vg}\| \leq \Delta L_1 \\
    \psi'(-m) &= -\|\tilde{\vg}\| \\
    \psi'(m) &= \|\tilde{\vg}\|.
\end{align*}

The condition in the lemma statement gives two cases: $\alpha(\vg) \leq 0$ or
$\alpha(\vg) \geq \frac{4 m}{\|\vg\|}$. We handle the two cases separately below.

\paragraph{Case 1:} $\alpha(\vg) \leq 0$. This case is easy: the algorithm $A$ is
essentially employing a negative learning rate! For a piecewise linear function that has
a piece with gradient equal to $\vg$, the trajectory $\{x_t\}$ moves away from the
minimum indefinitely. To handle the case that $\|\vg\| > \Delta L_1$, we instead use a
gradient of $\tilde{\vg} = \min(\|\vg\|, \Delta L_1) \frac{\vg}{\|\vg\|} = \Delta L_1
\frac{\vg}{\|\vg\|}$, and construct a stochastic gradient that always returns either
$\vg$ or $\mathbf{0}$, so that each updated iterate $x_{t+1}$ either moves further from
the minimum than $\vx_t$, or doesn't move at all.

Define the objective \[
    f(\vx) = \begin{cases}
        -\|\tilde{\vg}\| (\hat{P}_{\vg}(\vx) + m) + \psi(m) & \hat{P}_{\vg}(\vx) < m \\
        \psi(\hat{P}_{\vg}(\vx)) & \hat{P}_{\vg}(\vx) \in [-m, m] \\
        \|\tilde{\vg}\| (\hat{P}_{\vg}(\vx) - m) + \psi(m) & \hat{P}_{\vg}(\vx) > m
    \end{cases}.
\]
Notice that $f$ is bounded from below by $\psi(0) = 0$. Since $\psi'(m) = \|\tilde{\vg}\|$ and
$\psi'(-m) = -\|\tilde{\vg}\|$, then $f$ is continuously differentiable. Since $\psi$ is
$(L_0, L_1)$-smooth, so is $f$. Consider the initial point $\vx_0 = m \vg$. From the
properties of $\psi$ from above, $f$ satisfies \[
    f(\vx_0) - f_* = \psi(m) \leq \Delta.
\]
and for all $\vx$ with $f(\vx) \leq f(\vx_0)$, it must be that $\hat{P}_{\vg}(\vx) \in
[-m, m]$, so
\begin{equation} \label{eq:negative_lr_small_grad}
    \|\nabla f(\vx)| = |\psi'(\hat{P}_{\vg}(\vx))| \leq \Delta L_1.
\end{equation}
Define a stochastic gradient $F$ for $f$ as follows: \[
    F(x, \xi) = \begin{cases}
        \left( \|\vg\|/\|\tilde{\vg}\| \right) \nabla f(\vx) & \xi = 0 \\
        0 & \xi = 1,
    \end{cases}
\]
where $\xi \in \{0, 1\}$ has distribution $\gD$, defined as $P(\xi = 0) =
\|\tilde{\vg}\|/\|\vg\|$. Then $\mathbb{E}_{\xi}[F(\vx, \xi)] = \nabla f(\vx)$. To see
that this stochastic gradient satisfies the noise condition: If $\|\vg\| \leq \Delta
L_1$, then $\tilde{\vg} = \vg$ and $P(\xi = 0) = 1$, so $F(\vx; \xi) = \nabla f(\vx)$
almost surely. Otherwise,
\begin{align*}
    \|F(\vx; 0) - \nabla f(\vx)| &= \left( \frac{\|\vg\|}{\|\tilde{\vg}\|} - 1 \right) \|\nabla f(\vx)\| \\
    &= \left( \frac{\|\vg\|}{\Delta L_1} - 1 \right) \|\nabla f(\vx)\| \\
    &\Eqmark{i}{\leq} \left( \frac{\sigma_1 + (\sigma_2 + 1) \Delta L_1}{\Delta L_1} - 1 \right) \|\nabla f(\vx)\| \\
    &= \frac{\sigma_1 + \sigma_2 \Delta L_1}{\Delta L_1} \|\nabla f(\vx)\| \\
    &= \frac{\|\nabla f(\vx)\|}{\Delta L_1} \sigma_1 + \sigma_2 \|\nabla f(\vx)\| \\
    &\Eqmark{ii}{\leq} \sigma_1 + \sigma_2 \|\nabla f(\vx)\|,
\end{align*}
where $(i)$ uses the assumption $\|\vg\| \leq \sigma_1 + (\sigma_2 + 1) \Delta L_1$, and
$(ii)$ uses \Eqref{eq:negative_lr_small_grad}. Also \[
    \|F(\vx; 1) - \nabla f(\vx)\| = \|\nabla f(\vx)\| \leq \sigma_1 + \sigma_2 \|f(\vx)\|,
\]
which uses the assumption $\sigma_2 > 1$. So the noise condition is satisfied, and
therefore $(f, F, \gD) \in \mathcal{F}_{\text{aff}}(\Delta, L_0, L_1, \sigma_1,
\sigma_2)$.

Now consider the trajectory of $A$ from the initial point $\vx_0 = m \vg$. We claim that
$\vx_t = c_t \vg$ for some $c_t \geq m$ for all $t \geq 0$. Clearly this holds for
$t=0$. Suppose it holds for some $t \geq 0$. Then $\|\nabla f(\vx_t)\| = \tilde{\vg}$.
The stochastic gradient has two cases: if $\xi = 0$, then \[
    F(\vx_t, \xi) = (\|\vg\|/\|\tilde{\vg}\|) \nabla f(\vx) = (\|\vg\|/\|\tilde{\vg}\|) \tilde{\vg} = \vg,
\]
so \[
    \vx_{t+1} = \vx_t - \alpha(\vg) \vg = (c_t - \alpha(\vg)) \vg = c_{t+1} \vg,
\]
and $c_{t+1} \geq c_t \geq m$ since $\alpha(\vg) \leq 0$. If $\xi = 1$, then $F(\vx_t,
\xi) = 0$, so $\vx_{t+1} = \vx_t = c_t \vg$. Either way, $\vx_{t+1} = c_{t+1} \vg$ holds
for some $c_{t+1} \geq m$, which completes the induction. Therefore $\|\nabla f(\vx_t)\|
= \|\tilde{\vg}\| \geq \epsilon$ for all $t \geq 0$.

\paragraph{Case 2:} $\alpha(\vg) \geq \frac{4 m}{\|\vg\|}$. In this case, the learning
rate $\alpha(\vg)$ is large enough to ensure that $f(\vx_{t+1}) \geq f(\vx_t)$ for an
exponentially increasing $f$. By creating $f$ that only depends on $\langle \vx, \vg
\rangle$ and which is piecewise linear and exponential in $\langle \vx, \vg \rangle$,
this increase of the objective function continues indefinitely.

Define $m' = \alpha(\vg) \|\vg\|$ and $\phi: [0, m']$ as \[
    \phi(x) = \begin{cases}
        \psi(x - m) & x \in [0, 2 m) \\
        \|\tilde{\vg}\| (x - 2 m) + \psi(m) & x \in (2 m, m' - 2 m) \\
        -\psi(x - (m' - m)) + \|\tilde{\vg}\| (m' - 4 m) + 2 \psi(m) & x \in (m' - 2 m, m']
    \end{cases}
\]
Note that the above definition makes sense since we assumed that $m' = \alpha(\vg)
\|\vg\| \geq 4 m$, so $m' - 2m \geq 2m$. Again, $\phi$ is continuously differentiable,
bounded from below, $(L_0, L_1)$-smooth, and satisfies
\begin{align*}
    |\phi'(x)| &\leq \|\tilde{\vg}\| \leq \Delta L_1 \text{ for all } x \in [0, m'] \\
    \phi(x) &\geq 0 \\
    \phi'(0) &= -\|\tilde{\vg}\| \\
    \phi'(m') &= -\|\tilde{\vg}\|.
\end{align*}
Now, we can define the objective $f$ as follows:
\begin{equation*}
    f(\vx) = \begin{cases}
        -\|\tilde{\vg}\| \hat{P}_{\vg}(\vx) + \phi(0) & \hat{P}_{\vg}(\vx) \leq 0 \\
        \phi(\hat{P}_{\vg}(\vx) - m' \floor{\hat{P}_{\vg}(\vx)/m'}) + \tilde{g} (m' - 4 m) \floor{\hat{P}_{\vw}(\vx)/m'} & \hat{P}_{\vg}(\vx) > 0
    \end{cases}
\end{equation*}
$f$ is continuous inside each "piece" (i.e. each region with $\hat{P}_{\vg}(\vx) \in (km',
(k+1)m')$ for $k \in \mathbb{Z}_{\geq 0}$). Also, using $\phi(0) = 0$, $f$ is continuous
at the boundary of each piece.  Similarly, $f$ is continuously differentiable inside
each piece, and using the fact that $\phi'(0) = -\|\tilde{\vg}\| = \phi'(m')$, is
continuously differentiable at the boundary of each piece. Also, $f$ is bounded below by
$\min_{x \in [0, m']} \phi(x) = \min_x \psi(x) = 0$. So with the initial point $\vx_0 =
\mathbf{0}$, $f$ satisfies \[
    f(\vx_0) - f_* = \phi(0) - 0 = \psi(-m) \leq \Delta.
\]
Since $\phi$ is $(L_0, L_1)$-smooth, so is $f$. Also, $\|\nabla f(\vx)\| \leq |\psi'(m)|
= \|\tilde{\vg}\| \leq \Delta L_1$ for every $x$.

Now we can define a stochastic gradient $F$ for $f$ as follows: \[
    F(\vx; \xi) = \begin{cases}
        \left( \|\vg\| / \|\tilde{\vg}\| \right) \nabla f(\vx) & \xi = 0 \\
        0 & \xi = 1
    \end{cases},
\]
where $\xi \in \{0, 1\}$ has distribution $\gD$, defined as $P(\xi = 0) =
\|\tilde{\vg}\|/\|\vg\|$. This is the same stochastic gradient that we used in Case 1,
and an identical argument shows that the noise conditions are satisfied. Therefore $(f,
F, \gD) \in \mathcal{F}_{\text{aff}}(\Delta, L_0, L_1, \sigma_1, \sigma_2)$.

Consider the execution of $A$ on $(f, F, \gD)$ from the initial point $x_0 =
\mathbf{0}$. We claim that $\vx_t$ is an integer multiple of $m' \frac{\vg}{\|\vg\|}$
for all $t \geq 0$, which we will show by induction. The base case $t=0$ holds by
construction. If $\vx_t = -k m' \frac{\vg}{\|\vg\|}$ for some $t \geq 0$, then there are
two outcomes of the stochastic gradient. If $\xi_t = 1$, then $\vx_{t+1} = \vx_t = -k m'
\frac{\vg}{\|\vg\|}$. Otherwise $\xi_t = 0$, so
\begin{align*}
    \vx_{t+1} &= \vx_t - \alpha(F(\vx_t, 0)) F(\vx_t, 0) \\
    &= -k m' \frac{\vg}{\|\vg\|} - \alpha \left( \frac{\|\vg\|}{\|\tilde{\vg}\|} \nabla f(\vx_t) \right) \frac{\|\vg\|}{\|\tilde{\vg}\|} \nabla f(\vx_t) \\
    &\Eqmark{i}{=} -k m' \frac{\vg}{\|\vg\|} - \alpha(\vg) \vg \\
    &= -k m' \frac{\vg}{\|\vg\|} - \alpha(\vg) \|\vg\| \frac{\vg}{\|\vg\|} \\
    &= -k m' \frac{\vg}{\|\vg\|} - m' \frac{\vg}{\|\vg\|} \\
    &= -(k+1) m' \frac{\vg}{\|\vg\|},
\end{align*}
where $(i)$ uses the fact that $\vx_t = -k m' \frac{\vg}{\|\vg\|} \implies \nabla f(\vx_t) =
\|\tilde{\vg}\| \frac{\vg}{\|\vg\|}$. This completes the induction. Therefore $\|\nabla
f(\vx_t)\| = \left\| \nabla f \left( -k m' \frac{\vg}{\|\vg\|} \right) \right\| =
\|\tilde{\vg}\| \geq \epsilon$ for all $t$.
\end{proof}

\begin{lemma} \label{lem:app_singlestep_tricky_diverge}
Suppose that
\begin{equation} \label{eq:singlestep_tricky_diverge_cond}
    0 < \alpha(\vg) < \frac{4}{L_1 \vg} \log \left( 1 + \frac{L_1 \min(\|\vg\|, \Delta L_1)}{L_0} \right),
\end{equation}
for all $\vg \in \mathbb{R}^d$ with $\|\vg\| \in [\epsilon, \sigma_1 + (\sigma_2 + 1)
\Delta L_1]$, and suppose that there exist $\vg_1, \vg_2 \in \mathbb{R}$ which is a $(p,
\delta)$-tricky pair with respect to $\alpha$. Then there exists $(f, g, \mathcal{D})
\in \gF_{\textup{aff}}(\Delta, L_0, L_1, \sigma_1, \sigma_2)$ such that $\|\nabla
f(\vx_t)\| \geq \epsilon$ for all $t \geq 0$ with probability at least $\delta$.
\end{lemma}

\begin{proof}
We will construct $(f, g, \gD)$ such that $f$ is a piecewise linear function, where one
piece has stochastic gradient equal to $\vg_1$ with probability $p$ and $\vg_2$ with
probability $1-p$. Using the properties of a $(p, \delta)$-tricky pair, this instance is
a member of $\gF_{\text{aff}}(\Delta, L_0, L_1, \sigma_1, \sigma_2)$, and $A$ will
diverge with probability at least $\delta$ when optimizing this instance.

From the tricky pair definition, $\vg_1 = c_1 \vg$ and $\vg_2 = c_2 \vg$ for a unit
vector $\vg$. Without loss of generality, assume that $c_1 < 0$ and $c_2 > 0$. The
following argument applies in the excluded case $c_1 > 0, c_2 < 0$ by replacing the
objective $f(x)$ below with $f(-x)$. Denote $\ell = p c_1 + (1-p) c_2$ and $a =
\frac{1}{L_1} \log \left( 1 + \frac{L_1 \ell}{L_0} \right)$, and define $f: \mathbb{R}^d
\rightarrow \mathbb{R}$ as \[
    f(x) = \begin{cases}
        -\ell \left( \hat{P}_{\vg}(\vx) + a \right) + \psi(a) & \hat{P}_{\vg}(\vx) \leq -a \\
        \psi(\hat{P}_{\vg}(\vx)) & \hat{P}_{\vg}(\vx) \in (-a, a) \\
        \ell \left( \hat{P}_{\vg}(\vx) - a \right) + \psi(a) & \hat{P}_{\vg}(\vx) \geq a
    \end{cases},
\]
where $\psi$ is as defined in Lemma \ref{lem:adagrad_norm_div}. Notice that $f$ is
continuously differentiable, bounded from below by $f_* = 0$, and $(L_0, L_1)$-smooth.

Next, set $\Xi = \{0, 1\}$ and define $F: \mathbb{R}^d \times \Xi \rightarrow \mathbb{R}^d$
as \[
    F(\vx, \xi) = \begin{cases}
        -\vg_1 & \hat{P}_{\vg}(\vx) \leq -a \text{ and } \xi = 0 \\
        -\vg_2 & \hat{P}_{\vg}(\vx) \leq -a \text{ and } \xi = 1 \\
        \psi'(x) & \hat{P}_{\vg}(\vx) \in \left( -a, a \right) \\
        \vg_1 & \hat{P}_{\vg}(\vx) \geq a \text{ and } \xi = 0 \\
        \vg_2 & \hat{P}_{\vg}(\vx) \geq a \text{ and } \xi = 1
    \end{cases}
\]
and define the distribution $\gD$ over $\Xi$ as \[
    \xi = \begin{cases}
        0 & \text{ with probability } p \\
        1 & \text{ with probability } 1-p \\
    \end{cases}
\]
for $\xi \sim \gD$. Notice that $F(x, \xi) = \nabla f(\vx)$ for $\vx$ with
$\hat{P}_{\vg}(\vx) \in \left( -a, a \right)$. Also, $\mathbb{E}_{\xi \sim \gD}[F(\vx,
\xi)] = p \vg_1 + (1-p) \vg_2 = \ell \vg = \nabla f(\vx)$ for $\vx$ with
$\hat{P}_{\vg}(\vx) \geq a$, and similarly for $\hat{P}_{\vg}(\vx) \leq -a$. Using the fact
that $\vg_1, \vg_2$ is a $p$-tricky pair, we have for all $x$:
\begin{equation} \label{eq:tricky_diverge_small_ell}
    \|\nabla f(\vx)\| \leq \ell = p c_1 + (1-p) c_2 = -p |c_1| + (1-p) |c_2| \Eqmark{i}{\leq} -p |c_1| + (1-p) \left( \frac{p |c_1| + G}{1-p} \right) = G.
\end{equation}
where $(i)$ uses \Eqref{eq:tricky_g2_max}. We can also use the tricky pair properties to
show that $(f, g, \gD)$ satisfies $\ell \geq \epsilon$ and the noise condition,
depending on the two cases in \Eqref{eq:tricky_g2_min}. In the first case,
\begin{align}
    |c_1| &\leq \frac{1-p}{p} \sigma_1 + \left( \frac{1-p}{p} \sigma_2 - 1 \right) \epsilon \label{eq:tricky_small_g1_max} \\
    |c_2| &\geq \frac{p |c_1| + \epsilon}{1-p}, \label{eq:tricky_small_g2_min}
\end{align}
so \[
    \ell = (1-p) |c_2| + p (-|c_1|) \geq \epsilon,
\] and
\begin{align*}
    \|\vg_2 - \ell \vg\| &= |c_2 - \ell| \|\vg\| \\
    &= c_2 - \ell \\
    &= \frac{p}{1-p} (\ell - c_1) \\
    &\Eqmark{i}{\leq} \frac{p}{1-p} \ell + \frac{p}{1-p} \left( \frac{1-p}{p} \sigma_1 + \left( \frac{1-p}{p} \sigma_2 - 1 \right) \epsilon \right) \\
    &= \frac{p}{1-p} \ell + \sigma_1 + \left( \sigma_2 - \frac{p}{1-p} \right) \epsilon \\
    &\Eqmark{ii}{\leq} \frac{p}{1-p} \ell + \sigma_1 + \left( \sigma_2 - \frac{p}{1-p} \right) \ell \\
    &= \sigma_1 + \sigma_2 \ell,
\end{align*}
where $(i)$ uses \Eqref{eq:tricky_small_g1_max} and $(ii)$ uses $\ell \geq \epsilon$. Also, \[
    \|\vg_1 - \ell \vg\| = |c_1 - \ell| \|\vg\| = \ell - c_1 = \frac{1-p}{p} (c_2 - \ell) \leq \frac{1-p}{p} (\sigma_1 + \sigma_2 \ell) \leq \sigma_1 + \sigma_2 \ell,
\]
where the last inequality uses $p > \frac{1}{2}$. Therefore $(f, g, \gD)$ satisfies
$\ell \geq \epsilon$ and the noise condition in the first case. In the second case,
\begin{align}
    |c_1| &> \frac{1-p}{p} \sigma_1 + \left( \frac{1-p}{p} \sigma_2 - 1 \right) \epsilon \label{eq:tricky_large_g1_min} \\
    |c_2| &\geq \frac{(\sigma_2+1) p|c_1| - \sigma_1}{(\sigma_2+1)(1-p)-1}, \label{eq:tricky_large_g2_min}
\end{align}
so
\begin{align*}
    c_2 &\geq \frac{(\sigma_2+1) p(-c_1) - \sigma_1}{(\sigma_2+1)(1-p)-1} \\
    ((\sigma_2+1)(1-p)-1) c_2 &\geq (\sigma_2+1) p(-c_1) - \sigma_1 \\
    c_2 &\leq \sigma_1 + (\sigma_2+1) p c_1 + (\sigma_2+1) (1-p) c_2 \\
    c_2 &\leq \sigma_1 + (\sigma_2+1) \ell \\
    c_2 - \ell &\leq \sigma_1 + \sigma_2 \ell,
\end{align*}
and
\begin{align*}
    c_2 &\geq \frac{(\sigma_2+1) p|c_1| - \sigma_1}{(\sigma_2+1)(1-p)-1} \\
    &= \frac{p}{1-p} |c_1| + \left( \frac{\sigma_2+1}{(\sigma_2+1)(1-p)-1} - \frac{1}{1-p} \right) p|c_1| - \frac{\sigma_1}{(\sigma_2+1)(1-p)-1} \\
    &= \frac{p}{1-p} |c_1| + \frac{1}{((\sigma_2+1)(1-p)-1)(1-p)} p|c_1| - \frac{\sigma_1}{(\sigma_2+1)(1-p)-1} \\
    &\Eqmark{i}{\geq} \frac{p}{1-p} |c_1| + \frac{(1-p) \sigma_1 + \left( (1-p) \sigma_2 - p \right) \epsilon}{((\sigma_2+1)(1-p)-1)(1-p)} - \frac{\sigma_1}{(\sigma_2+1)(1-p)-1} \\
    &\geq \frac{p}{1-p} |c_1| + \frac{\epsilon}{1-p},
\end{align*}
where $(i)$ uses \Eqref{eq:tricky_large_g1_min}. Therefore $\ell = pc_1 + (1-p) c_2 \geq
\epsilon$ as in the first case. Also as in the first case, $|c_1 - \ell| \leq |c_2 -
\ell|$. Therefore $(f, g, \gD)$ satisfies $\ell \geq \epsilon$ and the noise condition
in the second case.

Consider the initial point $x_0 = (a + \alpha(\vg_2) \|\vg_2\|) \vg$. Recall that \[
    \|\vg_2\| = |c_2| \leq |\ell| + |c_2 - \ell| \leq \sigma_1 + (\sigma_2 + 1) \ell \Eqmark{i}{\leq} \sigma_1 + (\sigma_2 + 1) G \Eqmark{ii}{\leq} \sigma_1 + (\sigma_2 + 1) \Delta L_1,
\]
where $(i)$ uses \Eqref{eq:tricky_diverge_small_ell} and $(ii)$ uses the definition of
$G$ (\Eqref{eq:G_def}). Also, by the tricky pair definition, $\|\vg_2\| = |c_2| \geq
\epsilon$. Therefore $\|\vg_2\| \in [\epsilon, \sigma_1 + (\sigma_2 + 1) \Delta L_1]$,
so we can use \Eqref{eq:singlestep_tricky_diverge_cond} to conclude that \[
    0 < \alpha(g_2) < \frac{4}{L_1 \|\vg_2\|} \log \left( 1 + \frac{\Delta L_1^2}{L_0} \right).
\]
Therefore, $f$ satisfies
\begin{align*}
    f(x_0) - f^* &= \psi(a) + \ell \alpha(\vg_2) \|\vg_2\| \\
    &= \frac{\ell}{L_1} - \frac{L_0}{L_1^2} \log \left( 1 + \frac{L_1 \ell}{L_0} \right) + \ell \alpha(\vg_2) \|\vg_2\| \\
    &\leq \frac{\ell}{L_1} + \ell \alpha(\vg_2) \|\vg_2\| \\
    &\leq \frac{\ell}{L_1} + \frac{4 \ell}{L_1} \log \left( 1 + \frac{\Delta L_1^2}{L_0} \right) \\
    &= \frac{\ell}{L_1} \left( 1 + 4 \log \left( 1 + \frac{\Delta L_1^2}{L_0} \right) \right) \\
    &\Eqmark{i}{\leq} \frac{\Delta L_1}{1 + 4 \log \left( 1 + \frac{\Delta L_1^2}{L_0} \right)} \frac{1}{L_1} \left( 1 + 4 \log \left( 1 + \frac{\Delta L_1^2}{L_0} \right) \right) \\
    &= \Delta,
\end{align*}
where $(i)$ uses \Eqref{eq:tricky_diverge_small_ell}. This shows that $(f, g, \gD) \in
\gF_{\text{aff}}(\Delta, L_0, L_1, \sigma_1, \sigma_2)$.

We now claim that $\|\nabla f(\vx_t)\| \geq \epsilon$ for all $t \geq 0$ with
probability $\delta$ when $A$ is initialized with $x_0 = (a + \alpha(\vg_2) \|\vg_2\|)
\vg$. To see this, consider the sequence \[
    y_t = \begin{cases}
        \frac{1}{\alpha(\vg_2) \|\vg_2\|} \left( \langle \vx_t, \vg \rangle - a \right) & \langle \vx_i, \vg \rangle \geq a \text{ for all } i \leq t \\
        0 & \text{ otherwise}
    \end{cases}.
\]
As long as $\langle \vx_t, \vg \rangle > a$, the sequence $y_t$ follows the exact same
distribution as the random walk in \Eqref{eq:random_walk} with $\lambda =
\frac{\alpha(\vg_1) \|\vg_1\|}{\alpha(\vg_2) \|\vg_2\|} > 0$.  Since $\vg_1, \vg_2$ is a
$(p, \delta)$-tricky pair, $\lambda \geq \lambda_0(p, \delta)$, so that $z_{p,\lambda}
\leq 1-\delta$ by the tricky pair definition. Therefore
\begin{align*}
    P \left( \|\nabla f(\vx_t)\| \geq \epsilon \text{ for all } t \geq 0 \right) &\geq P \left( \langle \vx_t, \vg \rangle > a \text{ for all } t \geq 0 \right) \\
    &= P \left( y_t > 0 \text{ for all } t \geq 0 \right) \\
    &= 1 - z_{p,\lambda} \\
    &\geq \delta.
\end{align*}
\end{proof}

\begin{lemma} \label{lem:app_singlestep_notricky_lr_ub}
Suppose that
\begin{equation}
    0 < \alpha(\vg) < \frac{4}{L_1 \|\vg\|} \log \left( 1 + \frac{L_1 \min(\|\vg\|, \Delta L_1)}{L_0} \right),
\end{equation}
for all $g \in \mathbb{R}^d$ with $\|\vg\| \in [\epsilon, \sigma_1 + (\sigma_2 + 1)
\Delta L_1]$, and that there do not exist any $(p, \delta)$-tricky pairs with respect to
$\alpha$. Suppose $\vg \in \mathbb{R}^d$ with $\|\vg\| = \epsilon$. If $\sigma_2 \geq
3$, then \[
    \alpha(\vg) \leq \tilde{\gO} \left( \frac{1}{ L_1 (\Delta L_1)^{1- \gamma_2 - \gamma_3} \epsilon^{\gamma_1} \sigma_1^{\gamma_2 + \gamma_3 - \gamma_1}} \right).
\]
On the other hand, if $\sigma_2 \in (1, 3)$, then \[
    \alpha(\vg) \leq ~\tilde{\gO} \left( \frac{1}{(\sigma_2-1)^{2-\gamma_4-\gamma_5-\gamma_6} \epsilon^{\gamma_4} L_1 (\Delta L_1)^{1-\gamma_5-\gamma_6} \sigma_1^{\gamma_5+\gamma_6-\gamma_4}} \right).
\]
\end{lemma}

\begin{proof}
Different from the proof sketch in Section \ref{sec:singlestep_sketch}, in our actual
construction below, we use two sequences $\{\vx_i\}$ and $\{\vy_i\}$ instead of one
sequence $\{z_i\}$. Every $\vx$ in the sequence $\{\vx_i\}$ satisfies
$|\hat{P}_{\vg}(\vx)| \in [\epsilon, \sigma_1 + (\sigma_2-1) \epsilon]$, and every $\vy$
in $\{\vy_i\}$ satisfies $|\hat{P}_{\vg}(\vx)| \in [\sigma_1, + (\sigma_2-1) \epsilon,
\sigma_1 + (\sigma_2-1) G]$ (see \Eqref{eq:G_def} for the definition of $G$).

Denote $\beta(\vx) = \alpha \left( \vx \right) \|\vx\|$, fix any $p_0 \in
\left( \frac{1}{2}, \frac{\sigma_2}{\sigma_2+1} \right)$ and define a sequence
$\{\vx_i\}_{i=0}^{\infty}$ as follows:
\begin{align*}
    \vx_0 &= \vg \\
    \vx_i &= (-1)^i \frac{p_0 \|\vx_{i-1}\| + \epsilon}{1-p_0} \frac{\vg}{\|\vg\|}.
\end{align*}
Also, denote $k_0 = \max \left\{ i \geq 0 : \|\vx_i\| \leq \frac{1-p_0}{p_0} \sigma_1 +
\left( \frac{1-p_0}{p_0} \sigma_2 - 1 \right) \epsilon \right\}$. We claim that, for
each $i$ with $0 \leq i \leq k_0$, the pair $(\vx_i, \vx_{i+1})$ satisfies all of the
conditions of a $(p_0, \delta)$-tricky pair, other than possibly \Eqref{eq:tricky_lr}.
Notice that $\|\vx_i\|$ is increasing and $\langle \vx_i, \vg \rangle$ has alternating
sign, so \Eqref{eq:tricky_opposite} and \Eqref{eq:tricky_g_min} are satisfied. Recall
that $\|\vx_i\| \leq \frac{1-p_0}{p_0} \sigma_1 + \left( \frac{1-p_0}{p_0} \sigma_2 - 1
\right) \epsilon$ by the definition of $k_0$. Since $\epsilon \leq G$ was assumed in
Theorem \ref{thm:singlestep}, this implies $\|\vx_i\| \leq \frac{1-p_0}{p_0} \sigma_1 +
\left( \frac{1-p_0}{p_0} \sigma_2 - 1 \right) G$. So \Eqref{eq:tricky_g1_max} is
satisfied. Since $\|\vx_i\| \leq \frac{1-p_0}{p_0} \sigma_1 + \left( \frac{1-p_0}{p_0}
\sigma_2 - 1 \right) \epsilon$, we must fulfill the first branch of the RHS of
\Eqref{eq:tricky_g2_min}. This only requires $\|\vx_{i+1}\| \geq \frac{p_0 \|vx_i\| +
\epsilon}{1-p_0}$, which holds by construction of the sequence $\{\vx_i\}$. Finally,
\Eqref{eq:tricky_g2_max} is satisfied again from $\epsilon \leq G$, since \[
    \|\vx_{i+1}\| = \frac{p_0 \|\vx_i\| + \epsilon}{1-p_0} \leq \frac{p_0 \|\vx_i\| + G}{1-p_0}.
\]
This verifies the claim that the pair $(\vx_i, \vx_{i+1})$ satisfies
\Eqref{eq:tricky_opposite} through \Eqref{eq:tricky_g2_max}. If $(\vx_i, \vx_{i+1})$
also satisfied \Eqref{eq:tricky_lr}, then it would be a $(p_0, \delta)$-tricky pair.
Since it was assumed that there do not exist any $(p, \delta)$-tricky pairs, it must be
that \Eqref{eq:tricky_lr} is not satisfied by $(\vx_i, \vx_{i+1})$, so that \[
    \beta(\vx_i) \leq \lambda_0(p_0, \delta) \beta(\vx_{i+1})
\]
for all $0 \leq i \leq k_0$. Choosing $i=0$ and unrolling to $i=k_0-2$:
\begin{equation} \label{eq:x0_ub_inter}
    \beta(\vx_0) \leq \left( \lambda_0(p_0, \delta) \right)^{k_0-1} \beta(\vx_{k_0-1}).
\end{equation}
Now choose $\vy_0 = (-1)^{k_0} (\sigma_1 + (\sigma_2-1) \epsilon) \frac{\vg}{\|\vg\|}$.
Then $\|\vy_0\| \geq \|\vx_{k_0}\|$ from the definition of $k_0$. We again want to show
that $(\vx_{k_0-1}, \vy_0)$ satisfies \Eqref{eq:tricky_opposite} through
\Eqref{eq:tricky_g2_max}. We can use an identical argument as above to demonstrate
\Eqref{eq:tricky_opposite} through \Eqref{eq:tricky_g2_min}, so it only remains to show
\Eqref{eq:tricky_g2_max}. It was assumed in the statement of Theorem
\ref{thm:singlestep} that $\sigma_1 + (\sigma_2 - 1) \epsilon \leq G$. Therefore \[
    \|\vy_0\| = \sigma_1 + (\sigma_2 - 1) \epsilon \leq \Delta L_1 \leq \frac{p \|\vx_{k_0-1}\| + G}{1-p},
\]
which demonstrates \Eqref{eq:tricky_g2_max}. This verifies the claim for $(\vx_{k_0-1},
\vy_0)$. Again, \Eqref{eq:tricky_lr} would imply that $(\vx_{k_0-1}, \vy_0)$ is a $(p_0,
\delta)$-tricky pair. But we assumed there are none, so \Eqref{eq:tricky_lr} cannot be
satisfied. Therefore \[
    \beta(\vx_{k_0-1}) \leq \lambda_0(p_0, \delta) \beta(\vy_0),
\]
and combining with \Eqref{eq:x0_ub_inter} yields
\begin{equation} \label{eq:notricky_inter_1}
    \beta(\vx_0) \leq \left( \lambda_0(p_0, \delta) \right)^{k_0} \beta(\vy_0).
\end{equation}

Now fix some $p_1 \in \left( \frac{1}{2}, \frac{\sigma_2}{\sigma_2+1} \right)$, define
the sequence $\{\vy_i\}_{i=0}^{\infty}$ as:
\begin{equation*}
    \vy_i = (-1)^{k_0+i} \frac{(\sigma_2+1) p_1 \|\vy_{i-1}\| - \sigma_1}{(\sigma_2+1)(1-p_1) - 1} \frac{\vg}{\|\vg\|}.
\end{equation*}
Denote $k_1 = \max \left\{ i \geq 0: \|\vy_i\| \leq \frac{1 - p_1}{p_1} \sigma_1 +
\left( \frac{1-p_1}{p_1} \sigma_2 - 1 \right) G \right\}$. Similarly as for the sequence
$\{\vx_i\}$, we claim that for each $i$ with $0 \leq i \leq k_1$, the pair $(\vy_i,
\vy_{i+1})$ satisfies all of the conditions of a $(p_1, \delta)$-tricky pair, other than
possibly \Eqref{eq:tricky_lr}. \Eqref{eq:tricky_opposite} and \Eqref{eq:tricky_g_min}
are satisfied, since $\|\vy_i\|$ is increasing and $\langle \vy_i, \vg \rangle$
alternates in sign. The upper bound of $\|\vy_i\|$ in the definition of $k_1$ ensures
that \Eqref{eq:tricky_g1_max} is satisfied. Since \[
    \|\vy_i\| \geq \|\vy_0| = \sigma_1 + (\sigma_2 - 1) \epsilon \geq \frac{1-p_1}{p_1} \sigma_1 + \left( \frac{1-p_1}{p_1} \sigma_2 - 1 \right) \epsilon,
\]
we must fulfill the second branch of the RHS of \Eqref{eq:tricky_g2_min}. This only
requires \[
    \|\vy_{i+1}\| \geq \frac{(\sigma_2+1) p_1 \|\vy_i\| - \sigma_1}{(\sigma_2+1)(1-p_1)-1},
\]
which holds by construction of the sequence $\{\vy_i\}$. Finally, to show
\Eqref{eq:tricky_g2_max}, we need \[
    \|\vy_{i+1}\| \leq \frac{p_1 \|\vy_i\| + G}{1-p_1},
\]
which is equivalent to
\begin{align*}
    \frac{(\sigma_2+1) p_1 \|\vy_{i-1}\| - \sigma_1}{(\sigma_2+1)(1-p_1) - 1} &\leq \frac{p_1 \|\vy_i\| + G}{1-p_1} \\
    (1-p_1) (\sigma_2+1) p_1 \|\vy_{i-1}\| - (1-p_1) \sigma_1 &\leq ((\sigma_2+1)(1-p_1)-1) p_1 \|\vy_i\| + ((\sigma_2+1)(1-p_1)-1) G \\
    p_1 \|\vy_i\| &\leq (1-p_1) \sigma_1 + ((\sigma_2+1)(1-p_1)-1) G \\
    \|\vy_i\| &\leq \frac{1-p_1}{p_1} \sigma_1 + \frac{(\sigma_2+1)(1-p_1)-1}{p_1} G \\
    \|\vy_i\| &\leq \frac{1-p_1}{p_1} \sigma_1 + \left( \frac{1-p_1}{p_1} \sigma_2 - 1 \right) G.
\end{align*}
All steps in this sequence are reversible, and the last inequality holds by the upper
bound of $\|\vy_i\|$ in the definition of $k_1$. Therefore, \Eqref{eq:tricky_g2_max} is
satisfied. This verifies the claim that $(\vy_i, \vy_{i+1})$ satisfies all of the
conditions of a $(p_1, \delta)$-tricky pair, other than possibly \Eqref{eq:tricky_lr}.
Again, \Eqref{eq:tricky_lr} cannot hold, since this would imply the existence of a
$(p_1, \delta)$-tricky pair, and we have already assumed otherwise. Therefore \[
    \beta(\vy_i) \leq \lambda_0(p_1, \delta) \beta(\vy_{i+1})
\]
for all $0 \leq i \leq k_1.$ Unrolling from $i=0$ to $i=k_1-1$ yields
\begin{equation} \label{eq:notricky_inter_2}
    \beta(\vy_0) \leq \left( \lambda_0(p_1, \delta) \right)^{k_1} \beta(\vy_{k_1}).
\end{equation}

Combining \Eqref{eq:notricky_inter_1} and \Eqref{eq:notricky_inter_2} yields
\begin{equation} \label{eq:notricky_inter_3}
    \beta(\vx_0) \leq \left( \lambda_0(p_0, \delta) \right)^{k_0} \left( \lambda_0(p_1, \delta) \right)^{k_1} \beta(\vy_{k_1}).
\end{equation}

We can use Lemma \ref{lem:exp_seq_ub} for the sequences $\{\|\vx_i\|\}_i$ and
$\{\|\vy_i\|\}_i$ to lower bound $k_0$ and $k_1$. For $k_0$, we apply Lemma
\ref{lem:exp_seq_ub} with
\begin{align*}
    a_0 = \epsilon, \quad r = \frac{p_0}{1-p_0}, \quad b = \frac{\epsilon}{1-p_0}, \quad A = \frac{1-p_0}{p_0} \sigma_1 + \left( \frac{1-p_0}{p_0} \sigma_2 - 1 \right) \epsilon.
\end{align*}
Then
\begin{align*}
    \frac{A (r-1) + b}{a_0 (r-1) + b} &= \frac{ \left( \frac{1-p_0}{p_0} \sigma_1 + \left( \frac{1-p_0}{p_0} \sigma_2 - 1 \right) \epsilon \right) \frac{2p_0-1}{1-p_0} + \frac{\epsilon}{1-p_0}}{\epsilon \frac{2p_0-1}{1-p_0} + \frac{\epsilon}{1-p_0} } \\
    &= \frac{ \left( \frac{1-p_0}{p_0} \sigma_1 + \left( \frac{1-p_0}{p_0} \sigma_2 - 1 \right) \epsilon \right) (2p_0-1) + \epsilon}{\epsilon (2p_0-1) + \epsilon } \\
    &= \frac{ \frac{1-p_0}{p_0} \left( \sigma_1 + \left( \sigma_2 - \frac{p_0}{1-p_0} \right) \epsilon \right) (2p_0-1) + \epsilon}{2 p_0 \epsilon } \\
    &= \frac{ \frac{(2p_0-1)(1-p_0)}{p_0} \left( \sigma_1 + \sigma_2 \epsilon \right) + 2 (1-p_0) \epsilon}{2 p_0 \epsilon } \\
    &= \frac{(2p_0-1)(1-p_0)}{2 p_0^2 \epsilon} \left( \sigma_1 + \sigma_2 \epsilon \right) + \frac{1-p_0}{p_0},
\end{align*}
so Lemma \ref{lem:exp_seq_ub} implies
\begin{align}
    k_0 &= \left\lfloor \frac{ \log \left( \frac{(2p_0-1)(1-p_0)}{2p_0^2 \epsilon} \left( \sigma_1 + \sigma_2 \epsilon \right) + \frac{1-p_0}{p_0} \right) }{ \log \frac{p_0}{1-p_0} } \right\rfloor \nonumber \\
    &\geq \frac{ \log \left( \frac{(2p_0-1)(1-p_0)}{2p_0^2 \epsilon} \left( \sigma_1 + \sigma_2 \epsilon \right) + \frac{1-p_0}{p_0} \right) }{ \log \frac{p_0}{1-p_0} } - 1 \nonumber \\
    &= \frac{ \log b_0 }{ \log \frac{p_0}{1-p_0} } - 1, \label{eq:tricky_seq_lb_1}
\end{align}
where we denoted \[
    b_0 = \frac{(2p_0-1)(1-p_0)}{2p_0^2 \epsilon} \left( \sigma_1 + \sigma_2 \epsilon \right) + \frac{1-p_0}{p_0}.
\]
Similarly, for $k_1$, we apply Lemma \ref{lem:exp_seq_ub} with
\begin{align*}
    a_0 &= \sigma_1 + (\sigma_2 - 1) \epsilon, \quad r = \frac{(\sigma_2+1) p_1}{(\sigma_2+1)(1-p_1) - 1}, \quad b = -\frac{\sigma_1}{(\sigma_2+1)(1-p_1) - 1} \\
    A &= \frac{1 - p_1}{p_1} \sigma_1 + \left( \frac{1-p_1}{p_1} \sigma_2 - 1 \right) G = \frac{1-p_1}{p_1} \left( \sigma_1 + \left( \sigma_2 - \frac{p_1}{1-p_1} \right) G \right).
\end{align*}
Then
\begin{equation*}
    r - 1 = \frac{(\sigma_2+1) p_1}{(\sigma_2+1) (1-p_1) - 1} - 1 = \frac{(\sigma_2+1) (2p_1-1) + 1}{(\sigma_2+1) (1-p_1) - 1},
\end{equation*}
so
\begin{align*}
    A (r-1) + b &= \frac{ \frac{1-p_1}{p_1} \left( \sigma_1 + \left( \sigma_2 - \frac{p_1}{1-p_1} \right) G \right) \left( (\sigma_2+1) (2p_1-1) + 1 \right) - \sigma_1 }{ (\sigma_2+1)(1-p_1) - 1 },
\end{align*}
and
\begin{align*}
    a_0 (r-1) + b &= \frac{ (\sigma_1 + (\sigma_2-1) \epsilon) \left( (\sigma_2+1) (2p_1-1) + 1 \right) - \sigma_1 }{ (\sigma_2+1)(1-p_1) - 1 }.
\end{align*}
So
\begin{align*}
    \frac{A (r-1) + b}{a_0 (r-1) + b} &= \frac{ \frac{1-p_1}{p_1} \left( \sigma_1 + \left( \sigma_2 - \frac{p_1}{1-p_1} \right) G \right) \left( (\sigma_2+1) (2p_1-1) + 1 \right) - \sigma_1 }{ (\sigma_1 + (\sigma_2-1) \epsilon) \left( (\sigma_2+1) (2p_1-1) + 1 \right) - \sigma_1 }.
\end{align*}
Denoting the RHS as $b_1$, this yields
\begin{align}
    k_1 &= \left\lfloor \frac{ \log b_1 }{ \log \frac{(\sigma_2+1)p_1}{(\sigma_2+1)(1-p_1)-1} } \right\rfloor \geq \frac{ \log b_1 }{ \log \frac{(\sigma_2+1)p_1}{(\sigma_2+1)(1-p_1)-1} } - 1. \label{eq:tricky_seq_lb_2}
\end{align}
Plugging \Eqref{eq:tricky_seq_lb_1} and \Eqref{eq:tricky_seq_lb_2} into \Eqref{eq:notricky_inter_3}:
\begin{equation*}
    \beta(\vx_0) \leq \left( \lambda_0(p_0, \delta) \right)^{-1} \left( \lambda_0(p_1, \delta) \right)^{-1} \left( \lambda_0(p_0, \delta) \right)^{\frac{ \log b_0 }{ \log \frac{p_0}{1-p_0} }} \left( \lambda_0(p_1, \delta) \right)^{\frac{ \log b_1 }{ \log \frac{(\sigma_2+1)p_1}{(\sigma_2+1)(1-p_1)-1} }} \beta(\vy_{k_1}).
\end{equation*}
Using the fact that for any $\rho$,
\begin{equation*}
    (\lambda_0(p_0, \delta))^{\log \rho} = (\lambda_0(p_0, \delta))^{\frac{\log \rho}{\log \lambda_0(p_0, \delta)} \log \lambda_0(p_0, \delta)} = \rho^{\log \lambda_0(p_0, \delta)},
\end{equation*}
we can choose $\rho = \log b_0$ and $\rho = \log b_1$,
\begin{equation} \label{eq:notricky_lr_ub_inter}
    \beta(\vx_0) \leq \left( \lambda_0(p_0, \delta) \right)^{-1} \left( \lambda_0(p_1, \delta) \right)^{-1} \left( \frac{1}{b_0} \right)^{\phi_0} \left( \frac{1}{b_1} \right)^{\phi_1} \beta(\vy_{k_1}),
\end{equation}
where
\begin{align*}
    \phi_0 &= \log \frac{1}{\lambda_0(p_0, \delta)} / \log \frac{p_0}{1-p_0} \\
    \phi_1 &= \log \frac{1}{\lambda_0(p_1, \delta)} / \log \frac{(\sigma_2+1)p_1}{(\sigma_2+1)(1-p_1)-1}.
\end{align*}
Note that $\phi_0 > \phi_1$, and denote $m = \frac{4}{L_1} \log \left( 1 + \frac{\Delta
L_1^2}{L_0} \right)$. We can also bound $\beta(\vy_{k_1})$ using the assumed condition
$\alpha(g) < \frac{4m}{|g|}$, since we previously showed that $(\vy_{k_1}, \vy_{k+1})$
satisfies \Eqref{eq:tricky_opposite} through \Eqref{eq:tricky_g2_max}. In particular,
\Eqref{eq:tricky_g1_max} implies that
\begin{equation*}
    \|\vy_{k_1}\| \leq \frac{1-p_1}{p_1} \sigma_1 + \left( \frac{1-p_1}{p_1} \sigma_2 - 1 \right) G \leq \sigma_1 + (\sigma_2 + 1) G \leq \sigma_1 + (\sigma_2 + 1) \Delta L_1,
\end{equation*}
so that $\|\vy_{k_1}\|$ falls within the range for which the bound on $\alpha(g)$
applies.  Therefore $\alpha(\vy_{k_1}) \leq \frac{4 m}{\|y_{k_1}\|}$, or
$\beta(\vy_{k_1}) \leq 4 m$.  Plugging back to \Eqref{eq:notricky_lr_ub_inter} yields
\begin{equation} \label{eq:notricky_ub_const}
    \beta(\vx_0) \leq 4 m \left( \lambda_0(p_0, \delta) \right)^{-1} \left( \lambda_0(p_1, \delta) \right)^{-1} \left( \frac{1}{b_0} \right)^{\phi_0} \left( \frac{1}{b_1} \right)^{\phi_1}.
\end{equation}

It only remains to choose $p_0$ and $p_1$ such that $b_0, b_1, \phi_0$, and $\phi_1$ can
be bounded in terms of the problem parameters. First, using Lemma
\ref{lem:random_walk_div}, we can rewrite $\lambda_0(p, \delta)$ as
\begin{align*}
    \lambda_0(p, \delta) &= \lambda_0(p, 0) + (\lambda_0(p, \delta) - \lambda_0(p, 0)) \\
    &= \frac{1-p}{p} + (\lambda_0(p, \delta) - \lambda_0(p, 0)) \\
    &= \frac{1-p}{p} + \zeta(p, \delta),
\end{align*}
so that
\begin{equation*}
    \frac{1}{\lambda_0(p, \delta)} = \frac{p}{1 - p + p \zeta(p, \delta)} = \frac{p}{1-p} \frac{1-p}{1-p+p \zeta(p, \delta)}.
\end{equation*}
We can then rewrite $\phi_0$ as
\begin{align*}
    \phi_0 &= \left( \log \frac{p_0}{1-p_0} + \log \left( \frac{1-p_0}{1-p_0+p_0 \zeta(p_0, \delta)} \right) \right) / \log \frac{p_0}{1-p_0} \\
    &= 1 - \frac{ \log \left( \frac{1-p_0+p_0 \zeta(p_0, \delta)}{1-p_0} \right) }{\log \frac{p_0}{1-p_0}}
\end{align*}
and $\phi_1$ as
\begin{align*}
    \phi_1 &= \left( \log \frac{p_1}{1-p_1} + \log \left( \frac{1-p_1}{1-p_1+p_1 \zeta(p_1, \delta)} \right) \right) / \log \frac{(\sigma_2+1)p_1}{(\sigma_2+1)(1-p_1)-1} \\
    &= \frac{ \log \left( \frac{p_1}{1-p_1} \right) }{\log \left( \frac{(\sigma_2+1)p_1}{(\sigma_2+1)(1-p_1)-1} \right)} - \frac{ \log \left( \frac{1-p_1+p_1 \zeta(p_1, \delta)}{1-p_1} \right) }{\log \left( \frac{(\sigma_2+1)p_1}{(\sigma_2+1)(1-p_1)-1} \right)}
\end{align*}

We choose $p_0$ and $p_1$ differently depending on the magnitude of $\sigma_2$. We
consider two cases: $\sigma_2 \geq 3$ (bounded away from $1$), and $\sigma_2 \in (1, 3)$
(close to $1$).

\paragraph{Case 1:} $\sigma_1 \geq 3$. Here we choose $p_0 = p_1 = \frac{2}{3}$, and
this satisfies $p_0, p_1 \in \left( \frac{1}{2}, \frac{\sigma_2}{\sigma_2+1} \right)$.
We now bound the remaining constants. For $b_0$:
\begin{align*}
    b_0 &= \frac{(2p_0-1)(1-p_0)}{2p_0^2 \epsilon} \left( \sigma_1 + \sigma_2 \epsilon \right) + \frac{1-p_0}{p_0} \\
    &= \frac{1}{8 \epsilon} \left( \sigma_1 + \sigma_2 \epsilon \right) + \frac{1}{2} \\
    &\geq \frac{\sigma_1}{8 \epsilon}.
\end{align*}
For $b_1$:
\begin{align*}
    b_1 &= \frac{ \frac{1-p_1}{p_1} \left( \sigma_1 + \left( \sigma_2 - \frac{p_1}{1-p_1} \right) G \right) \left( (\sigma_2+1) (2p_1-1) + 1 \right) - \sigma_1 }{ (\sigma_1 + (\sigma_2-1) \epsilon) \left( (\sigma_2+1) (2p_1-1) + 1 \right) - \sigma_1 } \\
    &= \frac{ \left( \frac{1-p_1}{p_1} \left( (\sigma_2+1)(2p_1-1) + 1 \right) - 1 \right) \sigma_1 + \left( \frac{1-p_1}{p_1} \sigma_2 - 1 \right) \left( (\sigma_2+1)(2p_1-1) + 1 \right) G }{ (\sigma_1 + (\sigma_2-1) \epsilon) \left( (\sigma_2+1) (2p_1-1) + 1 \right) - \sigma_1 } \\
    &\Eqmark{i}{\geq} \frac{ \left( \frac{1-p_1}{p_1} \sigma_2 - 1 \right) \left( (\sigma_2+1)(2p_1-1) + 1 \right) G }{ (\sigma_1 + (\sigma_2-1) \epsilon) \left( (\sigma_2+1) (2p_1-1) + 1 \right) - \sigma_1 } \\
    &\Eqmark{ii}{\geq} \frac{ \left( \frac{1-p_1}{p_1} \sigma_2 - 1 \right) \left( (\sigma_2+1)(2p_1-1) + 1 \right) G }{ (\sigma_1 + (\sigma_2-1) \sigma_1) \left( (\sigma_2+1) (2p_1-1) + 1 \right) - \sigma_1 } \\
    &= \frac{ \left( \frac{1-p_1}{p_1} \sigma_2 - 1 \right) \left( (\sigma_2+1)(2p_1-1) + 1 \right) G }{ \left( \sigma_2 \left( (\sigma_2+1) (2p_1-1) + 1 \right) - 1 \right) \sigma_1 } \\
    &\Eqmark{iii}{=} \frac{ \left( \frac{1}{2} \sigma_2 - 1 \right) \left( \frac{1}{3} \sigma_2 + \frac{4}{3} \right) G }{ \left( \sigma_2 \left( \frac{1}{3} \sigma_2 + \frac{4}{3} \right) - 1 \right) \sigma_1 }
    = \frac{ \left( \sigma_2 - 2 \right) \left( \sigma_2 + 4 \right) G }{ 2 \left( \sigma_2 \left( \sigma_2 + 4 \right) - 3 \right) \sigma_1 }
    \geq \frac{ \left( \sigma_2 - 2 \right) \left( \sigma_2 + 4 \right) G }{ 2 \sigma_2 \left( \sigma_2 + 4 \right) \sigma_1 } \\
    &= \frac{ \left( \sigma_2 - 2 \right) G }{ 2 \sigma_2 \sigma_1 } \Eqmark{iv}{\geq} \frac{G}{6 \sigma_1},
\end{align*}
where $(i)$ uses the fact that \[
    \frac{1-p_1}{p_1} \left( (\sigma_2+1)(2p_1-1) + 1 \right) - 1 = \frac{1}{2} \left( \frac{1}{3} \sigma_2 + \frac{4}{3} \right) - 1 = \frac{1}{6} \sigma_2 - \frac{1}{3} > 0,
\]
$(ii)$ uses $\epsilon \leq \sigma_1$ as assumed in the statement of Theorem
\ref{thm:singlestep}, $(iii)$ plugs in $p_1 = 2/3$, and $(iv)$ uses $\sigma_2 \geq 3$.
For $\phi_0:$
\begin{align*}
    \phi_0 &= 1 - \frac{ \log \left( \frac{1-p_0+p_0 \zeta(p_0, \delta)}{1-p_0} \right) }{\log \frac{p_0}{1-p_0}} = 1 - \frac{ \log \left( 1 + 2 \zeta(2/3, \delta) \right) }{\log 2} = 1 - \gamma_1,
\end{align*}
where we denoted \[
    \gamma_1 = \frac{ \log \left( 1 + 2 \zeta(2/3, \delta) \right) }{\log 2}.
\]
For $\phi_1$:
\begin{align*}
    \phi_1 &= \frac{ \log \left( \frac{p_1}{1-p_1} \right) }{\log \left( \frac{(\sigma_2+1)p_1}{(\sigma_2+1)(1-p_1)-1} \right)} - \frac{ \log \left( \frac{1-p_1+p_1 \zeta(p_1, \delta)}{1-p_1} \right) }{\log \left( \frac{(\sigma_2+1)p_1}{(\sigma_2+1)(1-p_1)-1} \right)} \\
    &= \frac{ \log 2 }{\log \left( 2 + \frac{6}{\sigma_2-2} \right)} - \frac{ \log \left( 1 + 2 \zeta(2/3, \delta) \right)}{\log \left( 2 + \frac{6}{\sigma_2-2} \right)} \\
    &= 1 - \left( 1 - \frac{ \log 2 }{\log \left( 2 + \frac{6}{\sigma_2-2} \right)} \right) - \frac{ \log \left( 1 + 2 \zeta(2/3, \delta) \right)}{\log \left( 2 + \frac{6}{\sigma_2-2} \right)} \\
    &= 1 - \gamma_2 - \gamma_3,
\end{align*}
where we denoted
\begin{align*}
    \gamma_2 &= 1 - \frac{ \log 2 }{\log \left( 2 + \frac{6}{\sigma_2-2} \right)} \\
    \gamma_3 &= \frac{ \log \left( 1 + 2 \zeta(2/3, \delta) \right)}{\log \left( 2 + \frac{6}{\sigma_2-2} \right)}.
\end{align*}
Finally, we can plug our bounds of $b_0, b_1, \phi_0, \phi_1$into
\Eqref{eq:notricky_ub_const}:
\begin{align*}
    \beta(\vx_0) &\leq 4 m \left( \lambda_0(2/3, \delta) \right)^{-1} \left( \lambda_0(2/3, \delta) \right)^{-1} \left( \frac{8 \epsilon}{\sigma_1} \right)^{1-\gamma_1} \left( \frac{6 \sigma_1}{G} \right)^{1-\gamma_2-\gamma_3} \\
    &\Eqmark{i}{\leq} 192 m \left( \lambda_0(2/3, 0) \right)^{-1} \left( \lambda_0(2/3, 0) \right)^{-1} \left( \frac{\epsilon}{\sigma_1} \right)^{1-\gamma_1} \left( \frac{\sigma_1}{G} \right)^{1-\gamma_2-\gamma_3} \\
    &\Eqmark{ii}{\leq} 768 m \left( \frac{\epsilon}{\sigma_1} \right)^{1-\gamma_1} \left( \frac{\sigma_1}{G} \right)^{1-\gamma_2-\gamma_3} \\
    &= 768 m \frac{ \epsilon^{1-\gamma_1} }{ G^{1- \gamma_2 - \gamma_3} \sigma_1^{\gamma_2 + \gamma_3 - \gamma_1}} \\
    &= 3072 \frac{ \epsilon^{1-\gamma_1} }{ L_1 G^{1- \gamma_2 - \gamma_3} \sigma_1^{\gamma_2 + \gamma_3 - \gamma_1}} \log \left( 1 + \frac{\Delta L_1^2}{L_0} \right),
\end{align*}
where $(i)$ uses the fact that $\lambda_0(p, \delta)$ is decreasing in terms of
$\delta$, and $(ii)$ uses $\lambda_0(p, 0) = \frac{1-p}{p}$ (from Lemma
\ref{lem:random_walk_balance}). As noted after \Eqref{eq:notricky_lr_ub_inter}, $\phi_0
> \phi_1$, so that $\gamma_1 < \gamma_2 + \gamma_3$. Replacing $\beta(\vx_0) =
\beta(\vg) = \alpha(\vg) \epsilon$ yields
\begin{align*}
    \alpha(\vg) &\leq \frac{3072}{ L_1 G^{1- \gamma_2 - \gamma_3} \epsilon^{\gamma_1} \sigma_1^{\gamma_2 + \gamma_3 - \gamma_1}} \log \left( 1 + \frac{\Delta L_1^2}{L_0} \right) \\
    &\leq \frac{3072}{ L_1 (\Delta L_1)^{1- \gamma_2 - \gamma_3} \epsilon^{\gamma_1} \sigma_1^{\gamma_2 + \gamma_3 - \gamma_1}} \log \left( 1 + \frac{\Delta L_1^2}{L_0} \right) \left( 1 + 4 \log \left( 1 + \frac{\Delta L_1^2}{L_0} \right) \right),
\end{align*}
which is the desired result.

\paragraph{Case 2:} $\sigma_2 \in (1, 3)$. Here we choose $p_0 = p_1 =
\frac{\sigma_2+5}{12}$, which satisfies $p_0, p_1 \in \left( \frac{1}{2},
\frac{\sigma_2}{\sigma_2+1} \right)$. With this choice, \[
    1-p_0 = \frac{-\sigma_2+7}{12}, \quad 2p_0 - 1 = \frac{\sigma_2-1}{6},
\]
and similarly for $p_1$. We can now bound the remaining constants. For $b_0$:
\begin{align*}
    b_0 &= \frac{(2p_0-1)(1-p_0)}{2p_0^2 \epsilon} \left( \sigma_1 + \sigma_2 \epsilon \right) + \frac{1-p_0}{p_0} \\
    &= \frac{(\sigma_2-1)(-\sigma_2+7)}{(\sigma_2+5)^2} \frac{ \sigma_1 + \sigma_2 \epsilon }{\epsilon} + \frac{-\sigma_2+7}{\sigma_2+5} \\
    &\Eqmark{i}{\geq} \frac{6 (\sigma_2-1)}{64} \frac{ \sigma_1 + \sigma_2 \epsilon }{\epsilon} + \frac{1}{2} \\
    &\geq \frac{3 (\sigma_2-1) \sigma_1}{32 \epsilon}
\end{align*}
where $(i)$ uses $\sigma_2 \in (1, 3)$. For $b_1$:
\begin{align*}
    b_1 &= \frac{ \frac{1-p_1}{p_1} \left( \sigma_1 + \left( \sigma_2 - \frac{p_1}{1-p_1} \right) G \right) \left( (\sigma_2+1) (2p_1-1) + 1 \right) - \sigma_1 }{ (\sigma_1 + (\sigma_2-1) \epsilon) \left( (\sigma_2+1) (2p_1-1) + 1 \right) - \sigma_1 } \\
    &= \frac{ \left( \frac{1-p_1}{p_1} \left( (\sigma_2+1)(2p_1-1) + 1 \right) - 1 \right) \sigma_1 + \left( \frac{1-p_1}{p_1} \sigma_2 - 1 \right) \left( (\sigma_2+1)(2p_1-1) + 1 \right) G }{ (\sigma_1 + (\sigma_2-1) \epsilon) \left( (\sigma_2+1) (2p_1-1) + 1 \right) - \sigma_1 } \\
    &\Eqmark{i}{\geq} \frac{ \left( \frac{1-p_1}{p_1} \sigma_2 - 1 \right) \left( (\sigma_2+1)(2p_1-1) + 1 \right) G }{ (\sigma_1 + (\sigma_2-1) \epsilon) \left( (\sigma_2+1) (2p_1-1) + 1 \right) - \sigma_1 } \\
    &\Eqmark{ii}{\geq} \frac{ \left( \frac{1-p_1}{p_1} \sigma_2 - 1 \right) \left( (\sigma_2+1)(2p_1-1) + 1 \right) G }{ (\sigma_1 + (\sigma_2-1) \sigma_1) \left( (\sigma_2+1) (2p_1-1) + 1 \right) - \sigma_1 } \\
    &= \frac{ \left( \frac{1-p_1}{p_1} \sigma_2 - 1 \right) \left( (\sigma_2+1)(2p_1-1) + 1 \right) G }{ \left( \sigma_2 \left( (\sigma_2+1) (2p_1-1) + 1 \right) - 1 \right) \sigma_1 } \\
    &\geq \frac{ \left( \frac{1-p_1}{p_1} \sigma_2 - 1 \right) G }{ \sigma_2 \sigma_1 } = \left( \frac{1-p_1}{p_1} - \frac{1}{\sigma_2} \right) \frac{G}{\sigma_1} = \left( \frac{-\sigma_2+7}{\sigma_2+5} - \frac{1}{\sigma_2} \right) \frac{G}{\sigma_1} \\
    &= \frac{ (\sigma_2-1) (-\sigma_2+5) }{\sigma_2 (\sigma_2+5)} \frac{G}{\sigma_1} \Eqmark{iii}{\geq} \frac{(\sigma_2-1) G}{12 \sigma_1},
\end{align*}
where $(i)$ uses
\begin{align*}
    \frac{1-p_1}{p_1} \left( (\sigma_2+1)(2p_1-1) + 1 \right) - 1 &= \frac{-\sigma_2+7}{\sigma_2+5} \left( (\sigma_2+1) \frac{\sigma_2-1}{6} + 1 \right) - 1 \\
    &= \frac{(-\sigma_2+7)(\sigma_2^2+5)}{6 (\sigma_2+5)} - 1 \\
    &= \frac{-\sigma_2+7}{6} - 1 > 0,
\end{align*}
$(ii)$ uses $\epsilon \leq \sigma_1$, as was assumed in the statement of Theorem
\ref{thm:singlestep}, and $(iii)$ uses $\sigma_2 \in (1, 3)$. For $\phi_0$:
\begin{align*}
    \phi_0 &= 1 - \frac{ \log \left( \frac{1-p_0+p_0 \zeta(p_0, \delta)}{1-p_0} \right) }{\log \frac{p_0}{1-p_0}} = 1 - \frac{ \log \left( 1 + \frac{\sigma_2+5}{-\sigma_2+7} \zeta(p_0, \delta) \right) }{\log \frac{\sigma_2+5}{-\sigma_2+7}} \\
    &= 1 - \frac{ \log \left( 1 + \frac{\sigma_2+5}{-\sigma_2+7} \zeta(p_0, \delta) \right) }{\log \left( \frac{12}{-\sigma_2+7} - 1 \right)} = 1 - \frac{ \log \left( 1 + 2 \zeta(p_0, \delta) \right) }{\log \left( \frac{12}{-\sigma_2+7} - 1 \right)} \\
    &= 1 - \gamma_4,
\end{align*}
where we denoted \[
    \gamma_4 = \frac{ \log \left( 1 + 2 \zeta(\frac{1}{12} (\sigma_2+5), \delta) \right) }{\log \left( \frac{12}{-\sigma_2+7} - 1 \right)}.
\]
Lastly, for $\phi_1$, notice that
\begin{align*}
    \frac{(\sigma_2+1)p_1}{(\sigma_2+1)(1-p_1)-1} &= \frac{(\sigma_2+1)(\sigma_2+5)}{12} \left( \frac{(\sigma_2+1)(-\sigma_2+7)}{12} - 1 \right)^{-1} \\
    &= \frac{(\sigma_2+1)(\sigma_2+5)}{(\sigma_2+1)(-\sigma_2+7) - 12} = \frac{(\sigma_2+1)(\sigma_2+5)}{(\sigma_2-1)(-\sigma_2+5)} \\
    &= \frac{3}{\sigma_2-1} + \frac{15}{-\sigma_2+5} - 1 \leq \frac{18}{\sigma_2-1} - 1.
\end{align*}
Therefore
\begin{align*}
    \phi_1 &= \frac{ \log \left( \frac{p_1}{1-p_1} \right) }{\log \left( \frac{(\sigma_2+1)p_1}{(\sigma_2+1)(1-p_1)-1} \right)} - \frac{ \log \left( \frac{1-p_1+p_1 \zeta(p_1, \delta)}{1-p_1} \right) }{\log \left( \frac{(\sigma_2+1)p_1}{(\sigma_2+1)(1-p_1)-1} \right)} \\
    &= \frac{ \log \left( \frac{\sigma_2+5}{-\sigma_2+7} \right) }{\log \left( \frac{18}{\sigma_2-1} - 1 \right)} - \frac{ \log \left( 1 + \frac{\sigma_2+5}{-\sigma_2+7} \zeta(p_1, \delta) \right) }{\log \left( \frac{18}{\sigma_2-1} - 1 \right)} \\
    &\geq \frac{ \log \left( \frac{12}{-\sigma_2+7} - 1 \right) }{\log \left( \frac{18}{\sigma_2-1} - 1 \right)} - \frac{ \log \left( 1 + 2 \zeta(p_1, \delta) \right) }{\log \left( \frac{18}{\sigma_2-1} - 1 \right)} \\
    &= 1 - \gamma_5 - \gamma_6
\end{align*}
where we denoted
\begin{align*}
    \gamma_5 &= 1 - \frac{ \log \left( \frac{12}{-\sigma_2+7} - 1 \right) }{\log \left( \frac{18}{\sigma_2-1} - 1 \right)} \\
    \gamma_6 &= \frac{ \log \left( 1 + 2 \zeta(\frac{1}{12}(\sigma_2+5), \delta) \right) }{\log \left( \frac{18}{\sigma_2-1} - 1 \right)}.
\end{align*}
Finally, we can plug our bounds of $b_0, b_1, \phi_0, \phi_1$ into \Eqref{eq:notricky_ub_const}:
\begin{align*}
    \beta(\vx_0) &\leq 4 m \left( \lambda_0(p_0, \delta) \right)^{-1} \left( \lambda_0(p_1, \delta) \right)^{-1} \left( \frac{32 \epsilon}{3 (\sigma_2-1) \sigma_1} \right)^{1-\gamma_4} \left( \frac{12 \sigma_1}{(\sigma_2-1) G} \right)^{1-\gamma_5-\gamma_6} \\
    &\Eqmark{i}{\leq} 512 m \left( \lambda_0(2/3, 0) \right)^{-1} \left( \lambda_0(2/3, 0) \right)^{-1} \left( \frac{\epsilon}{(\sigma_2-1) \sigma_1} \right)^{1-\gamma_4} \left( \frac{\sigma_1}{(\sigma_2-1) G} \right)^{1-\gamma_5-\gamma_6} \\
    &\Eqmark{ii}{\leq} 2048 m \left( \frac{\epsilon}{(\sigma_2-1) \sigma_1} \right)^{1-\gamma_4} \left( \frac{\sigma_1}{(\sigma_2-1) G} \right)^{1-\gamma_5-\gamma_6} \\
    &\leq 2048 m \frac{\epsilon^{1-\gamma_4}}{(\sigma_2-1)^{2-\gamma_4-\gamma_5-\gamma_6} G^{1-\gamma_5-\gamma_6} \sigma_1^{\gamma_5+\gamma_6-\gamma_4}} \\
    &\leq \frac{8192 \epsilon^{1-\gamma_4}}{(\sigma_2-1)^{2-\gamma_4-\gamma_5-\gamma_6} L_1 G^{1-\gamma_5-\gamma_6} \sigma_1^{\gamma_5+\gamma_6-\gamma_4}} \log \left( 1 + \frac{\Delta L_1^2}{L_0} \right) \\
    &= \frac{8192 \epsilon^{1-\gamma_4}}{(\sigma_2-1)^{2-\gamma_4-\gamma_5-\gamma_6} L_1 (\Delta L_1)^{1-\gamma_5-\gamma_6} \sigma_1^{\gamma_5+\gamma_6-\gamma_4}} \log \left( 1 + \frac{\Delta L_1^2}{L_0} \right) \left( 1 + 4 \log \left( 1 + \frac{\Delta L_1^2}{L_0} \right) \right),
\end{align*}
where $(i)$ uses the fact that $\lambda_0(p, \delta)$ is increasing in terms of $p$ and
decreasing in terms of $\delta$, and $(ii)$ uses $\lambda_0(p, 0) = \frac{1-p}{p}$ (from
Lemma \ref{lem:random_walk_balance}). As noted after \Eqref{eq:notricky_lr_ub_inter},
$\phi_0 > \phi_1$, so that $\gamma_4 < \gamma_5 + \gamma_6$. Replacing $\beta(\vx_0) =
\beta(\vg) = \alpha(\vg) \epsilon$ yields \[
    \alpha(\vg) \leq \frac{8192}{(\sigma_2-1)^{2-\gamma_4-\gamma_5-\gamma_6} \epsilon^{\gamma_4} L_1 (\Delta L_1)^{1-\gamma_5-\gamma_6} \sigma_1^{\gamma_5+\gamma_6-\gamma_4}} \log \left( 1 + \frac{\Delta L_1^2}{L_0} \right) \left( 1 + 4 \log \left( 1 + \frac{\Delta L_1^2}{L_0} \right) \right),
\]
which is the desired result.
\end{proof}

\begin{lemma} \label{lem:app_singlestep_linear_small_lr}
Define \[
    \alpha_0 = \max_{\|\vg\| = \epsilon} \alpha(\vg).
\]
There exists $f \in \mathcal{F}_{\textup{aff}}(\Delta, L_0, L_1, 0, 0)$ such that
$\|\nabla f(\vx_t)\| \geq \epsilon$ for all $t$ with \[
    t \leq \frac{\Delta}{2 \alpha_0 \epsilon^2}.
\]
\end{lemma}

\begin{proof}
Denote $a = \frac{1}{L_1} \log \left( 1 + \frac{L_1 \epsilon}{L_0} \right)$, and let
$\vg \in \mathbb{R}^d$ such that $\|\vg\| = \epsilon$ and $\alpha(\vg) = \alpha_0$.
Define the objective $f: \mathbb{R}^d \rightarrow \mathbb{R}$ as follows:
\begin{equation*}
    f(\vx) = \begin{cases}
        -\epsilon \left( \hat{P}_{\vg}(\vx) + a \right) + \psi(a) & \hat{P}_{\vg}(\vx) \leq -a \\
        \psi(\hat{P}_{\vg}(\vx)) & \hat{P}_{\vg}(\vx) \in (-a, a) \\
        \epsilon \left( \hat{P}_{\vg}(\vx) - a \right) + \psi(a) & \hat{P}_{\vg}(\vx) \geq a
    \end{cases},
\end{equation*}
where $\psi$ is as defined in Lemma \ref{lem:adagrad_norm_div}. It is straightforward
to show that $f$ is continuously differentiable, bounded from below ($f_* = 0$), and
$(L_0, L_1)$-smooth. Also, with the initial point $\vx_0 = \left( a + \frac{\Delta -
\psi(a)}{\epsilon} \right) \frac{\vg}{\|\vg\|}$, $f$ satisfies $f(x_0) - f^* = \Delta$.
So $f \in \mathcal{F}_{\text{aff}}(\Delta, L_0, L_1, 0, 0).$

Consider the execution of $A$ on $f$ from $\vx_0 = \left( a + \frac{\Delta -
\psi(a)}{\epsilon} \right) \frac{\vg}{\|\vg\|}$, and let $t_0 = \max \{ t \geq 0 ~|~
P_{\vg}(\vx_s) \geq a \text{ for all } 0 \leq s \leq t \}$. Then $\nabla f(x_t) = \vg$
for any $t \leq t_0$, so that \[
    \vx_{t+1} = \vx_t - \alpha(\nabla f(\vx_t)) \nabla f(\vx_t) = \vx_t - \alpha(\vg) \vg = \vx_t - \alpha_0 \vg,
\]
so \[
    \hat{P}_{\vg}(\vx_{t+1}) = \left\langle \vx_{t+1}, \frac{\vg}{\|\vg\|} \right\rangle = \left\langle \vx_t - \alpha_0 \vg, \frac{\vg}{\|\vg\|} \right\rangle = \hat{P}_{\vg}(\vx_t) - \alpha_0 \|\vg\| = \hat{P}_{\vg}(\vx_t) - \alpha_0 \epsilon.
\]
Unrolling over $t$ yields \[
    \hat{P}_{\vg}(\vx_{t+1}) = \hat{P}_{\vg}(\vx_0) - (t+1) \alpha_0 \epsilon.
\]
In particular, choosing $t = t_0$ yields ${\hat{P}_{\vg}(\vx_{t_0+1}) \geq
\hat{P}_{\vg}(\vx_0) - (t_0 + 1) \alpha(\epsilon) \epsilon}.$ By the definition of
$t_0$, we also know $\hat{P}_{\vg}(\vx_{t_0+1}) < a$. Therefore $\hat{P}_{\vg}(\vx_0) -
(t_0 + 1) \alpha(\epsilon) \epsilon < a$, and rearranging yields \[
    t_0 + 1 > \frac{ \hat{P}_{\vg}(\vx_0) - a }{\alpha(\epsilon) \epsilon} = \frac{\Delta - \psi(a)}{\alpha(\epsilon) \epsilon^2}.
\]
Also, \[
    \psi(a) = \frac{\epsilon}{L_1} - \log \left( 1 + \frac{L_1 \epsilon}{L_0} \right) \leq \frac{\epsilon}{L_1} \leq \frac{\Delta L_1}{2 L_1} \leq \frac{\Delta}{2},
\]
where the last inequality uses the condition $\epsilon \leq \frac{\Delta L_1}{2}$ from
Theorem \ref{thm:singlestep}. So \[
    t_0 + 1 > \frac{\Delta}{2 \alpha(\epsilon) \epsilon^2}.
\]
Therefore, $t \leq \frac{\Delta}{2 \alpha(\epsilon) \epsilon^2}$ implies that $t < t_0 +
1$, so that $\hat{P}_{\vg}(\vx_t) \geq a$ by the definition of $t_0$, and finally
$\|\nabla f(\vx_t)\| = \epsilon$.
\end{proof}

The following lemma is nearly identical to parts of the proof of Theorem 2 in
\cite{drori2020complexity}, with some small modifications to fit our requirements. We
include it here for the sake of completeness.

\begin{lemma} \label{lem:app_singlestep_quad_small_lr}
For any sufficiently large $d \in \mathbb{N}$ and any $\alpha: \mathbb{R}^d \rightarrow
\mathbb{R}^d$, there exists some $(f, g, \Xi) \in \gF_{\textup{aff}}(\Delta, L_0, L_1,
\sigma_1, \sigma_2)$ such that $\|\nabla f(\vx_t)\| = \epsilon$ for all $0 \leq t \leq
T$, where \[
    T = \frac{\Delta L_0 \sigma_1^2}{2 \epsilon^4}.
\]
\end{lemma}

\begin{proof}
Suppose $d \geq T$. Let $\alpha: \mathbb{R}^d \rightarrow \mathbb{R}^d$ and define $f:
\mathbb{R}^d \rightarrow \mathbb{R}$ as \[
    f(\vx) = \epsilon \langle \vx, \mathbf{e}_1 \rangle + \sum_{i=2}^T h_i(\langle \vx, \mathbf{e}_i \rangle),
\]
where
\begin{align*}
    h_i(x) &= \begin{cases}
        \frac{L_0}{4} a_i^2 & |x| < -a_i \\
        -\frac{L_0}{2} (x + a_i)^2 + \frac{L_0}{4} a_i^2 & |x| \in \left[ -a_i, -\frac{a_i}{2} \right] \\
        \frac{L_0}{2} x^2 & |x| \in \left( -\frac{a_i}{2}, \frac{b_i}{2} \right) \\
        -\frac{L_0}{2} (x - b_i)^2 + \frac{L_0}{4} b_i^2 & |x| \in \left[ \frac{b_i}{2}, b_i \right] \\
        \frac{L_0}{4} b_i^2 & |x| > b_i
    \end{cases} \\
    a_i &= \sigma_1 \alpha(\epsilon \mathbf{e}_1 + \sigma_1 \mathbf{e}_i) \\
    b_i &= \sigma_1 \alpha(\epsilon \mathbf{e}_1 - \sigma_1 \mathbf{e}_i).
\end{align*}
For any $\vx, \vy \in \mathbb{R}^d$,
\begin{align*}
    \|\nabla f(\vx) - \nabla f(\vy)\|^2 &= \sum_{i=1}^d \left( \nabla_i f(\vx) - \nabla_i f(\vy) \right)^2 \\
    &= \sum_{i=2}^d (h'(x_i) - h'(y_i))^2 \\
    &\Eqmark{i}{\leq} L_0^2 \sum_{i=2}^d (x_i - y_i)^2 \\
    &\leq L_0^2 \|\vx - \vy\|^2,
\end{align*}
where $(i)$ uses the fact that $h$ is $L_0$-smooth. Therefore $f$ is $L_0$-smooth, and
consequently is $(L_0, L_1)$-smooth. Also, define the following stochastic gradient for
$f$: \[
    F(\vx, \xi) = \nabla f(\vx) + (2 \xi - 1) \sigma_1 \mathbf{e}_{j(\vx)},
\]
where \[
    j(\vx) = \min \left\{ 1 \leq i \leq d \;|\; \langle \vx, \mathbf{e}_i \rangle = 0 \right\}.
\]
This oracle is defined so that the stochastic gradient noise at step $t$ only affects
coordinate $t+1$ (this will be shown later). Let $\gD$ be the distribution of $\xi$,
defined as $P(\xi = 0) = P(\xi = 1) = \frac{1}{2}$. With this definition, the stochastic
gradient $F$ satisfies
\begin{align*}
    \mathbb{E} [F(\vx, \xi)] &= \nabla f(\vx) \\
    \|F(\vx, \xi) - \nabla f(\vx)\| &\leq \sigma_1 \quad \text{(almost surely)}.
\end{align*}
Therefore, all of the conditions for $(f, F, \gD) \in \gF_{\text{aff}}(\Delta, L_0, L_1,
\sigma_1, \sigma_2)$ are satisfied other than the condition that $f$ is bounded from
below and $f(\vx_0) - \inf_{\vx} f(\vx) \leq \Delta$. This condition will be addressed
at the end of this lemma's proof.

Now consider the trajectory of $A$ on $f$ from the initial point $\vx_0 = \mathbf{0}$.
We claim that for all $0 \leq t \leq T$:
\begin{align}
    \langle \vx_t, \mathbf{e}_1 \rangle &= -\epsilon \sum_{i=0}^{t-1} \alpha(F(\vx_i, \xi_i)) \label{eq:singlestep_inductive_1} \\
    \langle \vx_t, \mathbf{e}_j \rangle &= \begin{cases}
        -a_j & \xi_j = 1 \\
        b_j & \xi_j = 0
    \end{cases} \quad \text{ for all } 2 \leq j \leq t+1 \label{eq:singlestep_inductive_2} \\
    \langle \vx_t, \mathbf{e}_j \rangle &= 0 \quad \text{ for all } j > t+1, \label{eq:singlestep_inductive_3}
\end{align}
which we will prove by induction on $t$. By construction, all three of the above hold
for the base case $t=0$. Now, suppose that they hold for some $0 \leq t \leq T-1$. Then
for $j \geq 2$, \[
    \nabla_j f(\vx_t) = h_j'(\langle \vx, \mathbf{e}_j \rangle) \Eqmark{i}{=} \begin{cases}
        h_j'(-a_j) & j \leq t+1 \text{ and } \xi_j = 1 \\
        h_j'(b_j) & j \leq t+1 \text{ and } \xi_j = 0 \\
        h_j'(0) & j > t + 1
    \end{cases} \Eqmark{ii}{=} 0,
\]
where $(i)$ uses \Eqref{eq:singlestep_inductive_2} and \Eqref{eq:singlestep_inductive_3}
from the induction hypothesis, and $(ii)$ comes from the definition of $h$. Therefore
$\nabla f(\vx_t) = \epsilon \mathbf{e}_1$. Also, \Eqref{eq:singlestep_inductive_2} and
\Eqref{eq:singlestep_inductive_3} imply that $j(\vx_t) = t+2$, so
\begin{align*}
    F(\vx_t, \xi_t) = \nabla f(\vx_t) + (2\xi_t - 1) \sigma_1 \mathbf{e}_{t+2}.
\end{align*}
Therefore, the next iterate $\vx_{t+1}$ is:
\begin{align*}
    \vx_{t+1} &= \vx_t - \alpha(F(\vx_t, \xi_t)) F(\vx_t, \xi_t) \\
    &= \vx_t - \alpha(F(\vx_t, \xi_t)) \left( \epsilon \mathbf{e}_1 + (2 \xi_t - 1) \sigma_1 \mathbf{e}_{t+2} \right) \\
    &\Eqmark{i}= -\epsilon \left( \sum_{i=0}^{t-1} \alpha(F(\vx_i, \xi_i)) \right) \mathbf{e}_1 + \sum_{i=2}^{t+1} \langle \vx_t, \mathbf{e}_i \rangle \mathbf{e}_i - \alpha(F(\vx_t, \xi_t)) \left( \epsilon \mathbf{e}_1 + (2 \xi_t - 1) \sigma_1 \mathbf{e}_{t+2} \right) \\
    &= -\epsilon \left( \sum_{i=0}^t \alpha(F(\vx_i, \xi_i)) \right) \mathbf{e}_1 + \sum_{i=2}^{t+1} \langle \vx_t, \mathbf{e}_i \rangle \mathbf{e}_i - (2 \xi_t - 1) \alpha(F(\vx_t, \xi_t)) \sigma_1 \mathbf{e}_{t+2},
\end{align*}
where $(i)$ uses \Eqref{eq:singlestep_inductive_1} from the inductive hypothesis. Notice
that the last term (i.e. the coefficient of $\mathbf{e}_{t+2}$) equals $-a_i$ when
$\xi_t = 1$ and it equals $b_i$ when $\xi_t = 0$. This proves
\Eqref{eq:singlestep_inductive_1}, \Eqref{eq:singlestep_inductive_2}, and
\Eqref{eq:singlestep_inductive_3} for step $t+1$. This completes the induction.
Together, these three equations imply that $\|\nabla f(\vx_t)\| = \epsilon$ for all $t
\leq T$, which is the desired conclusion.

The only remaining detail is the satisfaction of the condition $f(\vx_0) - \inf_{\vx}
f(\vx) \leq \Delta$. As currently stated, the objective $f$ does not satisfy this
condition because it is not even bounded from below due to the linear term $\epsilon
\langle \vx, \mathbf{e}_1 \rangle$. Similarly to \cite{drori2020complexity}, we instead
argue that there exists a lower bounded function $\hat{f}$ that has the same first-order
information as $f$ at all of the points $\vx_t$ for $0 \leq t \leq T$. If this happens,
then the behavior of $A$ when optimizing $\hat{f}$ is the same as that of $A$ when
optimizing $f$, so the conclusion $\|\nabla \hat{f}(\vx_t)\| = \epsilon$ still holds.
Specifically, we need $\hat{f}$ which is lower bounded and that satisfies: \[
    \nabla \hat{f}(\vx_t) = \nabla f(\vx_t), \quad \hat{f}(\vx_t) = f(\vx_t)
\]
for all $0 \leq t \leq T$. The existence of such an $\hat{f}$ follows immediately from
Lemma 1 of \cite{drori2020complexity}, and this $\hat{f}$ satisfies \[
    \inf_{\vx} \hat{f}(\vx) \geq \min_{0 \leq t \leq T} f(\vx_t) - \frac{3 \epsilon^2}{2L_0}.
\]
Therefore
\begin{align}
    \hat{f}(\vx_0) - \inf_{\vx} \hat{f}(\vx) &\leq \frac{3 \epsilon^2}{2L_0} - \min_{0 \leq t \leq T} f(\vx_t) \nonumber \\
    &\leq \frac{3 \epsilon^2}{2L_0} + \max_{0 \leq t \leq T} \left( -\epsilon \langle \vx_t, \mathbf{e}_1 \rangle - \sum_{i=2}^T h_i(\langle \vx_t, \mathbf{e}_i \rangle) \right) \nonumber \\
    &= \frac{3 \epsilon^2}{2L_0} + \max_{0 \leq t \leq T} \left( \epsilon^2 \sum_{i=0}^{t-1} \alpha(F(\vx_i, \xi_i)) - \sum_{i=2}^{t+1} h_i(\langle \vx_t, \mathbf{e}_i \rangle) \right). \label{eq:singlestep_shamir_inter}
\end{align}
Denote $\alpha_t = \alpha(F(\vx_t, \xi_t))$. Then for $2 \leq i \leq t+1$, \[
    h_i(\langle \vx_t, \mathbf{e}_i \rangle) = \frac{1}{4} L_0 \sigma_1^2 \alpha_{i-2}^2.
\]
Plugging into \Eqref{eq:singlestep_shamir_inter}:
\begin{align*}
    \hat{f}(\vx_0) - \inf_{\vx} \hat{f}(\vx) &\leq \frac{3 \epsilon^2}{2L_0} + \max_{0 \leq t \leq T} \left( \epsilon^2 \sum_{i=0}^{t-1} \alpha_i - \frac{1}{4} L_0 \sigma_1^2 \sum_{i=2}^{t+1} \alpha_{i-2}^2 \right) \\
    &\leq \frac{3 \epsilon^2}{2L_0} + \max_{0 \leq t \leq T} \sum_{i=0}^{t-1} \left( \underbrace{\epsilon^2 \alpha_i - \frac{1}{4} L_0 \sigma_1^2 \alpha_i^2}_{Q_i} \right)
\end{align*}
$Q_i$ can be upper bounded by the maximum value of $\epsilon^2 x - \frac{1}{4} L_0
\sigma_1^2 x^2$ as a function of $x$, which is $\frac{\epsilon^4}{L_0 \sigma_1^2}$.
Therefore
\begin{align*}
    \hat{f}(\vx_0) - \inf_{\vx} &\leq \frac{3 \epsilon^2}{2L_0} + \max_{0 \leq t \leq T} \sum_{i=0}^{t-1} \frac{\epsilon^4}{L_0 \sigma_1^2} \\
    &= \frac{3 \epsilon^2}{2L_0} + \frac{T \epsilon^4}{L_0 \sigma_1^2} \\
    &\Eqmark{i}{\leq} \frac{3 \epsilon^2}{2L_0} + \frac{\Delta}{2} \\
    &\Eqmark{ii}{\leq} \Delta,
\end{align*}
where $(i)$ uses \(T \leq \frac{\Delta L_0 \sigma_1^2}{2 \epsilon^4}\) and $(ii)$ uses
$\epsilon \leq \sqrt{\Delta L_0/3}$.
\end{proof}

\begin{theorem} \label{thm:app_singlestep}[Restatement of Theorem \ref{thm:singlestep}]
Let $\Delta, L_0, L_1, \sigma_1 > 0$ and $\sigma_2 > 1$. Denote \[
    G = \frac{\Delta L_1}{1 + 4 \log \left( 1 + \frac{\Delta L_1^2}{L_0} \right)},
\]
and suppose $G \geq \sigma_1$. Let \(
    0 < \epsilon \leq \min \left\{ \sigma_1, \frac{G}{2}, \frac{G - \sigma_1}{\sigma_2 - 1} \right\}.
\)
Let algorithm $A_{\text{ada}}$ denote single-step adaptive SGD with any step size
function $\alpha: \mathbb{R}^d \rightarrow \mathbb{R}$ for sufficiently large $d$, and
let $\gF = \gF_{\textup{aff}}(\Delta, L_0, L_1, \sigma_1, \sigma_2)$. If $\sigma_2 \geq
3$, then \[
    \gT(A_{\text{ada}}, \gF, \epsilon, \delta) \geq \tilde{\Omega} \left( \frac{(\Delta L_1)^{2 - \gamma_2 - \gamma_3} \sigma_1^{\gamma_2 + \gamma_3 - \gamma_1}}{\epsilon^{2-\gamma_1}} \right).
\]
Otherwise, if $\sigma_2 \in (1, 3)$, then \[
    \gT(A_{\text{ada}}, \gF, \epsilon, \delta) \geq \tilde{\Omega} \left( \frac{(\Delta L_1)^{2-\gamma_5-\gamma_6} \sigma_1^{\gamma_5+\gamma_6-\gamma_4} }{\epsilon^{2-\gamma_4}} (\sigma_2-1)^{2+\gamma_4-\gamma_5-\gamma_6} \right).
\]
\end{theorem}

\begin{proof}[Proof of Theorem \ref{thm:singlestep}]
We only need to combine Lemmas \ref{lem:app_singlestep_large_lr},
\ref{lem:app_singlestep_tricky_diverge}, \ref{lem:app_singlestep_notricky_lr_ub},
\ref{lem:app_singlestep_linear_small_lr} and \ref{lem:app_singlestep_quad_small_lr}. If
there exists any $\vg \in \mathbb{R}^d$ such that $\|\vg\| \in [\epsilon, \sigma_1 +
(\sigma_2 + 1) M]$ and \[
    \alpha(\vg) \leq 0 \quad \text{or} \quad \alpha(\vg) \geq \frac{4}{L_1 \|\vg\|} \log \left( 1 + \frac{L_1 \min(\|\vg\|, M)}{L_0} \right),
\]
then there exists some problem instance $(f, F, \gD)$ such that $\|\nabla f(\vx_t)\|
\geq \epsilon$ for all $t \geq 0$ (Lemma \ref{lem:app_singlestep_large_lr}). If no such
$\vg$ exists, and there exist any tricky pairs with respect to the stepsize function
$\alpha$, then there exists some problem instance $(f, F, \gD)$ such that $\|\nabla
f(\vx_t)\| \geq \epsilon$ for all $t \geq 0$ with probability at least $\delta$ (Lemma
\ref{lem:app_singlestep_tricky_diverge}). Suppose neither of these cases hold. Lemma
\ref{lem:app_singlestep_linear_small_lr} implies that there exists a problem instance
$(f, F, \gD)$ such that $\|\nabla f(\vx_t)\| \geq \epsilon$ for all $t$ with \[
    T \leq \frac{\Delta}{4 \alpha_0 \epsilon^2}.
\]
Since neither of the above cases hold, the conditions of Lemma
\ref{lem:app_singlestep_notricky_lr_ub} hold, so we can bound $\alpha_0$ with two cases.
If $\sigma_2 \geq 3$, then \[
    \alpha_0 \leq \frac{3072}{ L_1 (\Delta L_1)^{1- \gamma_2 - \gamma_3} \epsilon^{\gamma_1} \sigma_1^{\gamma_2 + \gamma_3 - \gamma_1}} \log \left( 1 + \frac{\Delta L_1^2}{L_0} \right) \left( 1 + 4 \log \left( 1 + \frac{\Delta L_1^2}{L_0} \right) \right),
\]
so $\|\nabla f(\vx_t)\| \geq \epsilon$ for all $t$ with
\begin{equation} \label{eq:singlestep_conclusion_1}
    t \leq \tilde{O} \left( \frac{\Delta L_0 \sigma_1^2}{\epsilon^4} + \frac{(\Delta L_1)^{2- \gamma_2 - \gamma_3} \sigma_1^{\gamma_2 + \gamma_3 - \gamma_1}}{\epsilon^{2-\gamma_1} } \right).
\end{equation}
If $\sigma_2 \in (1, 3)$, then \[
    \alpha_0 \leq \frac{8192}{(\sigma_2-1)^{2-\gamma_4-\gamma_5-\gamma_6} \epsilon^{\gamma_4} L_1 (\Delta L_1)^{1-\gamma_5-\gamma_6} \sigma_1^{\gamma_5+\gamma_6-\gamma_4}} \log \left( 1 + \frac{\Delta L_1^2}{L_0} \right) \left( 1 + 4 \log \left( 1 + \frac{\Delta L_1^2}{L_0} \right) \right),
\]
so $\|\nabla f(\vx_t)\| \geq \epsilon$ for all $t$ with
\begin{equation} \label{eq:singlestep_conclusion_2}
    t \leq \tilde{O} \left( \frac{\Delta L_0 \sigma_1^2}{\epsilon^4} + \frac{(\Delta L_1)^{2-\gamma_5-\gamma_6}}{\epsilon^{2-\gamma_4}} (\sigma_2-1)^2 \left( \frac{\sigma_1}{\sigma_2-1} \right)^{\gamma_5+\gamma_6-\gamma_4} \right).
\end{equation}
Lemma \ref{lem:app_singlestep_quad_small_lr} implies that \[
    \gT(A_{\text{ada}}, \gF, \epsilon, \delta) \geq \frac{\Delta L_0 \sigma_1^2}{2 \epsilon^4},
\]
and this can be combined with \Eqref{eq:singlestep_conclusion_1} and
\Eqref{eq:singlestep_conclusion_2} to obtain the two conclusions of Theorem
\ref{thm:singlestep}.
\end{proof}

\section{Auxiliary Lemmas}
Lemmas \ref{lem:random_walk_balance} and \ref{lem:random_walk_div} deal with the
asymmetric random walk described in the proof of Theorem \ref{thm:singlestep}. We
restate the associated definitions below.

For $p \in (0, 1)$ and $\lambda > 0$, consider the random walk parameterized by $(p, \lambda)$:
\begin{align}
    X_0 &= 1 \nonumber \\
    P(X_{t+1} = X_t + \lambda) &= p \label{eq:random_walk_def} \\
    P(X_{t+1} = X_t - 1) &= 1-p. \nonumber
\end{align}
Define
\begin{align*}
    z_{p,\lambda} &= P(\exists t > 0: X_t \leq 0) \\
    \lambda_0(p, \delta) &= \inf \left\{ \lambda \geq 0 : z_{p,\lambda} \leq 1-\delta \right\} \\
    \zeta(p, \delta) &= \lambda_0(p, \delta) - \lambda_0(p, 0).
\end{align*}
Informally, $z_{p,\lambda}$ is the probability that the random walk reaches a
non-positive value, and $\lambda_0(p, \delta)$ is the smallest $\lambda$ required to
ensure that the chance of never reaching a non-positive value is at least $\delta$.

\begin{lemma} \label{lem:random_walk_balance}
Let $X_t$ be as defined in \Eqref{eq:random_walk_def}. Then $\lambda_0(p,
0)=\frac{1-p}{p}$.
\end{lemma}

\begin{proof}
Denote $a\wedge b=\min(a,b)$. Define $\tau=\inf\limits_{t}\{t>0:X_t < 0\}$. Note that
$X_t=X_0+\sum_{i=1}^{t}\xi_i$, where $\{\xi_i\}_{i=1}^{t}$ are i.i.d. and follow the
same distribution: $\text{Pr}(\xi_i=\lambda)=p$ and $\text{Pr}(\xi_i=-1)=1-p$.

Now we first prove that $\{X_t-t((\lambda+1)p-1)\}_{t=1}^{\infty}$ is a martingale with
respect to itself. To see this, note that for any $t>0$, we have
\begin{equation*}
    \mathbb{E}\left[X_t-t((\lambda+1)p-1)\;|\; X_{t-1}\right]=\mathbb{E}\left[X_{t-1}+\xi_t-t((\lambda+1)p-1)\;|\; X_{t-1}\right]=X_{t-1}-(t-1)((\lambda+1)p-1),
\end{equation*}
where the last inequality holds because $\mathbb{E}\left[\xi_t\;|\;
X_{t-1}\right]=\lambda p -(1-p)=(\lambda+1)p-1$.

Let $T>0$ be a fixed constant. Note that $\tau \wedge T$ is a stopping time which is
almost surely bounded. Then by the optional sampling theorem,
\begin{equation*}
    \mathbb{E}\left[X_{\tau \wedge T}-(\tau \wedge T)((\lambda+1)p-1)\right]=\mathbb{E}[X_{0}]=1.
\end{equation*}
Therefore, we have
\begin{equation*}
    \mathbb{E}\left[\tau \wedge T\right]=\frac{\mathbb{E}\left[X_{\tau \wedge T}\right]-1}{(\lambda+1)p-1}.
\end{equation*}
Let $T\rightarrow \infty$. By the monotone convergence theorem,
\begin{equation*}
    \mathbb{E}[\tau]=\frac{\lim_{T\rightarrow \infty}\mathbb{E}\left[X_{\tau \wedge T}\right]-1}{(\lambda+1)p-1}
\end{equation*}
We consider the following cases.
\begin{itemize}
    \item If $\lambda<\frac{1-p}{p}$, then $(\lambda+1)p-1<0$. Combined with
        $\mathbb{E}[\tau]>0$, this implies \[
            \lim_{T\rightarrow \infty}\mathbb{E}\left[X_{\tau \wedge T}\right]-1<0.
        \]
        Now we show that $\mathbb{E}[\tau]<\infty$ by contradiction. If
        $\mathbb{E}[\tau]=\infty$, then $\lim_{T\rightarrow
        \infty}\mathbb{E}\left[X_{\tau \wedge T}\right]=-\infty$, which is impossible
        because $X_{\tau \wedge T}\geq -1$ for any $T$. Therefore
        $\text{Pr}(\tau=\infty)=0$ and $z_{p,\lambda}=1$.
    \item If $\lambda\geq \frac{1-p}{p}$, then $(\lambda+1)p-1 \geq 0$. Combined with
        $\mathbb{E}[\tau]>0$, this implies \[
            \lim_{T\rightarrow \infty}\mathbb{E}\left[X_{\tau \wedge T}\right]-1\geq 0.
        \]
        Now we show that $\text{Pr}(\tau=\infty)>0$ by contradiction. If
        $\text{Pr}(\tau=\infty)=0$, then by the bounded convergence theorem, we have
        $\lim_{T\rightarrow\infty}\mathbb{E}\left[X_{\tau \wedge
        T}\right]=\mathbb{E}[X_{\tau}] < 0$, which contradicts $\lim_{T\rightarrow
        \infty}\mathbb{E}\left[X_{\tau \wedge T}\right]-1\geq 0$. Therefore
        $\text{Pr}(\tau=\infty)>0$ and $z_{p,\lambda} < 1$.
\end{itemize}
Therefore $\lambda_0(p)=\frac{1-p}{p}$.
\end{proof}

\begin{lemma} \label{lem:random_walk_div}
Let $X_t$ be as defined in \Eqref{eq:random_walk_def}. Then $\lim_{\delta \rightarrow
0^+} \lambda_0(p, \delta) = \lambda_0(p, 0)$ for all $p \in (0, 1)$.
\end{lemma}

\begin{proof}
The idea of the proof is, given some $\lambda$, to find some $\alpha \in (0, 1)$ such
that $Y_t = \alpha^{X_t}$ is a martingale. We can then apply the optional sampling
theorem to $\alpha^{X_t}$ in order to get a bound of $z_{p,\lambda}$ in terms of
$\lambda$, which we can use to upper bound $\lambda_0(p, \delta)$. This upper bound
goes to $\lambda_0(p, 0)$ as $\delta \rightarrow 0^+$. Combining with the fact that
$\lambda_0(p, \delta)$ is increasing in terms of $\delta$ yields $\lim_{\delta
\rightarrow 0^+} \lambda_0(p, \delta) = \lambda_0(p, 0)$.

Let $p \in (0, 1)$ and $\delta \in (0, p)$. We want to find $\tilde{\lambda}$ such that
$z_{p,\tilde{\lambda}} \leq 1-\delta$ (so that $\lambda_0(p, \delta) \leq
\tilde{\lambda}$) and $\tilde{\lambda} \rightarrow \lambda_0(p, 0)$ as $\delta
\rightarrow 0$. First, we need $\alpha \in (0, 1)$ such that $\alpha^{X_t}$ is a
martingale. This requires
\begin{align}
    \mathbb{E}[\alpha^{X_{t+1}} \;|\; X_t] &= \alpha^{X_t} \nonumber \\
    p \alpha^{X_t+\lambda} + (1-p) \alpha^{X_t-1} &= \alpha^{X_t} \nonumber \\
    p \alpha^{X_t+\lambda} - \alpha^{X_t} + (1-p) \alpha^{X_t-1} &= 0 \nonumber \\
    p \alpha^{\lambda+1} - \alpha + (1-p) &= 0, \label{eq:martingale_cond}
\end{align}
so we are looking for a root of $h_{\lambda}(x) = p x^{\lambda+1} - x + (1-p)$ in the
interval $x \in (0, 1)$. We claim that for all $\lambda > \frac{1-p}{p}$, there is
exactly one root of $h_{\lambda}$ in $(0, 1)$. To see that such a root exists, notice
that $h_{\lambda}(0) = 1-p > 0$ and $h_{\lambda}(1) = 0$. Also, $h_{\lambda}'(1) = p
(\lambda + 1) - 1 > 0$ (since $\lambda > \frac{1-p}{p}$). Therefore $h_{\lambda}(1-z) <
0$ for sufficiently small $z > 0$. Then we can apply the intermediate value theorem to
$h_{\lambda}$ at $h_{\lambda}(0) > 0$ and $h_{\lambda}(1-z) < 0$ to conclude that $h$
must have a root in $(0, 1)$.

To see that this root is unique, note that $h_{\lambda}$ is strictly convex in $(0,
\infty)$, since $h_{\lambda}''(x) = p \lambda (\lambda + 1) x^{\lambda-1} > 0$ for $x >
0$. Suppose $h_{\lambda}$ had two roots $x_1, x_2 \in (0, 1)$, with $x_1 < x_2$.
Letting $\alpha = (x_2-x_1)/(1-x_1)$, we have by strict convexity $h_{\lambda}(x_2) <
(1-\alpha) h_{\lambda}(x_1) + \alpha h_{\lambda}(1) = 0$, which contradicts
$h_{\lambda}(x_2) = 0$. Therefore, $h_{\lambda}$ has a unique root in $(0, 1)$ for every
$\lambda > \frac{p}{1-p}$. Denote this root as $r(\lambda)$.

Now define $\tilde{\lambda} = \inf \left\{ \lambda > \frac{1-p}{p} \;|\; r(\lambda) \leq
1-\delta \right\}$ (the threshold $1-\delta$ will be used later to show $z_{p,
\tilde{\lambda}} \leq 1-\delta$). In order to show that $\tilde{\lambda}$ exists and
that $\tilde{\lambda} \rightarrow \lambda_0(p, 0)$ as $\delta \rightarrow 0$, we need a
few facts about $r(\lambda)$. Specifically, we need
\begin{align}
    &r(\lambda) \text{ is decreasing} \label{eq:root_dec} \\
    \lim_{\lambda \rightarrow \frac{1-p}{p}^+} &r(\lambda) = 1 \label{eq:root_small_lim} \\
    \lim_{\lambda \rightarrow \infty} &r(\lambda) = 1-p. \label{eq:root_large_lim}
\end{align}
To see \Eqref{eq:root_dec}, let $\lambda_2 > \lambda_1 > \frac{1-p}{p}$. For any $x \in
\left[ r(\lambda_1) , 1 \right)$: \[
    h_{\lambda_2(x)} < h_{\lambda_1(x)} \leq \left( 1 - \frac{x-r(\lambda_1)}{1-r(\lambda_1)} \right) h_{\lambda_1}(r(\lambda_1)) + \frac{x-r(\lambda_1)}{1-r(\lambda_1)} h_{\lambda_1}(1) = 0,
\]
where the first inequality uses the fact that $h_{\lambda}(x)$ is decreasing in terms of
$\lambda$ for any fixed $x$, and the second inequality uses convexity. Then
$r(\lambda_2)$ cannot lie in the interval $\left[ r(\lambda_1) , 1 \right)$, so
$r(\lambda_2) < r(\lambda_1)$. This shows that $r(\lambda)$ is decreasing.

To prove \Eqref{eq:root_small_lim}, notice that $r(\lambda) \in (0, 1)$ already implies
$\lim_{\lambda \rightarrow \frac{1-p}{p}^+} \leq 1$. So it suffices to show for any
$\epsilon \in (0, 1)$ that $r(\lambda) > 1-\epsilon$ for sufficiently small $\lambda$.
Denoting $\ell = \frac{1-p}{p}$,
\begin{equation} \label{eq:h_pos_lim}
    \lim_{\lambda \rightarrow \frac{1-p}{p}^+} h_{\lambda}(1-\epsilon) = p (1-\epsilon)^{1/p} -x + (1-p) = h_{\ell}(1-\epsilon) > h_{\ell}(1) - \epsilon h_{\ell}'(1) = 0,
\end{equation}
where the inequality uses strict convexity of $h_{\ell}$ and the last equality uses
$h_{\ell}(1) = h_{\ell}'(1) = 0$. Also
\begin{equation} \label{eq:h_prime_neg_lim}
    \lim_{\lambda \rightarrow \frac{1-p}{p}^+} h_{\lambda}'(1-\epsilon) = \lim_{\lambda \rightarrow \frac{1-p}{p}^+} p(\lambda+1) x^{\lambda} - 1 = (1-\epsilon)^{\frac{1-p}{p}} - 1 < 0.
\end{equation}
Together, \Eqref{eq:h_pos_lim} and \Eqref{eq:h_prime_neg_lim} tell us that for
sufficiently small $\lambda$: $h_{\lambda}(1-\epsilon) > 0$ and
$h_{\lambda}'(1-\epsilon) < 0$. Then for any $x \leq 1-\epsilon$, \[
    h_{\lambda}(x) \geq h_{\lambda}(1-\epsilon) + (x-(1-\epsilon)) h_{\lambda}'(1-\epsilon) > 0.
\]
In other words, for sufficiently large $\lambda$, the root of $h_{\lambda}$ cannot be
smaller than $1-\epsilon$, or $r(\lambda) > 1-\epsilon$. This proves
\Eqref{eq:root_small_lim}.

For \Eqref{eq:root_large_lim}, let $x \in (0, 1)$ and $\lambda > \frac{1-p}{p}$. Then by
strict convexity of $h_{\lambda}$: \[
    h_{\lambda}(x) > h_{\lambda}(0) + x h_{\lambda}'(0) = (1-p) - x,
\]
so $h_{\lambda}(x) > 0$ for any $x \leq 1-p$. Therefore $r(\lambda) > 1-p$ for any
$\lambda$, so that $\lim_{\lambda \rightarrow \infty} r(\lambda) \geq 1-p$. We can also
show that $\lim_{\lambda \rightarrow \infty} r(\lambda) \leq 1-p$ by showing for any
$\epsilon > 0$ that $r(\lambda) \leq 1-p+\epsilon$ for sufficiently large $\lambda$. By
convexity of $h_{\lambda}$: \[
    \lim_{\lambda \rightarrow \infty} h_{\lambda}(1-p+\epsilon) = \lim_{\lambda \rightarrow \infty} p (1-p+\epsilon)^{\lambda+1} - (1-p+\epsilon) + (1-p) = -\epsilon.
\]
So $h_{\lambda}(1-p+\epsilon) < -\epsilon/2$ sufficiently large $\lambda$. Then for any
$x \geq 1-p+\epsilon$, \[
    h_{\lambda}(x) \leq (1-\alpha) h_{\lambda}(1-p+\epsilon) + \alpha h_{\lambda}(1) = -(1-\alpha) \epsilon < 0.
\]
So the root of $h_{\lambda}$ must be smaller than $1-p+\epsilon$, or $r(\lambda) \leq
1-p+\epsilon$. This proves that $\lim_{\lambda \rightarrow \infty} r(\lambda) \leq 1-p$,
and completes the proof of \Eqref{eq:root_large_lim}.

Recall the definition $\tilde{\lambda} = \inf \left\{ \lambda > \frac{1-p}{p} \;|\;
r(\lambda) \leq 1-\delta \right\}$. \Eqref{eq:root_large_lim} and
\Eqref{eq:root_small_lim} together imply that $\tilde{\lambda}$ exists, since $\delta
\in (0, p) \implies 1-\delta \in (1-p, 1)$. Also, \Eqref{eq:root_dec} and
\Eqref{eq:root_small_lim} imply that $\tilde{\lambda} \rightarrow \frac{1-p}{p} =
\lambda_0(p, \delta)$ as $\delta \rightarrow 0$.

We can now consider the random walk $X_t$ defined in \Eqref{eq:random_walk_def} with
$\lambda = \tilde{\lambda}$. Our goal is to show that $z_{p,\tilde{\lambda}} \leq
1-\delta$, which implies that $\lambda_0(p, \delta) \leq \tilde{\lambda}$. Let $\alpha =
r(\tilde{\lambda}) \leq 1-\delta$. We have constructed $\alpha$ to be a root of
$h_{\tilde{\lambda}}$, so that $\alpha^{X_t}$ is a martingale, as shown in
\Eqref{eq:martingale_cond}. Let $T_0 = \inf \{ t \geq 0 \;|\; X_t \leq 0 \}$, and
$T_b=\inf\{t\geq 0\;|\; X_t \geq b\}$, where $b>0$. Define $T=\min(T_0,T_b)$. We have
$\alpha^{X_{\min(T, n)}}$ is bounded for any $n$ and it is nonnegative, therefore by
optional sampling theorem and martingale convergence theorem (e.g., Theorem 4.8.2
in~\cite{durrett2019probability}), we have
\begin{align*}
    \alpha &= \alpha^{X_0} \\
    &= \mathbb{E}\left[\alpha^{X_T}\right] \\
    &= \text{Pr}(T_0<T_b) \alpha^{X_{T_0}} + (1-\text{Pr}(T_0<T_b)) \alpha^{X_{T_b}} \\
    &\geq \text{Pr}(T_0<T_b)+(1-\text{Pr}(T_0<T_b))\alpha^{b+\lambda},
\end{align*}
where the inequality holds due to $\alpha \in (0, 1)$, $X_{T_0} \leq 0$, and $X_{T_b}
\leq b + \lambda$. Let $b\rightarrow \infty$ on both sides, and note that $\alpha^b
\rightarrow 0$, we have
\begin{equation}
    \alpha\geq \text{Pr}(T_0<\infty)=z_{p,\tilde{\lambda}}.
\end{equation}


Therefore \[
    z_{p,\tilde{\lambda}} \leq \alpha \leq 1 - \delta,
\]
so that $\lambda_0(p, \delta) \leq \tilde{\lambda}$. Finally,
\begin{align*}
    \lambda_0(p, 0) \leq \lim_{\delta \rightarrow 0^+} \lambda_0(p, \delta) \leq \lim_{\delta \rightarrow 0^+} \tilde{\lambda} = \lambda_0(p, 0),
\end{align*}
so that $\lim_{\delta \rightarrow 0^+} \lambda_0(p, \delta) = \lambda_0(p, 0)$.

\end{proof}

\begin{lemma} \label{lem:exp_seq_ub}
Let $\{a_i\}_{i=0}^{\infty}$ be a positive sequence of reals satisfying $a_{i+1} = r a_i
+ b$ for $r > 1$, and let $A \geq a_0$. Define $k = \max \{ i \geq 0: a_i \leq A \}$.
Then
\begin{equation*}
    k = \left\lfloor \frac{\log \left( \frac{A(r-1) + b}{a_0(r-1) + b} \right)}{\log r} \right\rfloor
\end{equation*}
\end{lemma}

\begin{proof}
It is straightforward to show by induction that for any $i \geq 0$:
\begin{align*}
    a_i = a_0 r^i + b \sum_{j=0}^{i-1} r^j = a_0 r^i + b \frac{r^i - 1}{r-1} = r^i \left( a_0 + \frac{b}{r-1} \right) - \frac{b}{r-1}.
\end{align*}
Then $a_i \leq A$ if and only if
\begin{align*}
    r^i \left( a_0 + \frac{b}{r-1} \right) - \frac{b}{r-1} &\leq A \\
    r^i \left( a_0 + \frac{b}{r-1} \right) &\leq A + \frac{b}{r-1} \\
    r^i &\leq \frac{A + \frac{b}{r-1}}{a_0 + \frac{b}{r-1}} = \frac{A(r-1) + b}{a_0(r-1) + b} \\
    i &\leq \frac{ \log \left( \frac{A(r-1) + b}{a_0(r-1) + b} \right) }{\log r}.
\end{align*}
So $k$ is the largest integer smaller than or equal to the RHS of the above.
\end{proof}

\section{Discussion on Stabilization Constant $\gamma$} \label{app:gamma}
In Theorem \ref{thm:adagrad_norm}, we showed that Decorrelated AdaGrad-Norm exhibits a
quadratic dependence on $\Delta, L_1$ in the dominating term of its convergence rate, so
that the number of iterations required to find an $\epsilon$-stationary point is
$\Omega(\Delta^2 L_1^2 \sigma^2 \epsilon^{-4})$. This result depends on the condition
$\gamma \leq \tilde{\mathcal{O}}(\Delta L_1)$, which covers the standard protocol in
practice of choosing $\gamma$ to be a small constant, e.g. $\gamma = 1e-8$. However, it
is natural to ask whether our result can extend to any choice of $\gamma$.

In this section, we answer this question in the deterministic setting, that is,
with $\sigma = 0$, we show that the lower bound of Theorem
\ref{thm:adagrad_norm} can be recovered even if the condition $\gamma \leq
\tilde{\mathcal{O}}(\Delta L_1)$ is removed. This shows that (deterministic)
Decorrelated AdaGrad-Norm cannot recover the optimal complexity from the smooth
(deterministic) setting, no matter the choice of $\gamma$. This result is stated
below.
\begin{theorem} \label{thm:adagrad_norm_deterministic}
Denote $\gF_{\text{det}} = \gF_{\textup{as}}(\Delta, L_0, L_1, 0)$, and let algorithm
$A_{\text{DAN}}$ denote Decorrelated AdaGrad-Norm (\Eqref{eq:adagrad_norm}) with
parameters $\eta, \gamma > 0$. Let \(
    0 < \epsilon \leq \min \left\{ \frac{\Delta L_1}{2}, \sqrt{\frac{\Delta \gamma}{4 \eta}} \right\}.
\)
If $\Delta L_1^2 \geq L_0$, then \[
    \gT(A_{\text{DAN}}, \gF_{\text{det}}, \epsilon) \geq \tilde{\Omega} \left( \frac{\Delta^2 L_1^2}{\epsilon^2} \right).
\]
\end{theorem}
The proof structure is similar as Theorems \ref{thm:adagrad_norm},
\ref{thm:adagrad}, and \ref{thm:vanilla_adagrad}, by splitting into cases
depending on the choice of $\eta$ and $\gamma$. However, for this proof we split
into cases slightly differently than in these three theorems; here, the cases
are determined by the magnitude of $\eta$ and $\gamma/\eta$. The proof relies on
Lemma \ref{lem:adagrad_norm_div} for one case and reuses the hard instance of
Lemma \ref{lem:adagrad_slow} for the other.

\begin{proof}
We consider two cases: In the first case, both of the following hold:
\begin{equation}
    \eta \geq 1/L_1 \quad \text{and} \quad \gamma/\eta \leq \frac{\Delta L_1^2}{8 \log \left( 1 + \frac{48 \Delta L_1^2}{L_0} \right)}.
\end{equation}
In the second case, one or both of these two conditions fail:
\begin{equation}
    \eta \leq 1/L_1 \quad \text{or} \quad \gamma/\eta \geq \frac{\Delta L_1^2}{8 \log \left( 1 + \frac{48 \Delta L_1^2}{L_0} \right)}.
\end{equation}
We will show that in the first case, there exists an objective for which
Decorrelated AdaGrad-Norm will never converge, and in the second case, there
exists an objective for which convergence requires $\Omega(\Delta^2 L_1^2
\epsilon^{-2})$ iterations.

\paragraph{Case 1} This case is the simpler of the two, since we can directly
apply Lemma \ref{lem:app_adagrad_norm_div}. The conditions of this lemma are
immediately satisfied by the conditions of this case. Therefore, in Case 1,
there exists an objective $(f, g, \gD) \in \mathcal{F}_{\text{det}}$ such that
$\|\nabla f(\vx_t)\| \geq \Delta L_1$ for all $t \geq 0$.

\paragraph{Case 2} For this case, we will reuse the hard instance from Lemma
\ref{lem:app_adagrad_slow}. Denoting $m = \frac{1}{L_1} \log \left( 1 +
\frac{L_1 \epsilon}{L_0} \right)$, the objective is defined as \[
    f(x) = \begin{cases}
        -\epsilon (x+m) + \psi(m) & x < -m \\
        \psi(x) & x \in [-m, m] \\
        \epsilon (x-m) + \psi(m) & x > m
    \end{cases},
\] where \[
    \psi(x) = \frac{L_0}{L_1^2} \left( \exp(L_1 |x|) - L_1 |x| - 1 \right).
\] With $g, \gD$ defined so that $g(x, \xi) = f'(x)$ almost surely when $\xi
\sim \gD$, it was already shown in the proof of Lemma \ref{lem:app_adagrad_slow}
that $(f, g, \gD) \in \mathcal{F}_{\text{det}}$, when we use the initial point
$x_0 = m + \frac{\Delta}{2 \epsilon}$.

Letting $x_t$ be the sequence of iterates generated by Decorrelated AdaGrad-Norm,
we define $t_0 = \max \left\{ t \geq 0 \;|\; x_t \geq m \right\}$. Notice that
$f'(x) = \epsilon$ for all $x \geq m$, so the definition of $t_0$ implies that
$|f'(x_t)| = \epsilon$ for all $t \leq t_0$. Accordingly, we want to show that \[
    t_0 \geq \tilde{\Omega} \left( \frac{\Delta^2 L_1^2}{\epsilon^2} \right).
\]
Actually, the trajectory of Decorrelated AdaGrad-Norm for this objective is
identical to that of Decorrelated AdaGrad, since the objective's domain is
one-dimensional. Therefore, to analyze the trajectory $x_t$, we can reuse the
analysis from the proof of Lemma \ref{lem:app_adagrad_slow}. Starting from
\Eqref{eq:adagrad_slow_inter_3}, \[
    \frac{t_0}{\sqrt{\gamma^2 + t_0 \epsilon^2}} \geq \frac{1}{\epsilon} \left( \frac{\Delta}{4 \eta \epsilon} - \frac{\epsilon}{2 \gamma} \right).
\]
Using the assumed upper bound on $\epsilon$,
\begin{align*}
    \epsilon &\leq \sqrt{ \frac{\Delta \gamma}{4 \eta} } \\
    \epsilon^2 &\leq \frac{\Delta \gamma}{4 \eta} \\
    \frac{\epsilon}{2 \gamma} &\leq \frac{\Delta}{8 \eta \epsilon},
\end{align*}
so
\begin{align*}
    \frac{t_0}{\sqrt{\gamma^2 + t_0 \epsilon^2}} &\geq \frac{\Delta}{8 \eta \epsilon^2} \\
    8 \eta \epsilon^2 t_0 &\geq \Delta \sqrt{\gamma^2 + \epsilon^2 t_0} \\
    64 \eta^2 \epsilon^4 t_0^2 &\geq \Delta^2 \gamma^2 + \Delta^2 \epsilon^2 t_0 \\
    t_0^2 &\geq \frac{\Delta^2 \gamma^2}{64 \eta^2 \epsilon^4} + \frac{\Delta^2}{64 \eta^2 \epsilon^2} t_0.
\end{align*}
Denoting $b = \frac{\Delta^2}{64 \eta^2 \epsilon^2}$ and $c = \frac{\Delta^2
\gamma^2}{64 \eta^2 \epsilon^4}$, this gives the quadratic inequality \[
    t_0^2 - b t_0 - c \geq 0.
\]
Since $t_0 > 0$, this implies
\begin{equation} \label{eq:gamma_inter}
    t_0 \geq \frac{b + \sqrt{b^2 + 4c}}{2} \geq \frac{b}{2} + \sqrt{c} = \frac{\Delta^2}{128 \eta^2 \epsilon^2} + \frac{\Delta \gamma}{8 \eta \epsilon^2}.
\end{equation}
Finally, we can apply the conditions on $\eta$ and $\gamma/\eta$ from the case analysis. We know that either
\begin{equation}
    \eta \leq 1/L_1 \quad \text{or} \quad \gamma/\eta \geq \frac{\Delta L_1^2}{8 \log \left( 1 + \frac{48 \Delta L_1^2}{L_0} \right)}.
\end{equation}
If $\eta \leq 1/L_1$, then \Eqref{eq:gamma_inter} implies \[
    t_0 \geq \frac{\Delta^2}{128 \eta^2 \epsilon^2} \geq \frac{\Delta^2 L_1^2}{\epsilon^2}.
\]
On the other hand, if \[
    \gamma/\eta \geq \frac{\Delta L_1^2}{8 \log \left( 1 + \frac{48 \Delta L_1^2}{L_0} \right)},
\]
then \Eqref{eq:gamma_inter} implies \[
    t_0 \geq \frac{\Delta \gamma}{8 \eta \epsilon^2} \geq \frac{\Delta^2 L_1^2}{64 \epsilon^2 \log \left( 1 + \frac{48 \Delta L_1^2}{L_0} \right)}.
\]
Either way, we have \[
    t_0 \geq \tilde{\Omega} \left( \frac{\Delta^2 L_1^2}{\epsilon^2} \right),
\] which finishes the analysis for Case 2.

\paragraph{Putting The Cases Together} The case analysis above shows that, no
matter the choice of $\gamma, \eta$, there always exists some objective $(f, g,
\gD) \in \gF_{\text{det}}$ such that the number of iterations to find an
$\epsilon$-stationary point is at least \[
    \tilde{\Omega} \left( \frac{\Delta^2 L_1^2}{\epsilon^2} \right).
\]
\end{proof}